%% file: main.tex
\documentclass[conference]{IEEEtran}

\usepackage{times}
\usepackage{cuted}
\usepackage{capt-of}
\usepackage{multicol}

\input{include/packages}
\input{include/macros}

\begin{document}

\title{\huge MISO: Multiresolution Submap Optimization for Efficient Globally Consistent Neural Implicit Reconstruction}

\author{\authorblockN{Yulun Tian,
Hanwen Cao,
Sunghwan Kim and
Nikolay Atanasov}
\authorblockA{Department of Electrical and Computer Engineering\\
University of California, San Diego\\
La Jolla, CA 92093, USA\\ 
Email: \{yut034,h1cao,suk063,natanasov\}@ucsd.edu}}

\maketitle
\begin{strip}\centering
\vspace{-1.5cm}
\includegraphics[trim=15 230 15 230, clip, width=\linewidth]{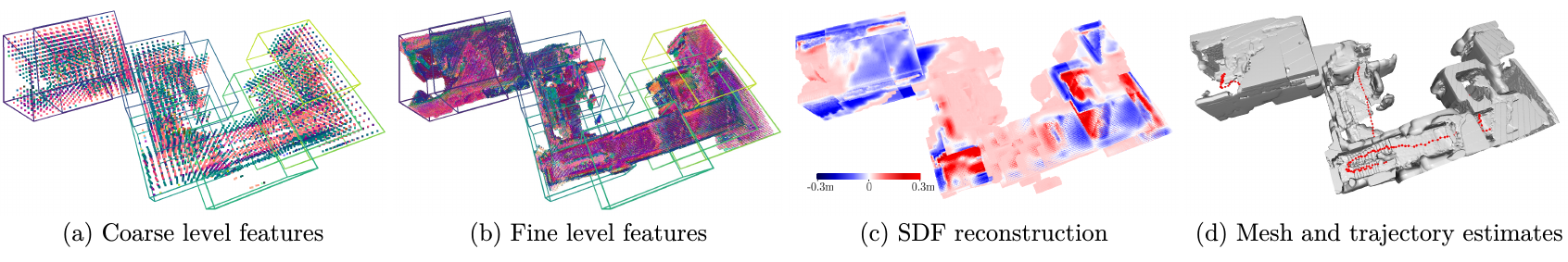}
\captionof{figure}{Demonstration of \AlgName on the \Fastcamo dataset \cite{tang2023mips}. 
    \AlgName leverages \emph{multiresolution submaps} that organize neural implicit features at different spatial resolutions (visualized in (a) and (b) using principal component analysis).
    By performing hierarchical optimization within (local) and across (global) submaps, \AlgName can efficiently and accurately reconstruct SDF (c) and estimate mesh and robot trajectory (d). For clarity, we only visualize the features and SDF values within $30$~cm of the surface. 
    \edit{The scene size is $26.0$~m $\times$ $16.7$~m $\times$ $7.5$~m.}
\label{fig:demo}}
\end{strip}

\begin{abstract}
Neural implicit representations have had a significant impact on simultaneous localization and mapping (SLAM) by enabling robots to build continuous, differentiable, and high-fidelity 3D maps from sensor data.
However, as the scale and complexity of the environment increase, neural SLAM approaches face renewed challenges in the back-end optimization process to keep up with runtime requirements and maintain global consistency.
We introduce \AlgName, a {hierarchical} optimization approach that leverages \emph{multiresolution submaps} to achieve efficient and scalable neural implicit reconstruction.
For local SLAM within each submap, we develop a hierarchical optimization scheme with learned initialization that substantially reduces the time needed to optimize the implicit submap features.
To correct estimation drift globally, we develop a hierarchical method to align and fuse the multiresolution submaps, leading to substantial acceleration by avoiding the need to decode the full scene geometry. 
\AlgName significantly improves computational efficiency and estimation accuracy of neural signed distance function (SDF) SLAM on large-scale real-world benchmarks.
\end{abstract}

\IEEEpeerreviewmaketitle

\input{sections/intro}
\input{sections/related_work}

\input{sections/overview}

\input{sections/local_mapping}

\input{sections/global_mapping}

\input{sections/experiments}

\input{sections/discussions}

\input{sections/conclusion}

\edit{
\section*{Acknowledgments}
We gratefully acknowledge support from ARL DCIST CRA W911NF-17-2-0181, ONR N00014-23-1-2353, and NSF CCF-2112665 (TILOS).
}

\bibliographystyle{ieeetr}
\bibliography{references}

\begin{appendices}
    \input{sections/grid_details}

    \input{sections/proof}
    \input{sections/additional_results}

\end{appendices}

\end{document}

%% file: include/packages.tex
\usepackage{amsmath}
\usepackage{amssymb}
\usepackage{cases}
\usepackage{graphicx}
\usepackage[sort,compress]{cite}
\usepackage{booktabs}

\usepackage{amsthm}

\newtheorem{proposition}{Proposition}

\theoremstyle{definition}
\newtheorem{definition}{Definition}
\newtheorem{problem}{Problem}

\usepackage{mdframed} 

\usepackage[]{xcolor}
\usepackage{subcaption}
\usepackage{wrapfig}
\usepackage{multirow}
\usepackage{mathtools}
\usepackage{empheq}
\usepackage{color}
\usepackage{bbm}
\usepackage{xspace}
\let\labelindent\relax
\usepackage[shortlabels]{enumitem}
\usepackage{tabularray}

\usepackage{tikz} 
\usepackage{tkz-graph}
\usepackage{tkz-berge}
\usetikzlibrary{backgrounds,fit,shapes,snakes,arrows,shapes.geometric,positioning}
\usetikzlibrary{intersections,patterns,shapes.misc}
\usetikzlibrary{decorations.pathmorphing}
\usepackage{pgfplots}

\tikzstyle{block} = [rectangle, rounded corners, minimum width=3cm, minimum height=1cm,text centered, draw=black, fill=red!30]
\tikzstyle{new} = [rectangle, rounded corners, minimum width=1cm, minimum
height=1cm,text centered, draw=black, fill=blue!10!white, dashed]
\tikzstyle{arrow} = [thick,->,>=stealth]
\usetikzlibrary{calc, quotes}

\tikzstyle{fblock} = [rectangle, draw, fill=gray!20, 
text width=8em, text centered, rounded corners, minimum height=4em, minimum width = 8em]
\tikzstyle{line} = [draw, -latex']

\usepackage{algorithmicx}
\usepackage{algorithm}
\usepackage[]{algpseudocode}

\usepackage{fontawesome}

\usepackage{soul}
\usepackage{comment}
\usepackage[]{url}
\usepackage[font=small]{caption}
\usepackage[normalem]{ulem}
\useunder{\uline}{\ul}{}

\usepackage[pdftex, colorlinks=true, linkcolor=blue, citecolor = black, urlcolor = cyan, filecolor=black, pagebackref=true, hypertexnames=false]{hyperref}
\usepackage{cleveref}
\crefname{problem}{Problem}{Problems}
\crefname{example}{Example}{Examples}
\crefname{section}{Sec.}{Secs.}
\Crefname{section}{Section}{Sections}
\Crefname{table}{Table}{Tables}
\crefname{table}{Table}{Tabs.}
\crefname{figure}{Fig.}{Figs.}
\crefname{algorithm}{Algorithm}{Algorithms}
\crefname{remark}{Remark}{Remarks}
\crefname{theorem}{Theorem}{Theorems}
\crefname{proposition}{Proposition}{Propositions}
\crefname{lemma}{Lemma}{Lemmas}
\crefname{corollary}{Corollary}{Corollaries}
\crefname{assumption}{Assumption}{Assumptions}
\crefname{definition}{Definition}{Definitions}

%% file: include/macros.tex
\newcommand{\Real}{\mathbb{R}}

\providecommand{\norm}[1]{\left\|#1\right\|}

\def \pinv {^\dagger}

\DeclareMathOperator{\dist}{\mathbf{d}}

\DeclareMathOperator{\Exp}{Exp}
\DeclareMathOperator{\Log}{Log}

\DeclareMathOperator{\se}{\mathfrak{se}}
\providecommand{\hatop}[1]{\left[#1\right]_\times}
\DeclareMathOperator{\SE}{SE}

\newcommand{\bmat}{\begin{bmatrix}}
\newcommand{\emat}{\end{bmatrix}}

\newcommand{\etal}{\emph{et~al.}\xspace}

\newcommand{\eg}{\emph{e.g.}\xspace}
\newcommand{\ie}{\emph{i.e.}\xspace}

\providecommand{\method}[1]{{\small \textsf{#1}}\xspace}

\newcommand{\Ecal}{\mathcal{E}}

\newcommand{\That}{\widehat{T}}

\newcommand{\lidar}{{LiDAR}\xspace}
\newcommand{\AlgName}{{MISO}\xspace}
\newcommand{\Point}{{Neural Points}\xspace}
\newcommand{\iSDF}{{iSDF}\xspace}
\newcommand{\mips}{{MIPS}\xspace}
\newcommand{\vfpp}{{VFPP}\xspace}
\newcommand{\cfeat}{c^\text{feat}}
\newcommand{\csdf}{c^\text{sdf}}
\newcommand{\Fastcamo}{FastCaMo-Large\xspace}

\newcommand{\myParagraph}[1]{{\bf #1.}\xspace}

\providecommand{\optional}[1]{{}}

\providecommand{\edit}[1]{{#1}}

\providecommand{\techreport}[1]{{}}  

%% file: sections/intro.tex
\section{Introduction}

In recent years, \emph{neural fields} \cite{xie2022neural} have emerged as a new frontier for scene representation in simultaneous localization and mapping (SLAM).
Compared to conventional approaches based on hand-crafted features or volumetric representations, neural SLAM \cite{tosi2024nerfs} offers advantages including continuous and differentiable scene modeling, improved memory efficiency, and better handling of measurement noise.
However, a crucial limitation remains in the back-end optimization of neural SLAM: as the environment size and mission time grow, most existing approaches consider increasingly larger optimization problems that ultimately limit their real-time performance.

A powerful idea to achieve more efficient SLAM is to depend on a \emph{hierarchical} representation that explicitly disentangles coarse and fine information of the environment.
Equipped with such a hierarchical model, a robot can perform inference over varying spatial resolutions, \eg, by first capturing the core structure in the environment and then optimizing the fine details later.
In SLAM, this idea dates back to several seminal works such as \cite{estrada2005hierarchical,frese2005multilevel,grisetti2010hierarchical}.
Recently, hierarchical or multiresolution representations have also achieved success in neural fields \cite{takikawa2021nglod,sun2022direct,muller2022instant}, demonstrating state-of-the-art performance and cost-quality trade-offs in many computer vision tasks.
Nevertheless, neural SLAM systems have yet to benefit from these recent advances, as the majority of back-end solvers do not fully utilize the hierarchical form of the representations.

In this work, we develop a \emph{hierarchical optimization} approach that directly uses multiresolution implicit features for neural SLAM.
This approach enables us to solve a significant portion of the back-end optimization in the \emph{implicit feature space}, and thus obtain substantial gains in efficiency and robustness compared to existing methods that depend on geometric reconstruction \cite{tang2023mips,zhai2024vox}. 
To scale to larger scenes, we adopt a submap-based design that models the environment as a collection of local neural implicit maps.
In this context, we show that the proposed hierarchical optimization significantly enhances both \emph{local submap optimization} and \emph{global submap fusion} stages in SLAM.
We apply our approach to neural signed distance function (SDF) SLAM \cite{ortiz2022isdf} and demonstrate its effectiveness on real-world, large-scale datasets.

\textbf{Contributions. }
We present \AlgName (\underline{M}ult\underline{I}resolution \underline{S}ubmap \underline{O}ptimization), a hierarchical optimization approach for neural implicit SLAM.
\edit{\AlgName performs local pose and submap optimization and global submap fusion, which can be used to achieve SDF SLAM from depth images or \lidar scans.}
For \emph{local} submap optimization, \AlgName introduces a learning-based hierarchical initialization method to generate multiresolution submap features, which are subsequently refined through joint optimization with robot poses.
As a theoretical motivation, we derive a closed-form solution to the initialization problem under the special case of linear least squares optimization. 
Leveraging this theoretical insight, we design hierarchical encoders to learn effective initializations in the general case.
For \emph{global} submap fusion, \AlgName presents a hierarchical optimization method to align and fuse submaps in the global frame.
Compared to previous works, our approach achieves faster and more robust performance by directly using information stored in the hierarchical implicit features rather than relying on geometric reconstruction. 
\edit{Evaluation on benchmark datasets shows that \AlgName achieves competitive estimation quality and significantly outperforms existing methods in computational efficiency.}
\Cref{fig:demo} demonstrates \AlgName on the real-world \Fastcamo dataset \cite{tang2023mips}.

\textbf{Notation. }
Unless stated otherwise, lowercase and uppercase letters denote vectors and matrices, respectively.
We define $[n] \triangleq \{1, 2, \hdots, n\}$ as the set of positive integers from 1 to $n$.
The special Euclidean group in 3D is denoted by $\SE(3)$, and $\SE(3)^n$ denotes its product manifold.
A local perturbation on the tangent space of $\SE(3)$ is represented by a vector $\varepsilon \in \Real^6$.
The exponential map $\Exp: \Real^6 \to \SE(3)$ is given by $\Exp(\varepsilon) = \exp(\hatop{\varepsilon})$,  where $\hatop{\varepsilon} \in \se(3)$ is the Lie algebra element corresponding to $\varepsilon$ and $\exp$ is the standard matrix exponential.
The inverse of the exponential map is denoted as $\Log: \SE(3) \to \Real^6$.
Given $T = (R, t) \in \SE(3)$ and a point $x \in \Real^3$, $Tx = Rx+t \in \Real^3$ denotes the transformed point.

%% file: sections/related_work.tex
\section{Related Work}
\label{sec:related_work}

We review related work on neural implicit representations for SLAM, with particular focus on neural SDF reconstruction and submap decompositions.
The reader is referred to recent surveys \cite{tosi2024nerfs,xie2022neural} for further discussions and reviews of alternative representations including 3D Gaussian splatting \cite{kerbl2023gaussian}.

\subsection{Neural Implicit Representations}

Neural implicit representations offer continuous and differentiable modeling of 3D scenes with high fidelity, memory efficiency, and robustness to noise \cite{park2019deepsdf, mescheder2019occupancy, mildenhall2021nerf, azinovic2022neural}. Early methods such as DeepSDF \cite{park2019deepsdf} and NeRF \cite{mildenhall2021nerf} rely solely on 3D coordinates and a single multi-layer perceptron (MLP) to reconstruct the scene. However, this approach is insufficient for capturing larger scenes or complex details, prompting subsequent works to introduce hybrid methods that combine MLP decoders with additional implicit features.
The implicit features are commonly organized in a 3D grid \cite{fridovich2022plenoxels, sun2022direct,takikawa2021nglod,muller2022instant}.
To enable continuous scene modeling, trilinear interpolation is used to infer a feature at an arbitrary query location that is subsequently passed through the MLP decoder to predict the environment model (\eg, occupancy, distance, radiance).
Recent works propose several alternative approaches to improve the memory efficiency over 3D feature grids.
K-Planes \cite{fridovich2023k} factorizes the scene representation into multiple 2D feature planes rather than using a full 3D voxel grid. Similarly, TensoRF~\cite{chen2022tensorf} employs tensor decomposition to compactly represent radiance fields. 
PointNeRF~\cite{xu2022point} constructs the scene representation directly from point clouds by efficiently aggregating local features at surface points.

Hierarchical strategies for organizing the implicit features have been particularly effective at capturing different levels of detail while maintaining efficiency \cite{muller2022instant, sun2022direct, li2023neuralangelo, yu2021plenoctrees, takikawa2021nglod}. 
DVGO~\cite{sun2022direct} performs progressive scaling that gradually increases the feature grid resolution during training.
InstantNGP \cite{muller2022instant} significantly accelerates feature grid training and inference by introducing a multiresolution hash encoding scheme.
Neuralangelo \cite{li2023neuralangelo} extends this concept with a coarse-to-fine optimization scheme that preserves fine-grained details.  
In parallel, octree-based frameworks provide an adaptive representations for large scenes by varying spatial resolution where needed \cite{takikawa2021nglod, yu2021plenoctrees}. 
\edit{H$_2$-Mapping~\cite{jiang2023h2} achieves incremental neural SDF mapping by combining octree-based coarse SDF and multiresolution feature grids, where the latter is optimized to learn residual geometry.} 
Hierarchical representations have also been explored for fast RGB-D surface reconstruction \cite{wang2022go, azinovic2022neural}. 
In this work, we leverage these hierarchical neural representations to achieve efficient and accurate back-end optimization for neural SLAM.
 
\subsection{Neural SDF SLAM}

Recent SLAM systems have achieved remarkable progress by modeling the environment using neural implicit SDF.
Building on DeepSDF~\cite{park2019deepsdf},
iSDF \cite{ortiz2022isdf} uses a single MLP for online SDF reconstruction from streaming RGB-D data. iSDF selects keyframes based on information gain \cite{sucar2021imap} and samples free-space and near-surface points along camera rays to train the MLP.
VoxFusion~\cite{yang2022vox} leverages a sparse octree to organize implicit features and Morton coding for efficient allocation and retrieval, enabling real-time SLAM in dynamically expanding environments. 
Vox-Fusion++~\cite{zhai2024vox} extends the method to large-scale scenes through submap support.
NICER-SLAM~\cite{zhu2024nicer} transforms estimated SDF to density for volume rendering during monocular SLAM.
NeRF-LOAM~\cite{deng2023nerf} similarly uses SDF to represent the geometry for neural lidar odometry and mapping, and develops a dynamic voxel embeddings generation method to speed up octree queries. 
PIN-SLAM~\cite{pan2024pin} reconstructs SDF via sparse neural point features, and employs voxel hashing to speed up spatial querying for online SLAM.
\edit{PINGS~\cite{pan2025pings} is a concurrent work that extends PIN-SLAM to enable photorealistic rendering via Gaussian Splatting. The neural point features are decoded to spawn Gaussian primitives locally, and trained with both SDF- and GS-based losses to enhance geometric consistency.}
ESLAM~\cite{johari2023eslam} uses multi-scale axis-aligned tri-plane feature grids with a TSDF representation to achieve memory-efficient reconstruction and localization. 
Co-SLAM~\cite{wang2023co} combines smooth one-blob coordinate encoding with local-detail hash-grid embeddings to improve camera tracking.
GO-SLAM~\cite{zhang2023go} supports loop closing and online bundle adjustment with a multi-resolution hash-grid design for both SDF and color.
Despite these advancements, achieving globally consistent SDF reconstruction of large-scale scenes remains challenging. 
In this work, we address this limitation by developing a hierarchical approach for both local and global multiresolution submap optimization.

\begin{figure*}[t]
    \centering
    \begin{subfigure}[t]{0.56\linewidth}
        \centering
        \includegraphics[width=\linewidth]{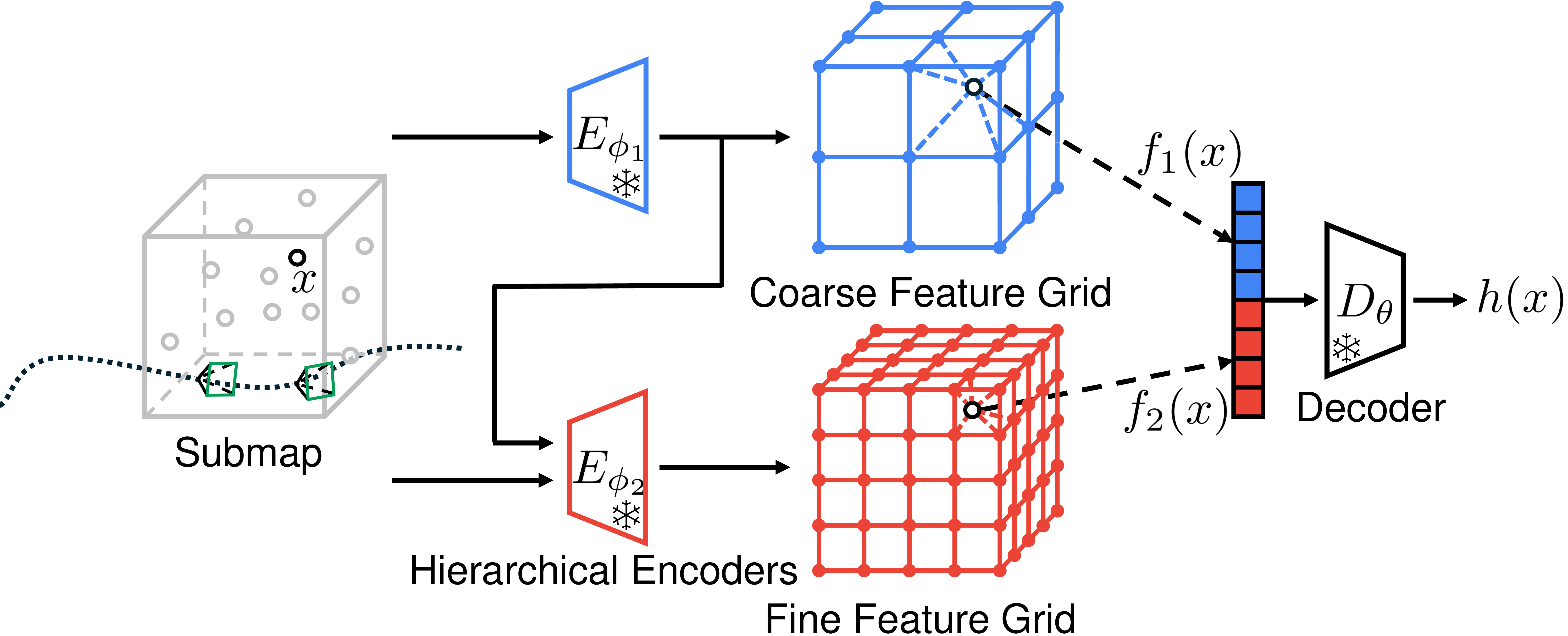}
        \caption{Local mapping with multiresolution feature grid}
        \label{fig:overview:local}
    \end{subfigure}%
    \hspace{1cm}%
    \begin{subfigure}[t]{0.38\linewidth}
        \centering
         \includegraphics[width=\linewidth]{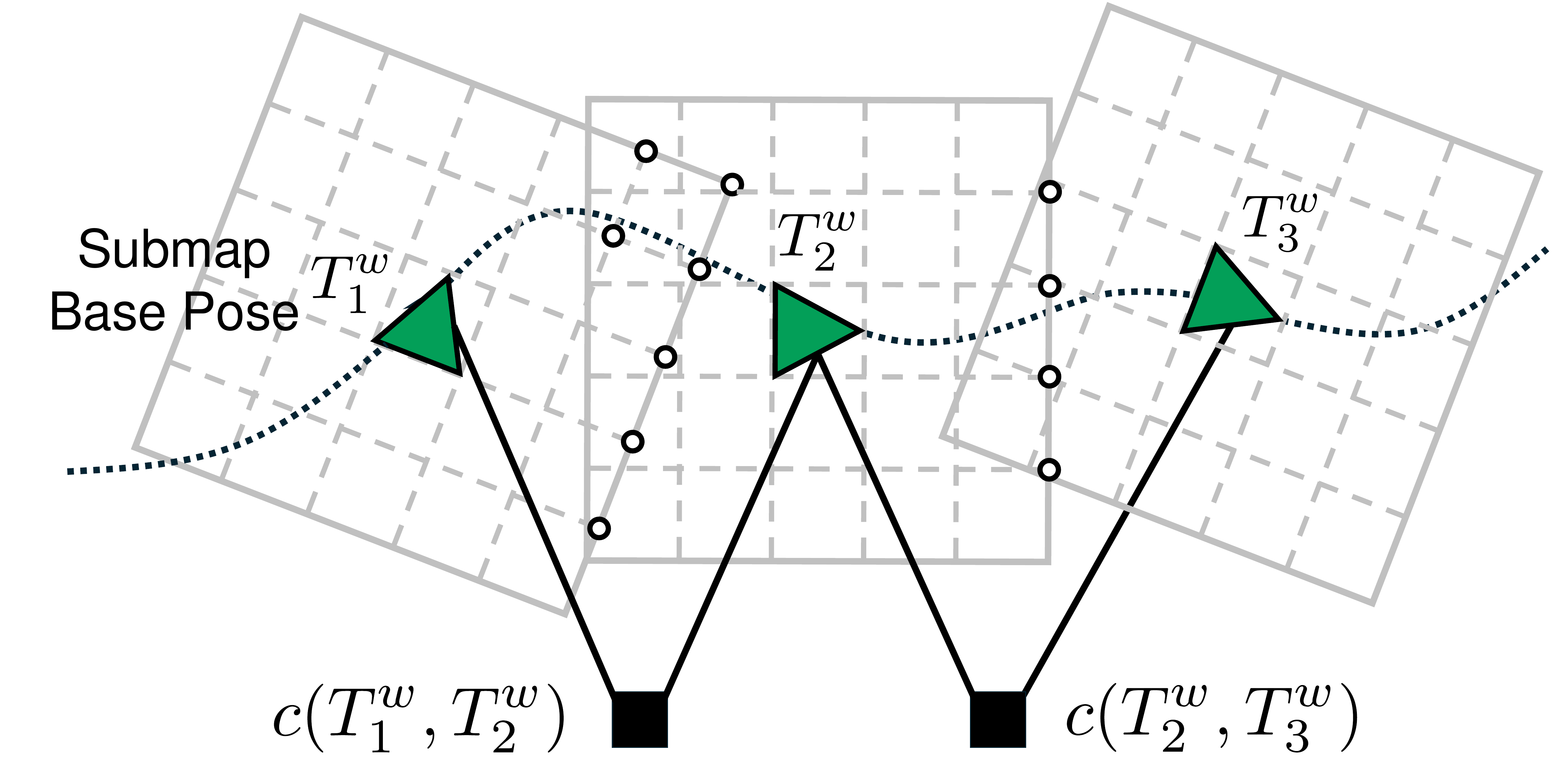}
        \caption{Global submap alignment and fusion}
        \label{fig:overview:global}
    \end{subfigure}
    \caption{\textbf{Overview of \AlgName.} 
    (a) Given point cloud observations, \AlgName performs local hierarchical SLAM within a submap represented as a multiresolution feature grid (\cref{sec:local_optimization}).
    (b) 
    Given locally optimized submaps, \AlgName performs global alignment and fusion across submaps
    to eliminate estimation drift and achieve globally consistent scene reconstruction (\cref{sec:global_optimization}).
    }
    \label{fig:overview}
\end{figure*}

\subsection{Submap-based Neural SLAM}

An effective strategy for large-scale 3D reconstruction is to partition the scene into multiple submaps. K\"ahler \etal~\cite{kahler2016real} create submaps storing truncated SDF values based on visibility criteria and align them by optimizing the relative poses of overlapping keyframes.
MIPS-Fusion~\cite{tang2023mips} extends this idea by incrementally generating MLP-based submaps based on the camera’s field of view and aligning them via point-to-plane refinement. 
Vox-Fusion++~\cite{zhai2024vox} adopts a dynamic octree structure for each submap and performs joint camera tracking and submap alignment by optimizing a differentiable rendering loss. 
Loopy-SLAM~\cite{liso2024loopy} uses a neural-point-based approach, creating submaps upon large camera rotations and later constructing a pose graph with iterative closest point (ICP) to detect loop closures. More recently, PLGSLAM~\cite{deng2024plgslam} combines axis-aligned tri-planes for high-frequency features with an MLP for low-frequency components, enabling multiple local representations to be merged efficiently.  NEWTON~\cite{matsuki2024newton} employs a spherical coordinate system to create local maps that accommodate flexible boundary adjustments. 
Multiple-SLAM~\cite{liu2023efficient} and CP-SLAM \cite{hu2024cp} consider collaborative scenarios and fuse local neural implicit maps from multiple agents.
Although effective, existing methods require reconstructing the scene’s geometry to align submaps, which can be costly and inaccurate in real-world settings. In contrast, our method aligns submaps directly in the feature space via hierarchical optimization, providing both fast and robust performance without explicit geometric reconstruction.

%% file: sections/overview.tex

\section{Overview}
In \AlgName, we represent the scene as a collection of posed submaps.
Correspondingly, the back-end optimization involves two types of problems: 
(i) local SLAM within each submap and (ii) global alignment and fusion across all submaps. See \cref{fig:overview} for an illustration.

Given odometry and point-cloud observations \edit{from depth images or \lidar scans}, a robot aims to estimate its trajectory and build a local map represented as a multiresolution feature grid (\cref{fig:overview:local}). 
Organizing implicit features into a \emph{hierarchy of grids} effectively disentangles information at different spatial resolutions.
At inference time, interpolated features from different hierarchy levels are aggregated and processed by a \emph{decoder network} to predict the scene geometry.
To speed up local optimization, we introduce \emph{hierarchical encoder networks} to initialize the grid features at each hierarchy level directly from input observations. To achieve further acceleration and enable generalization to new environments, both the encoder and decoder networks are pre-trained offline over multiple scenes and fixed during online SLAM. 
\cref{sec:local_optimization} presents in detail our local SLAM method.

In large environments or over long time durations, the robot trajectory estimates will inevitably drift and cause the submaps to be misaligned.
To address this challenge, \AlgName introduces an approach to align and fuse all submaps in the global reference frame (\cref{fig:overview:global}). 
Each submap is associated with a \emph{base pose} that determines the transformation from the local (submap) frame to the global frame.
Compared to existing approaches, which rely on decoding the scene geometry into an explicit representation like occupancy, mesh, or distance field, \AlgName performs alignment and fusion directly using the implicit features in the multiresolution submaps.
We show that this results in significantly faster optimization and outperforms other methods under large initial alignment errors.
\cref{sec:global_optimization} presents the details of the global alignment and fusion method.

%% file: sections/local_mapping.tex
\section{Local SLAM}
\label{sec:local_optimization}

This section introduces our submap representation utilizing  multiresolution feature grids and our hierarchical submap optimization method.

\subsection{Local SLAM with Multiresolution Feature Grid}
\label{sec:mapping_with_multiresoluion_feature_grid}

We represent each local submap as a multiresolution feature grid \cite{takikawa2021nglod,sun2022direct,muller2022instant}, defined formally below.

\begin{definition}[Multiresolution Feature Grid]
\label{def:grid}
A \emph{multiresolution feature grid} contains $L>1$ levels of regular grids with increasing spatial resolution ordered from coarse ($l=1$) to fine ($l=L$).
At each level $l$, each vertex located at $z_{l,i} \in \Real^3$ stores a learnable feature vector $f_{l,i} \in \Real^d$.
Together with a kernel function $k_l: \Real^3 \times \Real^3 \to \Real$, the feature grid defines a continuous feature field $f_l(x)= \sum_{i \in I_l} k_l(x, z_{l,i}) f_{l,i}$, 
where $x \in \Real^3$ is any query position, and $I_l$ indexes over all vertices at level $l$.
To obtain a scalar output (\eg, signed distance or occupancy) at query position $x$, the features at different levels are concatenated (denoted by $\bigoplus$) and processed by a decoder network $D_\theta$,
\begin{equation} \label{eq:multiresolution_grid}
h(x; F, \theta) = D_\theta \bigl(
	\bigoplus_{l \in [L]}
	f_l(x)
\bigr).
\end{equation}
The model has the set of features from all levels $F$ and the decoder parameters $\theta$ as learnable parameters.
\end{definition}

In this work, we implement the kernel functions $k_l$ using trilinear interpolation.  While the multiresolution feature grid offers a powerful representation, directly using it as a map representation in SLAM presents a computational challenge due to the need to train the decoder network $D_\theta$. Even if computation is not a concern, training the decoder with a small dataset or during a single SLAM session may lead to unreliable generalization or catastrophic forgetting \cite{tosi2024nerfs}. 

To address these challenges, we pre-train the decoder $D_\theta$ offline over multiple scenes, similar to prior works (\eg, \cite{zhu2022nice, pan2024pin}). The details of the offline decoder training are presented in Appendix~\ref{sec:grid_details}. During online SLAM, the decoder weights are fixed (as shown in \cref{fig:overview:local}), and the robot only needs to optimize the grid features $F$ and its own trajectory.
Specifically, 
within each submap $s$, we are given noisy pose estimates $\{\hat{T}_k^s\}_k$ (\eg, from odometry) and associated observations $\{X^k\}_k$, where each $X^k = \{x^k_1,\ldots,x^k_{m_k}\} \subset \Real^3$ is a point cloud observed at pose $k$.
Using this information, we seek to jointly refine the robot's pose estimates and the submap features $F$ via the following optimization problem.

\begin{problem}[Local SLAM within a submap] \label{prob:local_mapping}
Given $n$ initial pose estimates $\{\hat{T}_k^s\}_k$ in the reference frame of submap $s$ and associated point-cloud observations $\{X^k\}_k$ in the sensor frame, the local SLAM problem is defined as,
\begin{equation}
\label{eq:local_mapping}
\underset{F, \{T_k^s\}_k \subset \SE(3)}{\min}
	 \!
	 \sum_{k=1}^n \sum_{j=1}^{m_k} 
        c_j\bigl(
            h(T^s_k x^k_j; F)
        \bigr) 
        + \!
         \sum_{k=1}^n \rho(\That^s_k, T^s_k),
\end{equation}
where $c_j: \Real \to \Real$ is a cost function associated with the $j$-th observation and $\rho: \SE(3) \times \SE(3) \to \Real $ is a pose regularization term. We drop the dependence of the model $h$ on the decoder parameters $\theta$ to reflect that the decoder is trained offline. 
\end{problem}

The first group of terms in \eqref{eq:local_mapping} evaluates the environment reconstruction at observed points $x_j^k$ by transforming them to the submap frame (\ie, $x^s_j = T^s_k x^k_j$) and querying the feature grid model $h$. Empirical results show that introducing the second group of pose regularization terms helps the optimization remain robust against noisy or insufficient observations.
We use regularization inspired by trust-region methods \cite[Ch.~4]{nocedal1999numerical},
\begin{equation}
\rho(\That, T) = w^\rho \max \bigl( \bigl \| \Log(\That^{-1}T) \bigr \|_2 - \tau, 0 \bigr ),
\label{eq:trust_region}
\end{equation}
which penalizes pose updates larger than the trust-region radius $\tau$, and $w^\rho$ is a weight parameter (default to $10^3$).

In our implementation, we solve \cref{prob:local_mapping} approximately by parametrizing each pose variable locally as $T^s_k =  \That^s_k \Exp(\varepsilon^s_k)$ where $\varepsilon^s_k \in \Real^6$ is the local pose correction.
Both the grid features $F$ and the correction terms $\{\varepsilon^s_k\}_k $ are optimized using Adam \cite{kingma2014adam} in PyTorch \cite{paszke2019pytorch}.

We introduce definitions of the cost $c_j$ specific to neural SDF reconstruction next. Whenever clear from context, to ease the notation we use $x_j \equiv x^s_j = T^s_k x^k_j$ to represent a point in the submap frame.

\myParagraph{Cost functions for neural SDF reconstruction}
We follow iSDF \cite{ortiz2022isdf} to design measurement costs for SDF reconstruction. Specifically, we classify all observed points
as either (i) on or near surface (default to $30$~cm as in iSDF), or (ii) in free space. For on or near surface observations, the cost function $c_j = c_j^{\text{sdf}}$ is based on direct SDF supervision,
\begin{equation}
	c_j^{\text{sdf}} \left(h(x_j)\right) = w_j^\text{sdf} \left |h(x_j) - y_j \right|,
	\label{eq:sdf_residual}
\end{equation}
where $w_j^\text{sdf}>0$ is measurement weight (default to $5.4$ as in iSDF) and $y_j \in \Real$ is a measured SDF value on or near surface obtained using the approach from iSDF \cite{ortiz2022isdf}.

For free-space observations, we use the cost to enforce bounds on the SDF values. Specifically, we follow iSDF to obtain lower and upper bounds $\underline{b}_j, \bar{b}_j$ on the SDF from sensor measurements, and define $c_j = c_j^{\text{bnd}}$ as,
\begin{align}
    c_j^\text{lo} \left(h(x_j) \right) &= \max (e^{\beta (\underline{b}_j - h(x_j))} - 1, \; 0), \\
    c_j^\text{up} \left(h(x_j) \right) &= \max (h(x_j) - \bar{b}_j, \; 0), \\
    c_j^{\text{bnd}} \left(h(x_j) \right) &= \max(c_j^\text{lo} \left(h(x_j) \right), c_j^\text{up} \left(h(x_j) \right) ).
    \label{eq:bnd_residual}
\end{align}
This cost applies exponential penalty ($\beta=5$ by default) for the lower bound and linear penalty for the upper bound.
This is because, in practice, violation of the lower bound is usually more critical, \eg, if $\underline{b}_j = 0$ and the model predicts negative SDF values. We do not include Eikonal regularization \cite{gropp2020implicit} because we observed that it has limited impact on accuracy while making the optimization slower.

\subsection{Hierarchical Feature Initialization for Local SLAM}
\label{sec:hierarchical_local_mapping}

In practice, the bulk of the computational cost in \cref{prob:local_mapping} is incurred by the optimization over the high-dimensional grid features $F$.
To address this challenge, we propose a method that leverages the structure of the multiresolution grid to learn to initialize $F$ from sensor observations. 
While prior works such as Neuralangelo~\cite{li2023neuralangelo} advocate for coarse-to-fine training strategies, a crucial gap remains since the features at each level are still optimized from scratch, \eg, from zero or random initialization.
Our key intuition is that, at any level, a much more effective initialization can be obtained by accounting for optimization results from the previous levels.

In the following, we use $F_l$ to denote the subset of latent features at level $l$,
and $F_{1:l}$ denote all latent features up to and including level $l$.
We consider the problem of initializing $F_l$ given fixed submap poses and coarser features $F_{1:l-1}$. 
This amounts to solving the following subproblem of \cref{prob:local_mapping},
\begin{equation}
\label{eq:level_local_mapping}
\underset{{F_l}}{\min}
	\quad \sum_{k=1}^n \sum_{j=1}^{m_k} 
		c_j \left(
		h(T^s_k x^k_j; F_{1:l-1}, F_l, 0_{l+1:L})
		\right),
\end{equation}
where we explicitly expand $F$ into the (known) coarser features $F_{1:l-1}$, 
the target feature to be initialized $F_l$, and finer features (assumed to be zero).
During initialization, we do not consider pose optimization and thus drop the trust-region regularization in \cref{prob:local_mapping}.

To develop our approach, we first present theoretical analysis and derive a closed-form solution to \eqref{eq:level_local_mapping} in a special linear-least-squares case. Leveraging insight from the closed-form solution in the linear case, we then develop a learning approach to initialize the grid features at each level, applicable to the general (nonlinear) problem in \eqref{eq:level_local_mapping}.

\myParagraph{Special case: linear least squares}
Consider the special case where the decoder $D_\theta$ in \cref{def:grid} is a linear function.
Further, assume that the cost function $c_j$ in \eqref{eq:level_local_mapping} is quadratic, \eg, $c_j(h(x_j)) =  (h(x_j) - y_j)^2$.
For instance, this would correspond to using squared norm for the SDF cost in \eqref{eq:sdf_residual}. Under these assumptions, problem \eqref{eq:level_local_mapping} is a linear least squares problem, for which we can obtain a closed-form solution from the normal equations, as shown next.

\begin{proposition}[Linear least squares]
\label{lem:linear_case}
With linear decoder $D_\theta$ and quadratic costs $c_j(h(x_j)) =  (h(x_j) - y_j)^2$,
the optimal solution to \eqref{eq:level_local_mapping} is:
\begin{equation}
    F_l^\star 
    = E(r_{1:l-1}(x))
    := - \left[
    J^\top J
    \right]\pinv
    J^\top r_{1:l-1}(x),
    \label{eq:level_local_mapping_closed_form}
\end{equation}
where $x = \{T^s_k x^k_j\}$ and $y = \{y_j\}$ collect all observed points and labels in two vectors,
$J = \partial h(x;F) / \partial F_l$ is the Jacobian matrix evaluated at $x$, 
and $r_{1:l-1}(x)$ are the residuals of prior levels, represented in vector form as,
\begin{equation}
	r_{1:l-1}(x) = h(x; F_{1:l-1}, 0_{l:L}) - y.
\end{equation}
Observe that the residual vector $r_{1:l-1}(x)$ is mapped to the least-squares solution $F_l^\star$ by a linear function, which we denote as $E(\cdot)$.
\end{proposition}

\begin{proof}
Please refer to Appendix~\ref{sec:proof}.
\end{proof}
\cref{lem:linear_case} reveals an interesting structure of the optimal initialization $F_l^\star$:
namely, it can be interpreted as a function of the prior levels' residuals $r_{1:l-1}(x)$.
We will build on this insight to approach the problem in the general case.

\begin{algorithm}[t]
	\caption{\small \textsc{Hierarchical Local SLAM}}
	\label{alg:hier_local_mapping}
	\begin{algorithmic}[1]
		\small 
		\Function{$\{T_k^s\}_k, F$ = HierarchicalLocalSLAM}{}
        \For{level $l = 1, 2, \hdots, L$}
            \State Initialize features at level $l$:
             $F_{l} \leftarrow E_{\phi_{l}}(r_{1:l-1}(x)). $
		\EndFor \label{alg:hier_local_mapping:init}
		\State From the initialized values, jointly update features $F$ and poses $\{T_k^s\}_k$ by minimizing \eqref{eq:local_mapping}.
            \State \Return $\{T_k^s\}_k$ and $F$.
		\EndFunction
	\end{algorithmic}
\end{algorithm}

\myParagraph{General case: learning hierarchical initialization}
We take inspiration from \cref{lem:linear_case} to develop a learning-based solution for the general case, where the decoder is nonlinear (\eg, an MLP) and the measurement costs are generic functions. 
Motivated by the previous insight, we propose to replace the linear mapping $E$ in \cref{lem:linear_case} with a neural network 
$E_{\phi_l}$ to approximate $F^\star_l$ from the residuals $r_{1:l-1}(x)$,
\begin{equation}\label{eq:encoder}
    F_l^\star \approx E_{\phi_l} (r_{1:l-1}(x)),
\end{equation}
where $\phi_l$ are the neural network parameters. We refer to $E_{\phi_l}$ as an \emph{encoder} due to its similarity to an encoding module used by prior works such as Convolutional Occupancy Networks \cite{peng2020convolutional} and Hierarchical Variational Autoencoder \cite{vahdat2020nvae}. 
In this work, we train a separate encoder $E_{\phi_l}$ to initialize the feature grid $F_l$ at each level $l$.
Given the learned encoder networks, we apply them to initialize the multiresolution feature grid progressively in a coarse-to-fine manner, 
before jointly optimizing all levels together with the robot trajectories, as shown in \cref{alg:hier_local_mapping}.

Next, we present the details of our encoder network for neural SDF reconstruction.
The input to the encoder is represented as a point cloud.
For each 3D position $x_j \in \Real^3$ in the submap frame, we use the following measurement residuals to construct an initial feature vector
$r_j^{\text{in}} = \begin{bmatrix}
r_{j,1}^{\text{in}} &  r_{j,2}^{\text{in}}  & r_{j,3}^{\text{in}} 
\end{bmatrix}^\top \in \Real^3$:
\begin{align}
	r_{j,1}^{\text{in}} &= \begin{cases}
	h(x_j) - y_j, & \text{if $x_j$ near surface,} \\
	0, & \text{otherwise,}
	\end{cases} 
	\label{eq:encoder_input_sdf} \\ 
	r_{j,2}^{\text{in}} &= \begin{cases}
	\max (h(x_j) - \bar{b}_j, 0), & \text{if $x_j$ in free space,} \\
	0, & \text{otherwise,}
	\end{cases}
	\label{eq:encoder_upper_bnd} \\
	r_{j,3}^{\text{in}} &= \begin{cases}
	\max (\underline{b}_j - h(x_j), 0), & \text{if $x_j$ in free space,} \\
	0, & \text{otherwise.}
	\label{eq:encoder_lower_bnd}
	\end{cases}
\end{align}
\begin{figure}[t]
    \centering
    \includegraphics[width=\linewidth]{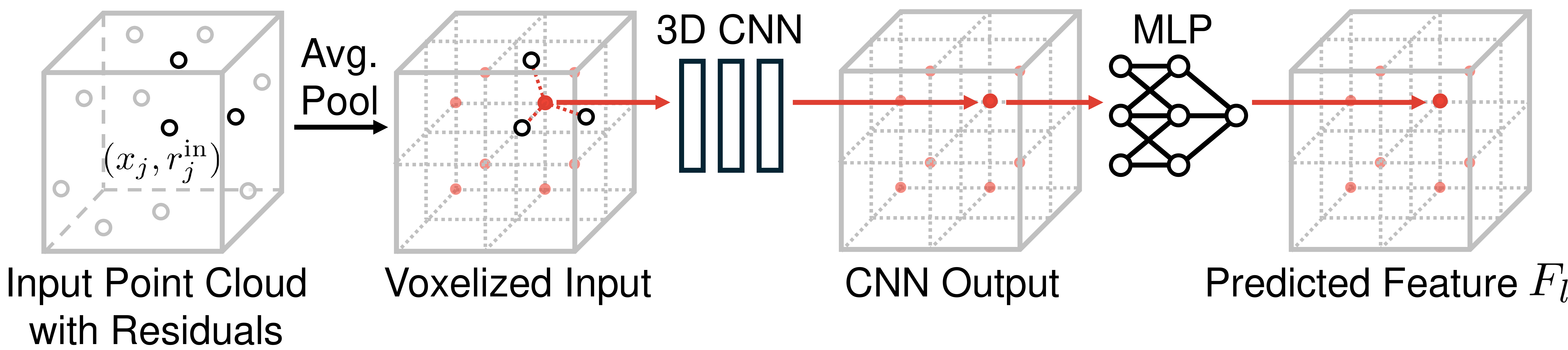}
    \caption{Illustration of the level-$l$ encoder $E_{\phi_l}$. Input point cloud with residuals $\{x_j, r_j^\text{in}\}_j$ is voxelized via averaging pooling and processed by a 3D CNN. The CNN outputs at all vertices are then transformed via a shared MLP to predict the target feature grid $F_l$. }
    \label{fig:encoder}
\end{figure}
The first feature $r^{\text{in}}_{j,1}$ corresponds to the SDF residual \eqref{eq:sdf_residual}.
The remaining two features correspond to the residuals of upper and lower bounds used to compute \eqref{eq:bnd_residual}.
The initial point features are pooled onto a 3D voxel grid with the same resolution as the target feature grid at level $l$.
Each 3D vertex stores the average residual features from points nearby.
A vertex's feature is set to be zero if there are no nearby points.
This voxelized input is then passed through a small 3D convolutional neural network (CNN).
Finally, the CNN outputs at all vertices are passed through a shared MLP to predict the target feature grid $F_l$.
\cref{fig:encoder} shows a conceptual illustration, and \cref{fig:encoder_io} in the experiments shows example inputs and predictions on a real-world dataset. 
Similar to the decoder, we train the encoders offline using submaps from multiple environments.
In particular, when training the level $l$ encoder $E_{\phi_l}$,
we use a training loss based on \eqref{eq:level_local_mapping},
\begin{equation}
\label{eq:encoder_training_loss}
\underset{{\phi_l}}{\min} 
	\sum_{s \in S} \sum_{j \in J_s} 
		c_j \left(
		h(x^s_j; F^s_{1:l-1}, E_{\phi_{l}}(r^s_{1:l-1}), 0_{l+1:L})
		\right),
\end{equation}
where $S$ contains the indices of all training submaps, $J_s$ contains the indices of all points in submap $s$, and $F^s$ denotes the features for submap $s$.
During training, we use noisy poses within the submaps to compute $x^s_j$ to account for the possible pose estimation errors at test time.

\edit{
\subsection{Extension to Incremental Processing}
\label{sec:incremental}
The local SLAM optimization formulated in \eqref{eq:local_mapping} can be performed in an incremental manner.
We describe an implementation inspired by PIN-SLAM~\cite{pan2024pin} and present corresponding evaluation on outdoor datasets in \cref{sec:experiments:newer_college}.
At each time step $k$, the received depth image or \lidar scan introduces new cost terms $\{c_j\}_{j=1}^{m_k}$ in \eqref{eq:local_mapping}.
We then alternate between tracking and mapping to update the estimated robot pose and the neural implicit submap.
During tracking, we only optimize the current robot pose $T^s_k$ and keep the submap features $F$ fixed.
Specifically, we only use surface observations and further perform voxel downsampling (voxel size $0.6$~m in \cref{sec:experiments:newer_college}).
We estimate the robot motion by minimizing \eqref{eq:local_mapping} with respect to $T^s_k$, where the pose regularization term is disabled.
Following PIN-SLAM, a Geman-McClure robust kernel is also applied to improve robustness against outlier measurements.
During mapping, we fix all pose estimates $T_{1:k}^s$ and only optimize the submap features $F$.
Voxel downsampling is similarly applied but with a smaller voxel size ($0.08$~m in \cref{sec:experiments:newer_college}).
In the incremental setting, the encoder initialization is disabled. Instead, we use the latest frame and 10 evenly spaced historical frames to update the submap features $F$ by minimizing \eqref{eq:local_mapping}.}

%% file: sections/global_mapping.tex
\section{Global Submap Alignment and Fusion}
\label{sec:global_optimization}

As the robot navigates in a large environment or for an extended time, its onboard pose estimation will inevitably drift.
To achieve globally consistent 3D reconstruction (\eg, after loop closures), it is imperative to accurately align and fuse the submaps in the global frame.
Many state-of-the-art systems, such as MIPS-Fusion \cite{tang2023mips} and Vox-Fusion++ \cite{zhai2024vox}, employ approaches that align submaps using learned SDF values.
However, this is computationally expensive and susceptible to noise.
In this section, we address this limitation by developing a hierarchical method for submap alignment and fusion, which attains significant speed-up by directly performing optimization using the features from the multiresolution submaps.

\myParagraph{Hierarchical Submap Alignment}
Consider the problem of aligning a collection of $n_s$ submaps, each represented as a multiresolution feature grid from \cref{sec:local_optimization}.
For each submap $u \in [n_s]$, we aim to optimize the submap base pose in the world frame, denoted as $T^w_u \in \SE(3)$.
The key intuition for our approach is that, for any pair of submaps to be well aligned, their implicit feature fields should also be aligned in the global frame.
We present our \emph{hierarchical} and \emph{correspondence-free} approach to exploit this intuition. 
Our method performs alignment by progressively including features at finer levels.
In the following, let $f^u_l(x) \in \Real^d$ denote the level-$l$ interpolated feature at query position $x$ in submap $u$.
Furthermore, let $f^u_{1:l}(x)$ denote the result after concatenating features up to and including level $l$, \ie, $f^u_{1:l}(x) = \bigoplus_{l'=1}^{l} f^u_{l'}(x) \in \Real^{ld}$.

Consider a pair of overlapping submaps $u,v \in [n_s]$.
Let $\{z^u_{l,i}\} \subset \Real^3$ denote the vertex positions at level $l$ in submap $u$.
Using these vertices, we define the following pairwise cost to align features,
\begin{equation}
\cfeat_l(T^w_u, T^w_v) = \sum_{i \in I^{uv}_l} 
\dist 
\bigl(
	f^u_{1:l}(z^u_{l,i}), f^v_{1:l}((T^w_v)^{-1} (T^w_u) z^u_{l,i})
\bigr).
\label{eq:latent_pairwise_cost}
\end{equation}
In \eqref{eq:latent_pairwise_cost}, $I^{uv}_l$ denotes the indices of level-$l$ vertices in submap $u$ that lie within the overlapping region of the two submaps. 
Intuitively, the right-hand side of \eqref{eq:latent_pairwise_cost} compares feature vectors interpolated from the two submaps.
The first feature comes from the source grid $u$ evaluated at its grid vertex position $z^u_{l,i}$.
To evaluate the corresponding feature in the target grid $v$, we use the submap base poses to transform the vertex position, \ie, $z^v_{l,i} = (T^w_v)^{-1} (T^w_u) z^u_{l,i}$ before querying the target feature grid.
Finally, $\dist$ denotes a distance metric in the space of implicit features.
In our implementation, we use the L2 distance, \ie, $\dist(f, f') = \norm{f - f'}^2_2$ as we find it works well empirically.
\cref{sec:experiments:ablations} presents an ablation study on alternative choices of $\dist$.

Given the pairwise alignment costs defined in \eqref{eq:latent_pairwise_cost}, 
\AlgName performs joint submap alignment by formulating and solving a problem similar to pose graph optimization.
Let $\Ecal$ denote the set of submap pairs with overlapping regions.
Then, we jointly optimize all submap poses $\{T^w_u\}_u \subset \SE(3)$
as follows.

\begin{problem}[Level-$l$ submap alignment]
\label{prob:feature_alignment}
Given $n_s$ submaps with current base pose estimates $\{\That^w_u\}_u$, 
solve for updated submap base poses via,
\begin{equation}
\underset{\{T^w_u\}_u \subset \SE(3)}{\min}
	\quad \sum_{(u,v) \in \Ecal} \cfeat_l(T^w_u, T^w_v)
	+ 
	\sum_{u=1}^{n_s} \rho(\That^w_u, T^w_u),
\label{eq:latent_pgo}
\end{equation}
where $\rho$ is the trust-region regularization defined in \eqref{eq:trust_region}.
\end{problem}

\begin{algorithm}[t]
\caption{\small \textsc{Hierarchical Submap Alignment}}
\label{alg:hier_alignment}
\begin{algorithmic}[1]
    \small 
    \Function{$\{T^w_u\}_u$ = SubmapAlignment}{}
    \State Initialize submap poses $\{T^w_u\}_u$.
    \For{level $l = 1, 2, \hdots, L$} \label{alg:hier_align:feat_start}
        \State Update $\{T^w_u\}_u$ by solving \eqref{eq:latent_pgo}
        at level $l$ for $k_{f,l}$ iters.
    \EndFor \label{alg:hier_align:feat_end}
    \State Update $\{T^w_u\}_u$ by solving \eqref{eq:sdf_pgo} for $k_s$ iters. \label{alg:hier_align:sdf}
    \State \Return $\{T^w_u\}_u$.
    \EndFunction
\end{algorithmic}
\end{algorithm}

Similar to local SLAM,  we solve \eqref{eq:latent_pgo} using PyTorch \cite{paszke2019pytorch} where the poses are updated by optimizing local corrections (represented in exponential coordinates) to the initial pose estimates $\{\That^w_u\}_u$. 
Our formulation naturally leads to a sequence of alignment problems that include features at increasingly fine levels.
We propose to solve these problems sequentially, using solutions from level $l$ as the initialization for level $l+1$; see lines~\ref{alg:hier_align:feat_start}-\ref{alg:hier_align:feat_end} in \cref{alg:hier_alignment}.

\input{tables/scannet_mae}

\input{tables/scannet_odometry}

The hierarchical, feature-based method presented above achieves robust and sufficiently accurate submap alignment.
To further enhance accuracy, we may finetune the submap pose estimates during a final alignment stage using predicted SDF values. 
Since only a few iterations are needed in typical scenarios, this approach allows us to preserve computational efficiency compared to other methods that directly use SDF for alignment.
We define the following SDF-based pairwise alignment cost for submap pair $(u,v)$,
\begin{equation}
	\csdf(T^w_u, T^w_v) = \!\!
	\sum_{j \in J^{uv}} \bigl(
	h^u (x^u_j; F^u) - h^v( (T^w_v)^{-1} T^w_u x^u_j; F^v)
	\bigr)^2.
	\label{eq:sdf_pairwise_cost}
\end{equation}
In \eqref{eq:sdf_pairwise_cost}, $J^{uv}$ contains the indices of observed points that are in the intersection region of the two submaps.
For each observation $j$, $x^u_j \in \Real^3$ denotes its position in the frame of submap $u$.
Compared to \eqref{eq:latent_pairwise_cost}, in \eqref{eq:sdf_pairwise_cost} we minimize the squared difference of the final SDF predictions from both submaps.
Using this in the pose-graph formulation leads to an SDF-based submap alignment.

\begin{problem}[SDF-based submap alignment]
\label{prob:sdf_alignment}
Given $n_s$ submaps with base pose estimates $\{\That^w_u\}_u$, 
solve for updated submap base poses via,
\begin{equation}
\underset{\{T^w_u\}_u \subset \SE(3)}{\min}
	\quad \sum_{(u,v) \in \Ecal} \csdf(T^w_u, T^w_v)
	+ 
	\sum_{u=1}^{n_s} \rho(\That^w_u, T^w_u),
\label{eq:sdf_pgo}
\end{equation}
where $\rho$ is the trust-region regularization defined in \eqref{eq:trust_region}.
\end{problem}

In \cref{alg:hier_alignment}, the SDF-based submap alignment is performed at the end to finetune the submap base poses (see line~\ref{alg:hier_align:sdf}). In \cref{sec:experiments:ablations}, we demonstrate that the combination of feature-based and SDF-based submap alignment yields the best performance in terms of both robustness and computational efficiency.

\myParagraph{Submap Fusion}
So far, we addressed the problem of aligning submaps in the global frame to reduce estimation drift. In some applications, there is an additional need to extract a global representation (\eg, a SDF or mesh) of the entire environment from the collection of local submaps. 
In \AlgName, we achieve this by using the average feature from all submaps to decode the global scene.
For any submap $u$, let $f^u(x^u)$ denote the output of its multiresolution feature field evaluated at a position $x^u \in \Real^3$ in the submap frame.
Given any query coordinate in the world frame $x^w \in \Real^3$, we first compute the weighted average of all submap features,
\begin{equation}
f^w(x^w) = 
\bigl (\sum_{u=1}^{n_s} w_u (x^w) \bigr)^{-1}
\sum_{u=1}^{n_s} w_u (x^w) f^u( (T^w_u)^{-1} x^w),
\end{equation}
where each submap is associated with a binary weight $w_u (x^w)$ computed using its bounding box,
\begin{equation*}
w_u (x^w) = \begin{cases}
1 & \text{if $x^w$ is inside submap $u$'s bounding box},\\
0 & \text{otherwise.}
\end{cases}
\end{equation*}
The final prediction is obtained by passing the average feature to the decoder network,
\begin{equation}
h^w(x^w) = D_\theta (f^w(x^w)).
\label{eq:fused_prediction}
\end{equation}
In summary, the proposed scheme achieves submap fusion via an averaging operation in the implicit feature space. 

\edit{Optionally, the fused prediction in \eqref{eq:fused_prediction} allows one to finetune the estimation by jointly optimizing all submap features and pose variables using global bundle adjustment:
\begin{equation}
\label{eq:global_ba}
\underset{
	\substack{
    F^u,\; T^w_u \in \SE(3), \\
    \{T_k^u\}_k \subset \SE(3),\; 
    \forall u \in [n_s]
  }}{\min}
	 \sum_{u=1}^{n_s}
	 \sum_{k=1}^{n_u} 
	 \sum_{j=1}^{m_k} 
        c_j\bigl(
            h^w(T^w_u T^u_k x^k_j)
        \bigr).
\end{equation}
In \eqref{eq:global_ba}, each $c_j$ is the same cost term induced by a local measurement as in \cref{sec:local_optimization}. 
Each observed local position $x^k_j$ is transformed to the world frame to evaluate the reconstruction.
The integers $n_s, n_u, m_k$ denote the number of submaps, the number of robot poses in submap $u$, and the number of measurements made at robot pose $k$, respectively.
In \cref{sec:experiments:newer_college}, we show that this global bundle adjustment step allows the method to further improve the reconstruction quality on outdoor datasets.}

%% file: tables/scannet_mae.tex
\begin{table*}[t]
\centering
\renewcommand{\arraystretch}{1.3}
\caption{Evaluation of local mapping quality for different methods on ScanNet \cite{dai2017scannet}.
\AlgName is optimized for 20 epochs and the baselines \iSDF and \Point are optimized for 100 epochs.
For each scene, we report optimization time (sec), Chamfer-L1 error (cm), and F-score (\%) computed using a threshold of $5$ cm. 
Best and second-best results are highlighted in \textbf{bold} and \underline{underline}, respectively.}
\label{tab:scannet_mae}
\resizebox{\textwidth}{!}{%
\begin{tabular}{|l|rrr|rrr|rrr|rrr|}
\hline
\multicolumn{1}{|c|}{Scene}  & \multicolumn{3}{c|}{0000}                                                                                              & \multicolumn{3}{c|}{0011}                                                                                              & \multicolumn{3}{c|}{0024}                                                                                              & \multicolumn{3}{c|}{0207}                                                                                              \\
\multicolumn{1}{|c|}{Method} & \multicolumn{1}{c}{Time$\downarrow$} & \multicolumn{1}{c}{C-l1 $\downarrow$} & \multicolumn{1}{c|}{F-score $\uparrow$} & \multicolumn{1}{c}{Time$\downarrow$} & \multicolumn{1}{c}{C-l1 $\downarrow$} & \multicolumn{1}{c|}{F-score $\uparrow$} & \multicolumn{1}{c}{Time$\downarrow$} & \multicolumn{1}{c}{C-l1 $\downarrow$} & \multicolumn{1}{c|}{F-score $\uparrow$} & \multicolumn{1}{c}{Time$\downarrow$} & \multicolumn{1}{c}{C-l1 $\downarrow$} & \multicolumn{1}{c|}{F-score $\uparrow$} \\ \hline
\iSDF \cite{ortiz2022isdf} (GT pose)               & 67.71                                & \textbf{4.77}                         & 77.87                                   & 25.89                                & {\ul 6.35}                            & {\ul 67.32}                             & 39.80                                & \textbf{4.98}                         & {\ul 73.94}                             & 26.17                                & 7.32                                  & 59.65                                   \\
\Point \cite{pan2024pin} (GT pose)              & 57.87                                & 12.23                                 & 40.89                                   & 22.41                                & 9.62                                  & 49.87                                   & 32.84                                & 10.63                                 & 46.34                                   & 21.29                                & 10.20                                 & 44.26                                   \\
MISO (GT pose)               & \textbf{1.49}                        & {\ul 4.97}                            & \textbf{81.04}                          & \textbf{0.80}                        & \textbf{6.07}                         & \textbf{71.23}                          & \textbf{0.95}                        & {\ul 5.70}                            & \textbf{73.98}                          & \textbf{0.70}                        & \textbf{6.31}                         & \textbf{69.28}                          \\
MISO (noisy pose)            & {\ul 7.96}                           & 5.36                                  & {\ul 78.43}                             & {\ul 4.11}                           & 6.85                                  & 65.21                                   & {\ul 4.79}                           & 5.89                                  & 71.61                                   & {\ul 3.35}                           & {\ul 6.86}                            & {\ul 64.59}                             \\ \hline
\end{tabular}%
}
\end{table*}

%% file: tables/scannet_odometry.tex
\begin{table*}[!t]
\centering
\vspace{-0.1cm}
\caption{\edit{Comparison between local mapping and SLAM on ScanNet \cite{dai2017scannet} using colored ICP odometry as initial guess. 
For each scene, we report translation RMSE (cm), rotation RMSE (deg), and Chamfer-L1 error (cm).}}
\label{tab:scannet_odometry}
\vspace{-0.1cm}
\resizebox{\textwidth}{!}{%
\begin{tabular}{|l|crr|crr|crr|crr|}
\hline
\multicolumn{1}{|c|}{Scene}  & \multicolumn{3}{c|}{0000}                                                                                               & \multicolumn{3}{c|}{0011}                                                                                              & \multicolumn{3}{c|}{0024}                                                                                              & \multicolumn{3}{c|}{0207}                                                                                               \\
\multicolumn{1}{|c|}{Method} & Tran err. $\downarrow$             & \multicolumn{1}{c}{Rot err. $\downarrow$} & \multicolumn{1}{c|}{C-l1 $\downarrow$} & Tran err. $\downarrow$            & \multicolumn{1}{c}{Rot err. $\downarrow$} & \multicolumn{1}{c|}{C-l1 $\downarrow$} & Tran err. $\downarrow$            & \multicolumn{1}{c}{Rot err. $\downarrow$} & \multicolumn{1}{c|}{C-l1 $\downarrow$} & Tran err. $\downarrow$             & \multicolumn{1}{c}{Rot err. $\downarrow$} & \multicolumn{1}{c|}{C-l1 $\downarrow$} \\ \hline
Color ICP~\cite{park2017colored} + Mapping             & \multicolumn{1}{r}{40.85}          & 10.17                                     & 15.81                                  & \multicolumn{1}{r}{17.25}         & 5.35                                      & 8.94                                   & \multicolumn{1}{r}{18.98}         & 5.04                                      & 10.89                                  & \multicolumn{1}{r}{19.47}          & 6.25                                      & 12.94                                  \\
Color ICP~\cite{park2017colored} + SLAM            & \multicolumn{1}{r}{\textbf{12.64}} & \textbf{5.56}                             & \textbf{10.1}                          & \multicolumn{1}{r}{\textbf{8.45}} & \textbf{3.14}                             & \textbf{7.88}                          & \multicolumn{1}{r}{\textbf{9.55}} & \textbf{3.68}                             & \textbf{8.85}                          & \multicolumn{1}{r}{\textbf{11.24}} & \textbf{4.25}                             & \textbf{9.63}                          \\ \hline
\end{tabular}%
}
\vspace{-0.3cm}
\end{table*}

%% file: sections/experiments.tex
\section{Evaluation}
\label{sec:experiments}

In this section, we evaluate \AlgName using several publicly available real-world datasets.
Our results show that \AlgName achieves superior computational efficiency and accuracy compared to state-of-the-art approaches during both local SLAM and global submap alignment and fusion.

\subsection{Experiment Setup}
\label{sec:experiments:setup}
\edit{We used four datasets in our experiments: Replica \cite{replica19arxiv}, ScanNet \cite{dai2017scannet}, \Fastcamo \cite{tang2023mips}, and Newer College~\cite{zhang2021ncdmulti}.}
Among these datasets, Replica is used to pre-train the encoders and decoder networks offline; see Appendix~\ref{sec:grid_details} for details.
Then, the pre-trained weights are used to evaluate \AlgName on the real-world ScanNet dataset and the large-scale \Fastcamo dataset without additional fine-tuning.
\edit{Lastly, we present a larger scale evaluation on sequences from the outdoor Newer College dataset.}

For quantitative evaluations, 
we compare the multiresolution submaps in \AlgName against the MLP-based representation from iSDF \cite{ortiz2022isdf} and the neural-point-based representation from PIN-SLAM \cite{pan2024pin}, using default parameters from their open-source code.
In the following, we refer to these two baselines as \iSDF and \Point, respectively.
When evaluating the performance of submap alignment,
we introduce two baseline techniques from state-of-the-art submap-based systems.
The first is the correspondence-based method from MIPS-Fusion \cite{tang2023mips}, hereafter referred to as \mips.
The second is the correspondence-free method from Vox-Fusion++ \cite{zhai2024vox}, hereafter referred to as \vfpp.
\edit{In addition, we compare against an ICP-based method introduced by Choi \etal~\cite{choi2015robust} and implemented in Open3D \cite{zhou2018open3d}.
Given the raw surface points observed in the submaps, this baseline first 
aligns pairs of submaps via point-to-plane ICP on voxel-downsampled point clouds. 
The aligned submaps are then fused in the global frame via outlier-robust pose graph optimization.
Whereas the submaps used by the neural approaches (including ours) have a fine-level resolution of 0.1~m, we allow the ICP baseline to use a higher resolution of 0.02~m and the other parameters are set to default.}
\edit{For \AlgName, we implement each submap as a two-level multiresolution feature grid with spatial resolutions [0.5~m, 0.1~m] for indoor and [1.0~m, 0.2~m] for outdoor scenes}. 
The feature dimension at each level is set to $d=4$.
All methods are implemented using PyTorch \cite{paszke2019pytorch}.
All experiments are run on a laptop equipped with an Intel i9-14900HX CPU, an NVIDIA GeForce RTX 4080 GPU, and 12 GB of GPU memory.

\subsection{Evaluation on ScanNet Dataset}
\label{sec:experiments:scannet}

ScanNet \cite{dai2017scannet} features a collection of real-world RGB-D sequences with accurate camera poses and 3D reconstructions.
In the following, we use ScanNet to separately evaluate the proposed local SLAM (\cref{sec:local_optimization}) and global alignment and fusion (\cref{sec:global_optimization}) methods.
Joint evaluation is reported in the next subsection on the larger \Fastcamo \cite{tang2023mips} datasets.

\begin{figure}[t]
    \centering
    \begin{subfigure}[t]{0.25\linewidth}
        \centering
        \includegraphics[trim=10 0 10 0, clip, width=\linewidth]{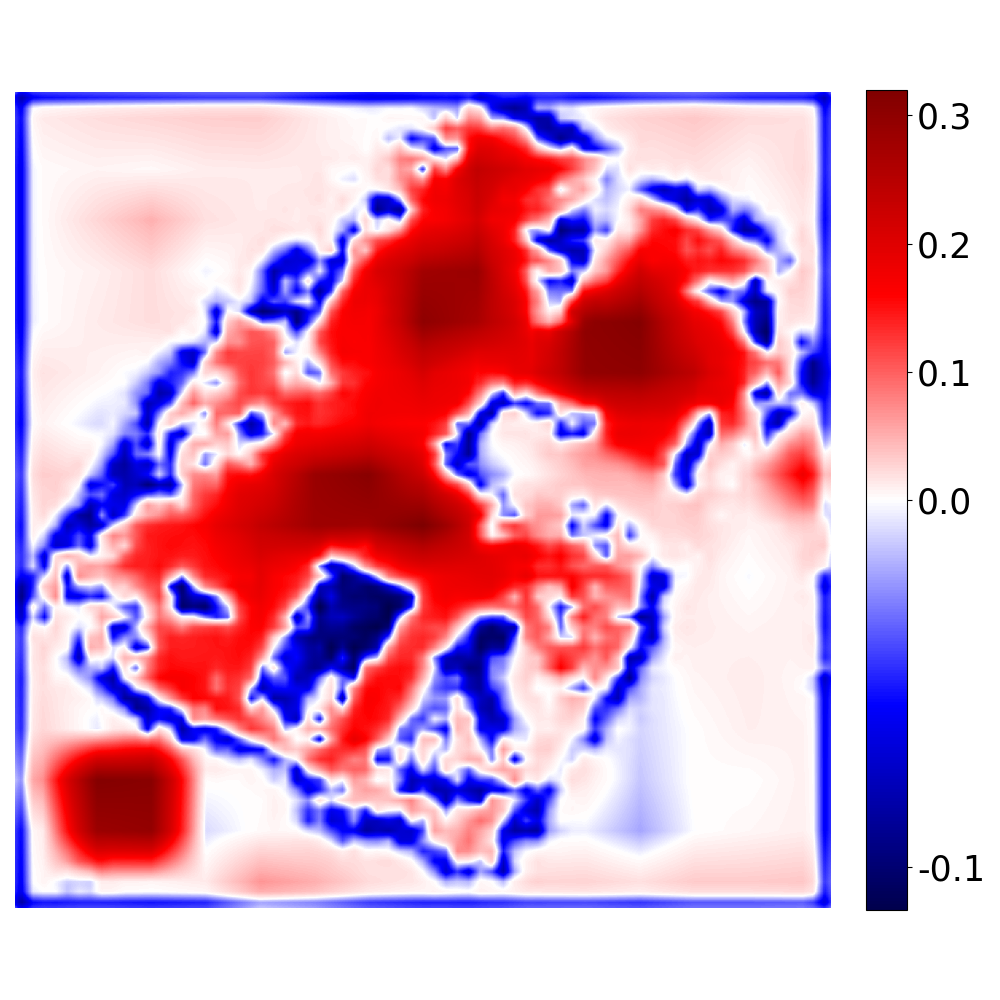}
        \caption{\AlgName (init)}
        \label{fig:scannet_mapping:init}
    \end{subfigure}%
    \begin{subfigure}[t]{0.25\linewidth}
        \centering
        \includegraphics[trim=10 0 10 0, clip, width=\linewidth]{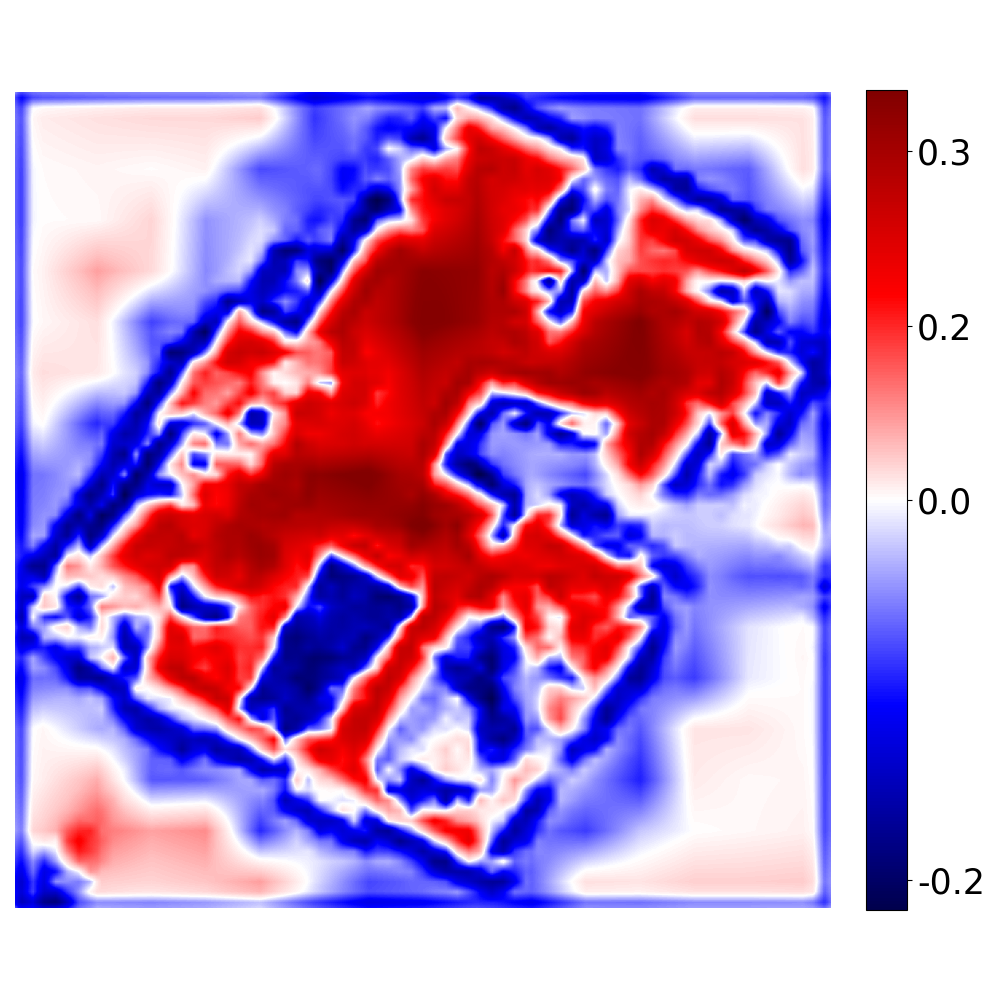}
        \caption{\AlgName (opt)}
        \label{fig:scannet_mapping:opt}
    \end{subfigure}%
    \begin{subfigure}[t]{0.25\linewidth}
        \centering
        \includegraphics[trim=10 0 10 0, clip, width=\linewidth]{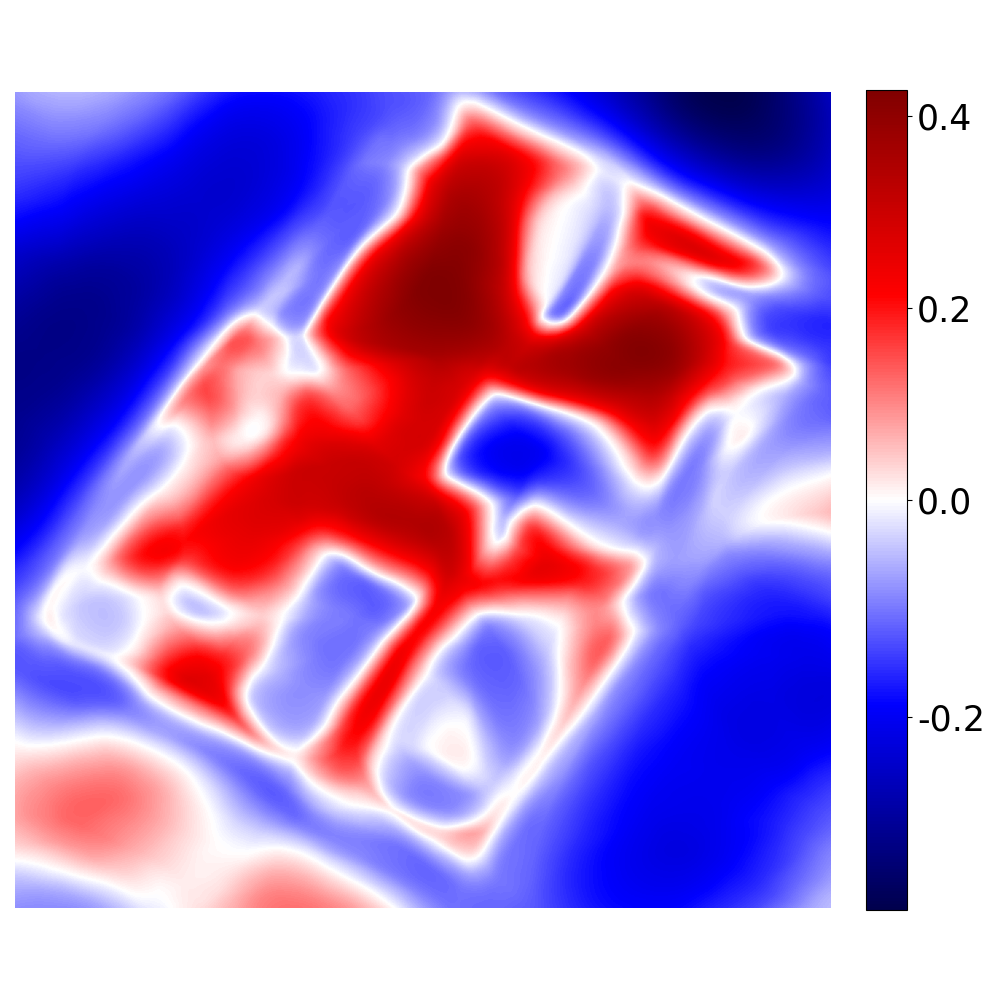}
        \caption{\iSDF}
        \label{fig:scannet_mapping:isdf}
    \end{subfigure}%
    \begin{subfigure}[t]{0.25\linewidth}
        \centering
        \includegraphics[trim=10 0 10 0, clip, width=\linewidth]{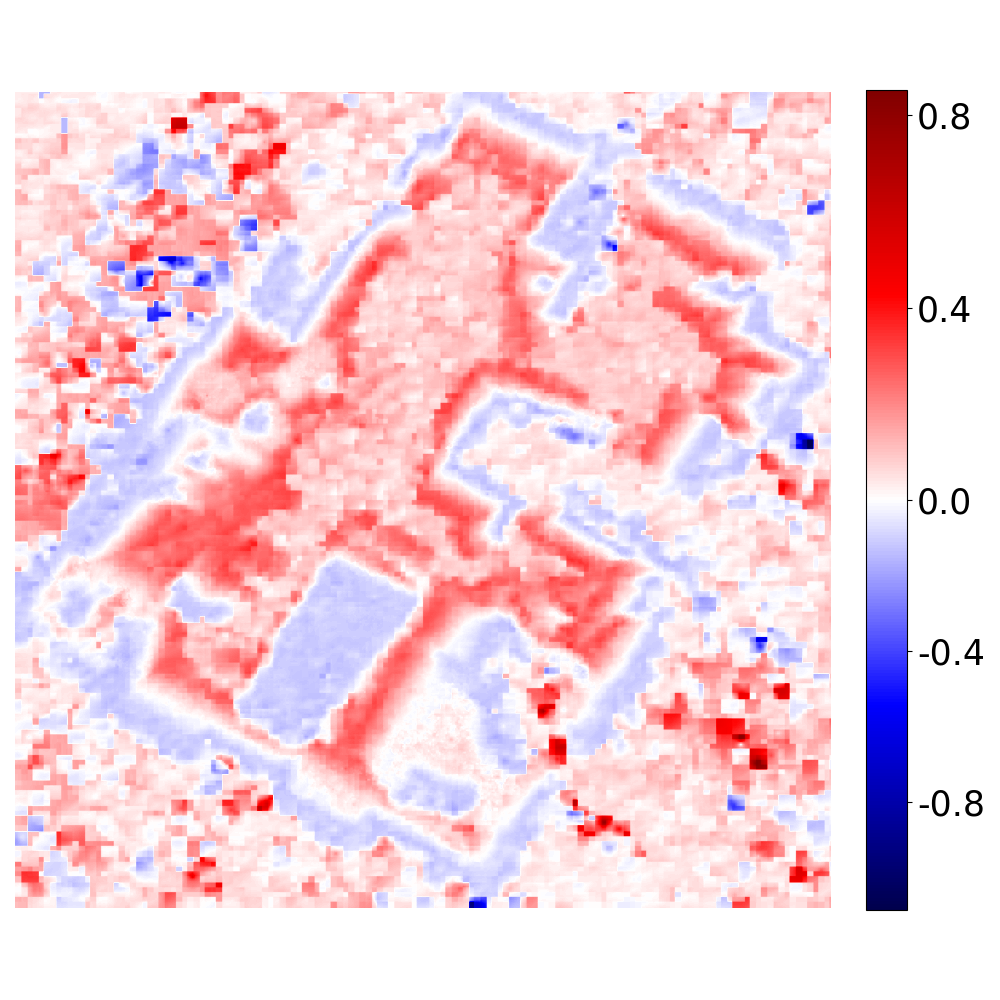}
        \caption{\Point}
        \label{fig:scannet_mapping:pin}
    \end{subfigure}
    \caption{Visualization of estimated SDF at a fixed height on {ScanNet} scene 0207. \AlgName performs SLAM using noisy poses. \iSDF and \Point use ground truth poses and only perform mapping. }
    \label{fig:scannet_mapping}
\end{figure}

\myParagraph{Local SLAM evaluation}
Our first experiment evaluates the performance of the local optimization approach in \AlgName (\cref{alg:hier_local_mapping}).
Since each scene in ScanNet is relatively small, we represent the entire scene as a single submap.
For each scene, we run \AlgName from initial pose estimates obtained by perturbing the ground truth poses with $3$~deg and $5$~cm errors.
For comparison, we also run \AlgName and the baseline \iSDF and \Point methods using ground truth poses, where pose optimization is disabled. 
\cref{tab:scannet_mae} reports results on four ScanNet scenes.
For each method, we report its GPU time and the mesh reconstruction error against the ground truth, measured in Chamfer-L1 distance and F-score.
Both \AlgName variants are optimized for 20 epochs.
For the \iSDF and \Point baselines, since they do not have access to pre-training, we optimize both for 100 epochs for a fair comparison.
As shown in \cref{tab:scannet_mae}, \AlgName achieves either the best or second-best reconstruction results on all scenes.
Using ground truth poses, \AlgName achieves superior speed, requiring only $0.7$–$1.5$~sec for optimization.
\AlgName with noisy poses takes longer due to the additional pose estimation but is still significantly faster than the baseline techniques.

\begin{figure}[t]
    \centering
    \begin{subfigure}[t]{0.24\linewidth}
        \centering
        \includegraphics[trim=0 0 0 0, clip, width=\linewidth]{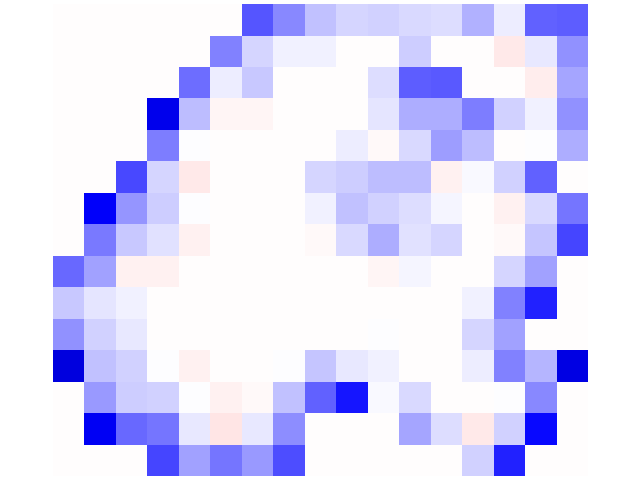}
        \caption{Level 1 $r^\text{in}_1$}
    \end{subfigure}
    \begin{subfigure}[t]{0.24\linewidth}
        \centering
        \includegraphics[trim=0 0 0 0, clip, width=\linewidth]{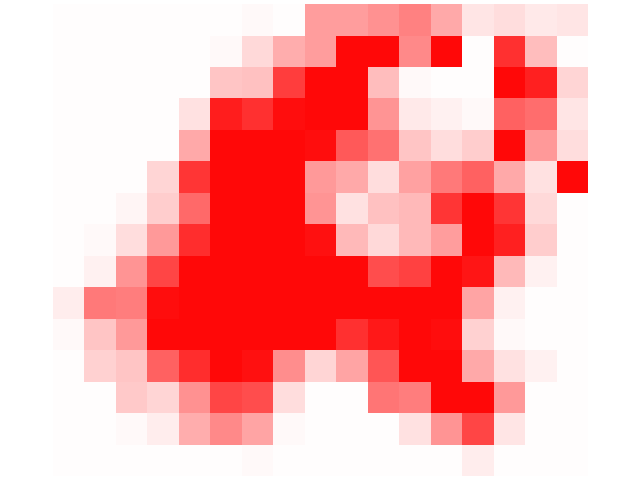}
        \caption{Level 1 $r^\text{in}_3$}
    \end{subfigure} 
    \begin{subfigure}[t]{0.22\linewidth}
        \centering
        \includegraphics[trim=0 0 0 0, clip, width=\linewidth]{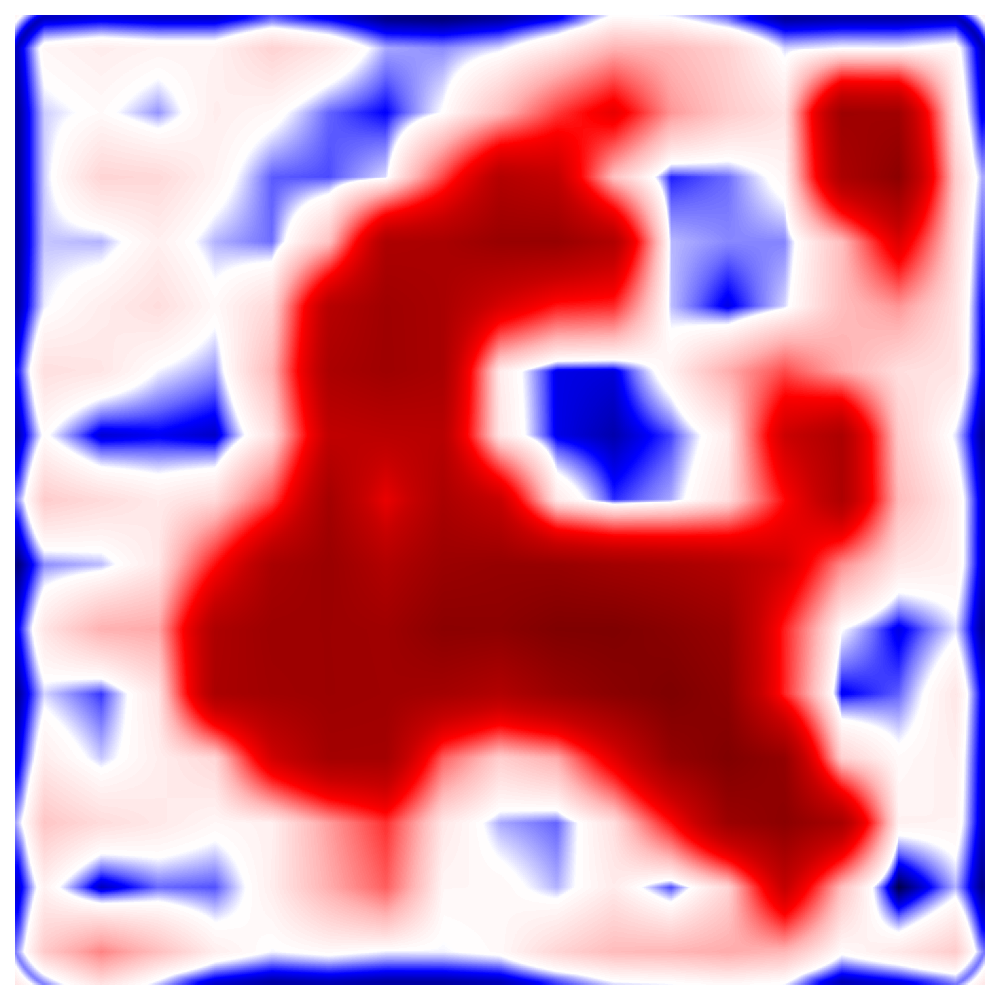}
        \caption{Level 1 SDF}
    \end{subfigure}
    \begin{subfigure}[t]{0.22\linewidth}
        \centering
        \includegraphics[trim=400 0 400 0, clip, width=\linewidth]{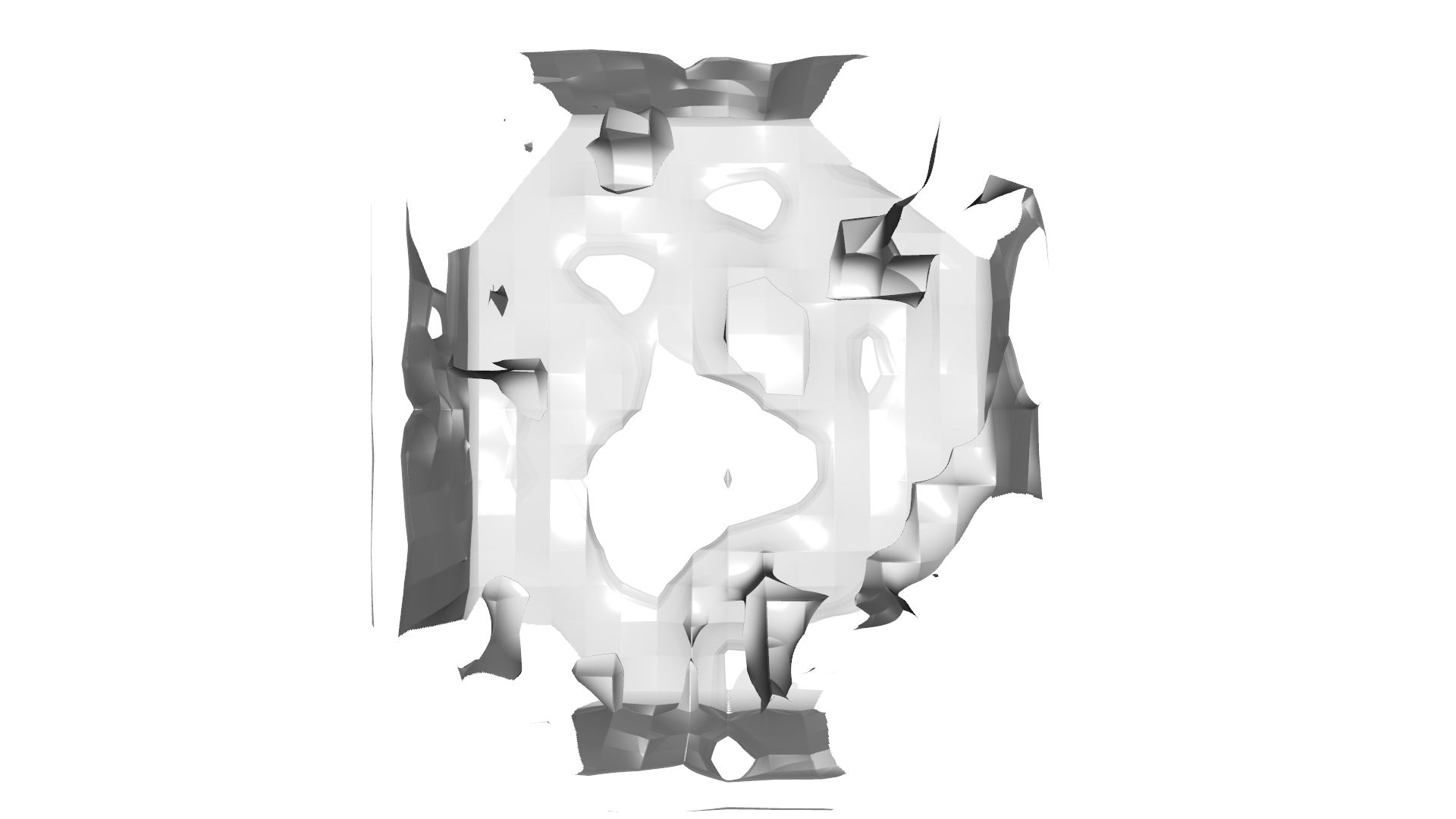}
        \caption{Level 1  mesh}
    \end{subfigure}
    \\
    \begin{subfigure}[t]{0.24\linewidth}
        \centering
        \includegraphics[trim=0 0 0 0, clip, width=\linewidth]{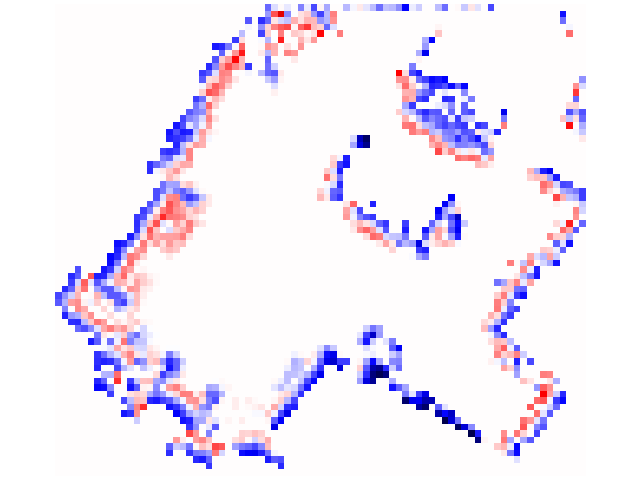}
        \caption{Level 2 $r^\text{in}_1$}
    \end{subfigure}
    \begin{subfigure}[t]{0.24\linewidth}
        \centering
        \includegraphics[trim=0 0 0 0, clip, width=\linewidth]{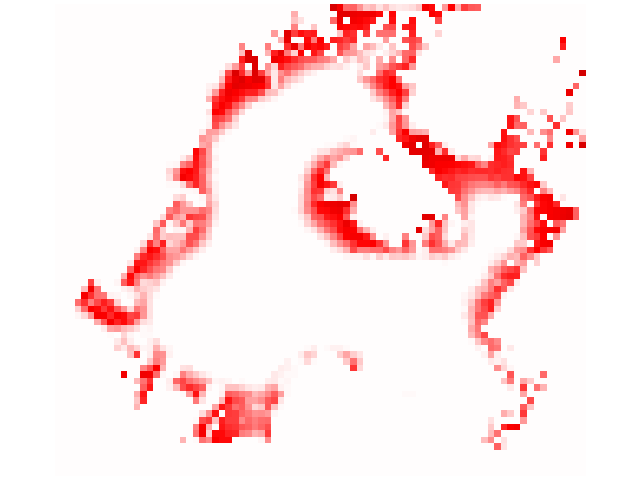}
        \caption{Level 2 $r^\text{in}_3$}
    \end{subfigure} 
    \begin{subfigure}[t]{0.22\linewidth}
        \centering
        \includegraphics[trim=0 0 0 0, clip, width=\linewidth]{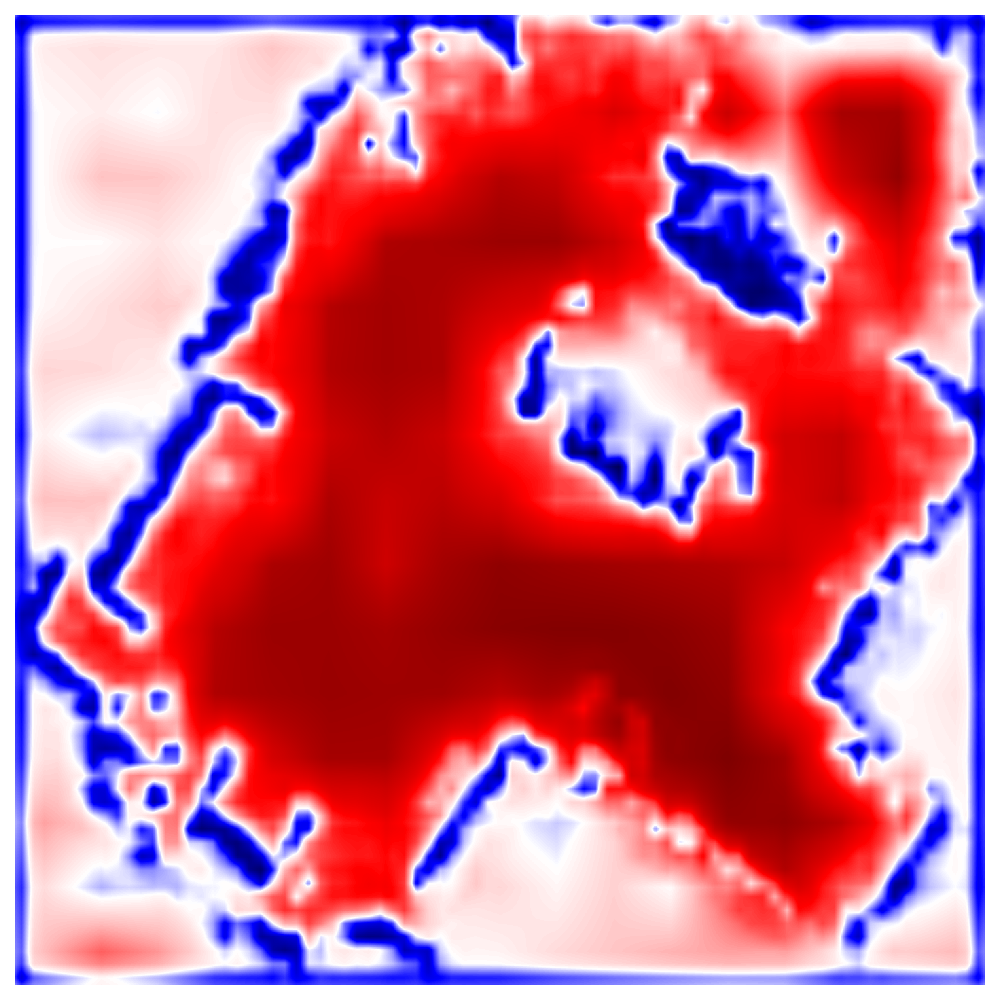}
        \caption{Level 2 SDF}
    \end{subfigure}
    \begin{subfigure}[t]{0.22\linewidth}
        \centering
        \includegraphics[trim=400 0 400 0, clip, width=\linewidth]{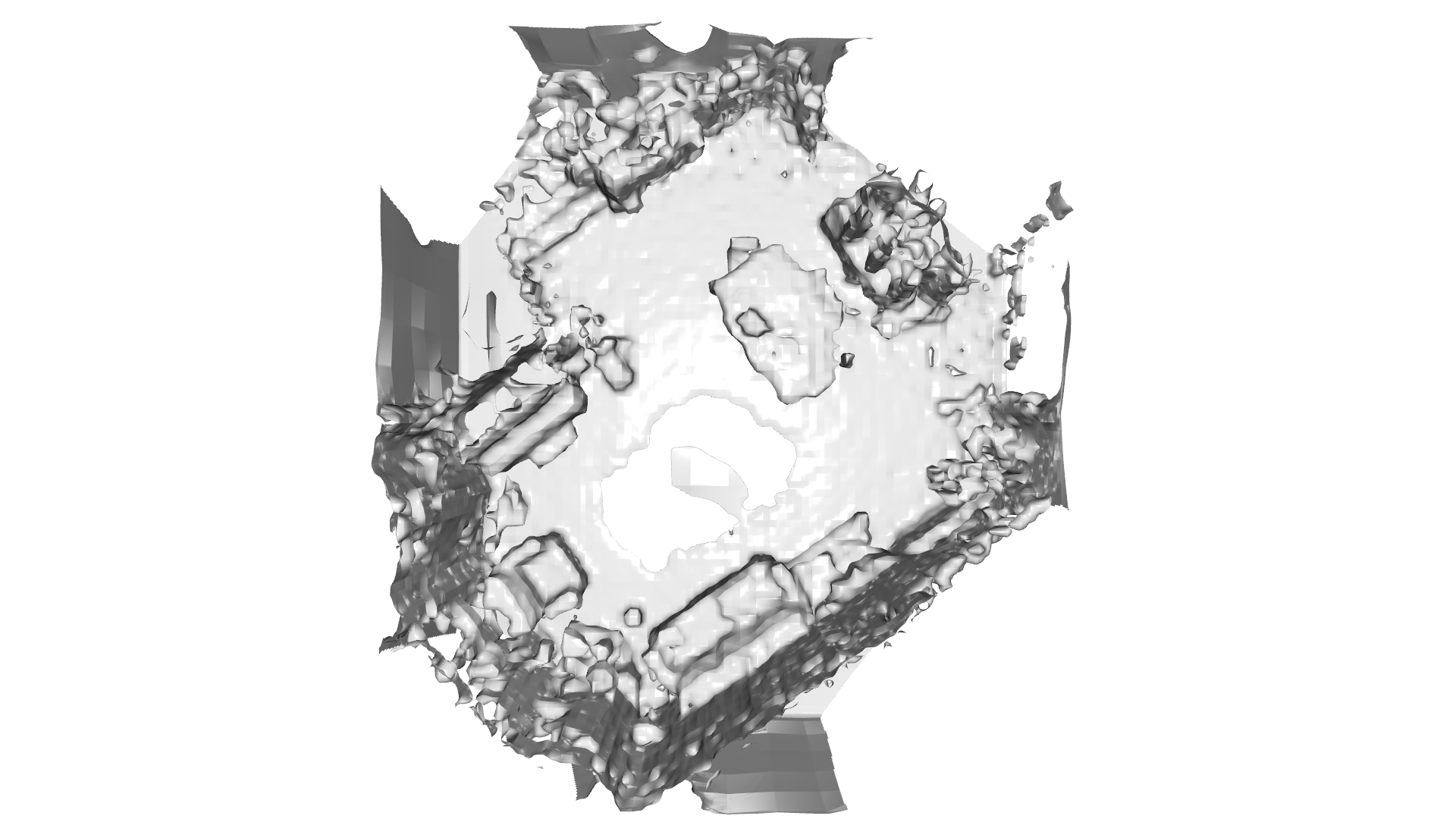}
        \caption{Level 2  mesh}
    \end{subfigure}
    \caption{Inputs and output predictions from learned hierarchical encoders on ScanNet scene 0024. Red and blue show positive and negative residual or SDF values, respectively. }
    \label{fig:encoder_io}
\end{figure}

\Cref{fig:scannet_mapping} shows a qualitative comparison of the estimated SDF at a fixed height on scene 0207.
\AlgName starts from noisy poses and performs full SLAM, while the baseline methods use ground truth poses and perform mapping only.
\cref{fig:scannet_mapping:init} shows the initialization produced by the learned encoder (corresponding to line~\ref{alg:hier_local_mapping:init} in \cref{alg:hier_local_mapping}), which already captures the scene geometry to a large extent.
The remaining errors and missing details are fixed after running 20 optimization epochs, as shown in \cref{fig:scannet_mapping:opt}.
Because \iSDF uses a single MLP to represent the scene, its output (see \cref{fig:scannet_mapping:isdf}) is overly smooth, which leads to lower recall and F-score than \AlgName.
Lastly, the SDF prediction from \Point (\cref{fig:scannet_mapping:pin}) is especially noisy in free space, due to the lack of neural point features far away from the surface.

To provide more insight on the performance of the learned hierarchical encoders, 
\cref{fig:encoder_io} visualizes the inputs and output predictions on scene 0024.
For this visualization, we do not use noisy pose estimates.
At each level (shown as a row in the figure), we visualize two channels $r^\text{in}_1, r^\text{in}_3$ of the voxelized input residuals defined in \eqref{eq:encoder_input_sdf} and \eqref{eq:encoder_lower_bnd}, the predicted SDF, and the predicted 3D mesh.
The level 1 encoder, although with a coarse resolution of $0.5$~m, already captures the rough scene geometry and free space SDF. 
On top of this coarse prediction, the level 2 encoder is able to add fine details and recover objects such as the sofa and the table.

\input{tables/scannet_align}

\edit{Lastly, to evaluate local SLAM under more challenging initial trajectory estimates, 
we include an experiment with color ICP odometry~\cite{park2017colored} as the initial guess.
When running local SLAM, we let \AlgName incrementally process a sliding window of 10 frames and disable pose regularization in \eqref{eq:local_mapping}.
\cref{tab:scannet_odometry} shows that performing local SLAM substantially improves performance compared to mapping only.
}

\myParagraph{Submap alignment and fusion evaluation}
The next set of experiments evaluates the submap alignment and fusion approach  in \AlgName.
We run MIPS-Fusion \cite{tang2023mips} to obtain the submap information.
We select scene 0011 and perform local SLAM within all submaps starting from noisy poses with $1$~deg and $0.1$~m errors.
We then evaluate the performance of the proposed alignment (\cref{alg:hier_alignment}) and the baseline \edit{MIPS \cite{tang2023mips}, VFPP~\cite{zhai2024vox}, and ICP~\cite{choi2015robust}} techniques under increasing submap alignment errors.
All methods are run for 100 iterations before evaluating their results.
For \AlgName, we allocate 45 iterations for alignment at each feature level and 10 iterations for final alignment using SDF, which corresponds to setting $k_{f,1}=k_{f,2}=45$ and $k_s=10$ in \cref{alg:hier_alignment}.
\edit{All neural methods} use the same trust-region-based pose regularization, where the trust-region radius $\tau$ in \eqref{eq:trust_region} is set based on the initial submap alignment error.
For each setting of the initial error shown on the $x$ axis, we perform 10 random trials and show the final results as boxplots in \cref{fig:scannet_align}.
With a small initial alignment error of $1$~deg and $0.1$~m, all methods produce accurate alignment results.
However, as the initial error increases, both MIPS and VFPP methods quickly fail, showing that solely relying on SDF prediction is very sensitive to the initial guess.
\edit{The ICP baseline results in higher variance indicating more frequent alignment failures despite the use of robust optimization.}
The proposed hierarchical alignment in \AlgName is more robust under more severe misalignments and outperforms the baseline methods significantly.
\edit{Since our experiments are offline, we record the total GPU time and final accuracy. 
Timing results for ICP are omitted as the method is not implemented on GPU.
\cref{tab:scannet_align} reports the results on scenes 0011 and 0024.}
The initial submap alignment error is set to $5$~deg and $0.2$~m, and results are collected over 10 trials.
\AlgName is significantly faster since the majority of iterations in \cref{alg:hier_alignment} only uses features for alignment and does not need to decode the SDF values.
In terms of accuracy, \AlgName demonstrates competitive performance and achieves either best or second-best results.

\begin{figure}[t]
    \centering
    \begin{subfigure}[t]{0.95\linewidth}
        \centering
        \includegraphics[trim=0 10 0 0, clip,width=\linewidth]{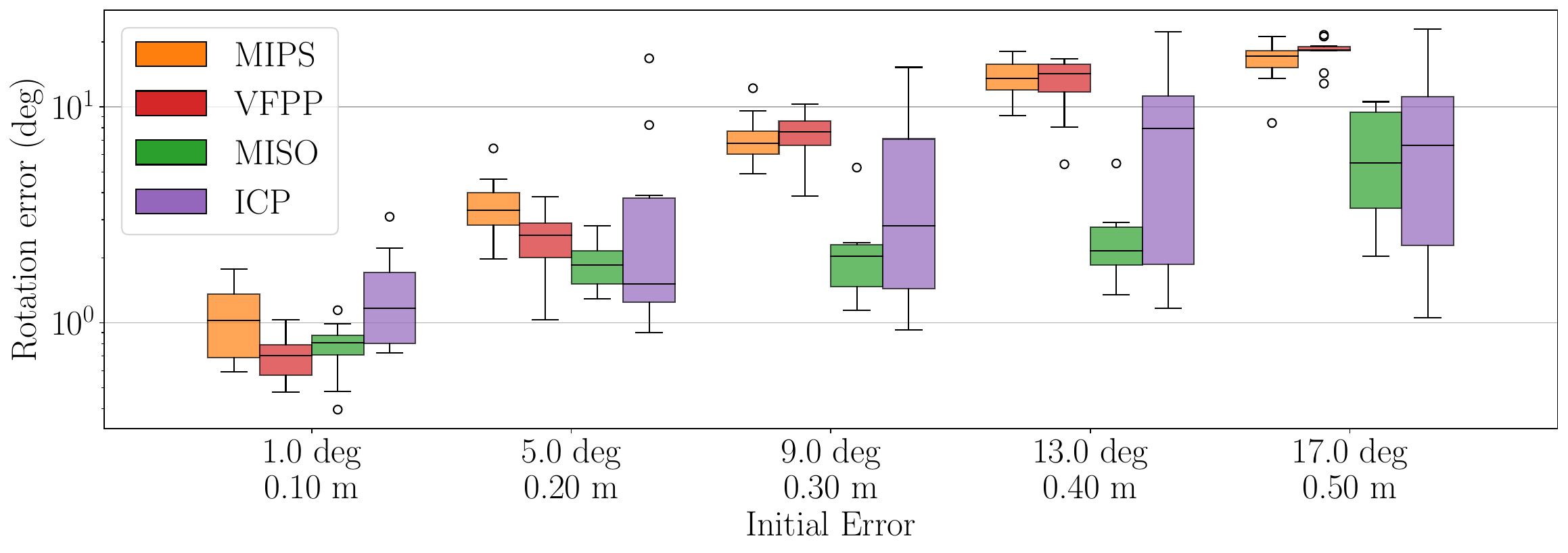}
        \caption{Rotation error (deg)}
        \label{fig:scannet_align:rot}
    \end{subfigure}
    \\
    \begin{subfigure}[t]{0.95\linewidth}
        \centering
        \includegraphics[trim=0 10 0 0, clip,width=\linewidth]{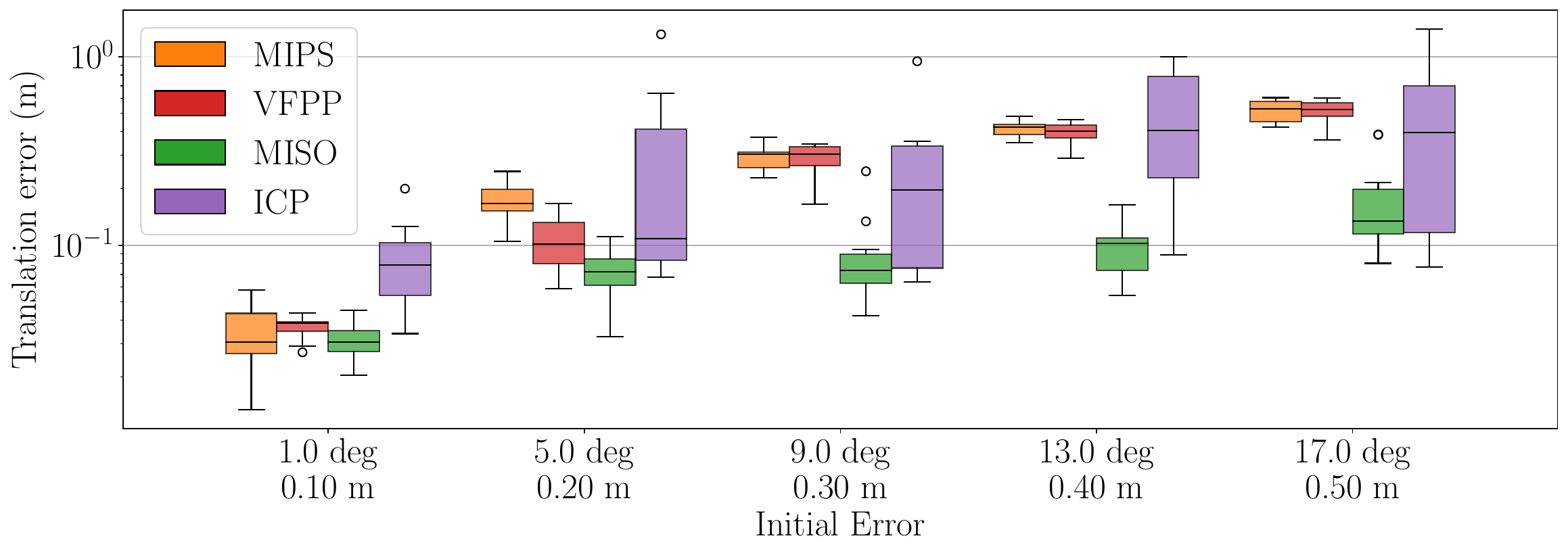}
        \caption{Translation error (m)}
        \label{fig:scannet_align:tran}
    \end{subfigure}
    \caption{Evaluation of submap alignment on ScanNet scene 0011 under varying initial errors over 10 trials for each configuration.}
    \label{fig:scannet_align}
\end{figure}

\begin{figure}[t]
    \centering
    \begin{subfigure}[t]{0.48\linewidth}
        \centering
        \includegraphics[trim=250 10 300 20, clip,width=\linewidth]{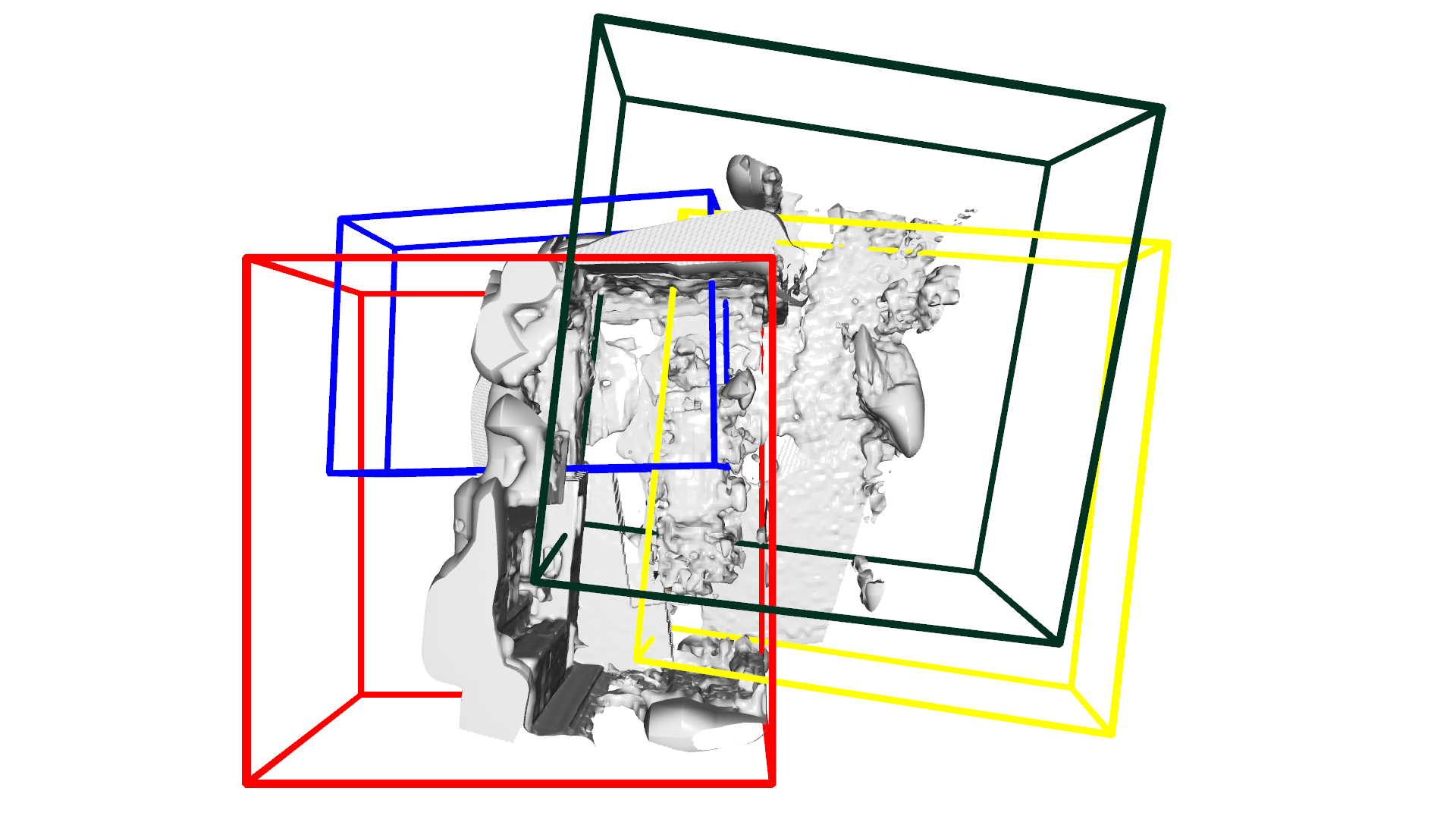}
        \caption{Before alignment}
        \label{fig:scannet_fusion:mesh_pre}
    \end{subfigure} 
    ~
    \begin{subfigure}[t]{0.48\linewidth}
        \centering
        \includegraphics[trim=250 10 300 20, clip,width=\linewidth]{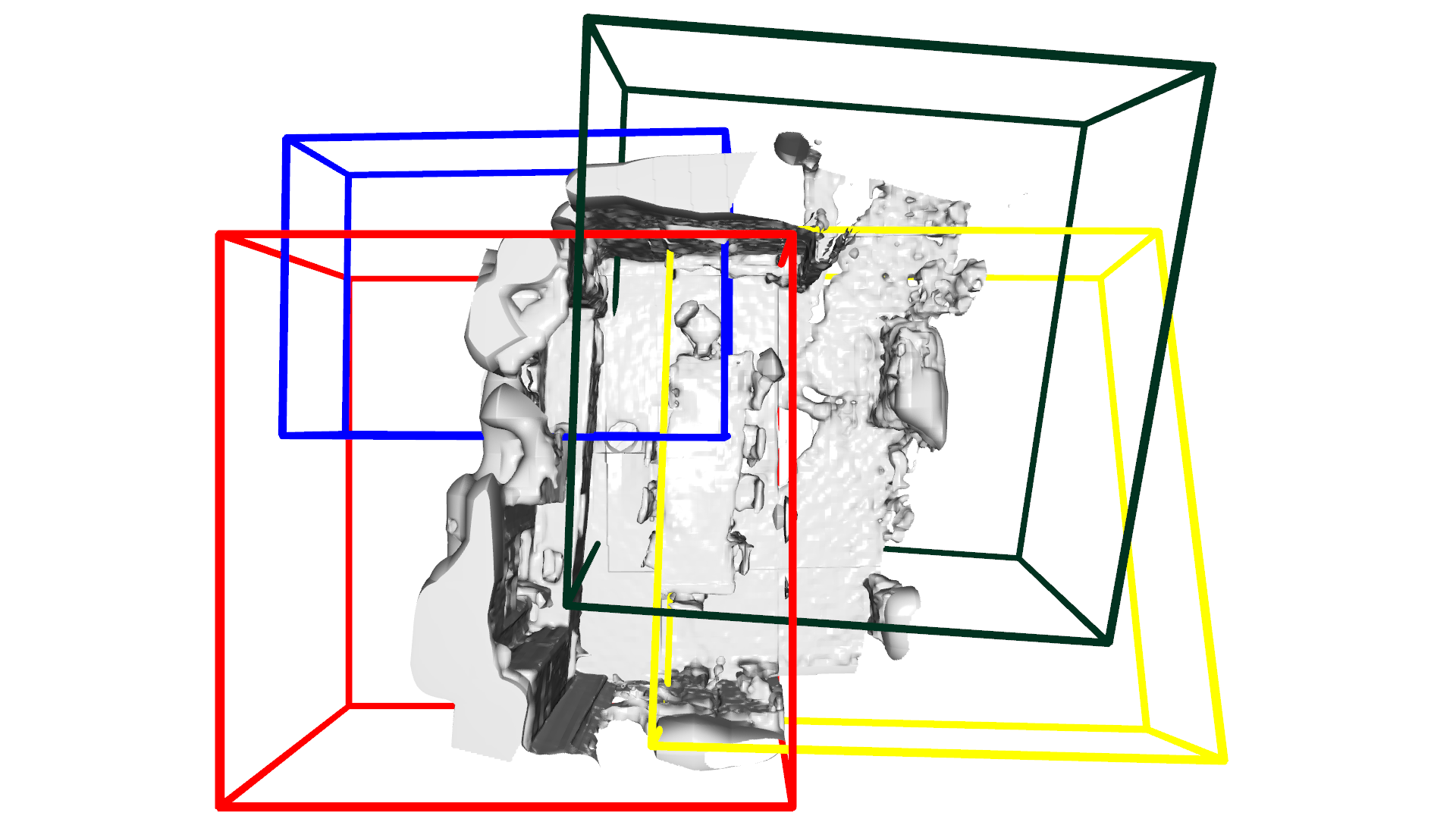}
        \caption{After alignment}
        \label{fig:scannet_fusion:mesh_post}
    \end{subfigure} 
    \caption{
    3D mesh reconstruction with oriented submap bounding boxes: (a) before submap alignment and (b) after submap alignment on ScanNet scene 0011.}
    \label{fig:scannet_fusion}
\end{figure}

\input{tables/ncd}

\Cref{fig:scannet_fusion} visualizes the reconstructed global mesh on scene 0011 before and after submap alignment.
To extract the global SDF, we use the approach presented at the end of \cref{sec:global_optimization}, which uses the average feature from all submaps to decode the scene in the world frame.
Along with the mesh, we also show the oriented bounding boxes of all submaps.
As shown in \cref{fig:scannet_fusion:mesh_pre}, the initial mesh is very noisy in areas where multiple submaps overlap.
\Cref{fig:scannet_fusion:mesh_post} shows that \AlgName effectively fixes this error and recovers clean geometry (\eg, the table and chairs) in the global frame.

\subsection{\edit{Qualitative Evaluation on \Fastcamo Dataset}}
\label{sec:experiments:fastcamo}

\input{sections/fastcamo_figure}

\Fastcamo \cite{tang2023mips} is a recent dataset with multiple real-world RGB-D sequences.
Compared to ScanNet, this dataset features larger indoor environments (up to $200$ m$^2$) and fast camera motions, making it suitable for evaluating the overall accuracy and scalability of our method.
\edit{Since ground truth is not available, 
we only focus on qualitative evaluation in this experiment.}
We compare the performance of \AlgName against \iSDF, where the latter is extended to perform joint optimization over both robot poses and the map.
The input pose estimates are obtained by perturbing robot poses estimated by MIPS-Fusion \cite{tang2023mips} by $3$~deg rotation error and $0.05$~m translation error in their corresponding submaps, and additionally perturbing all submap base poses by $5$~deg rotation error and $0.1$~m translation error.
The trust region radius used in pose regularization \eqref{eq:trust_region} is set correspondingly for both methods.
For \iSDF, we run optimization for 300 epochs to ensure  convergence. 
For \AlgName, we run 150 epochs of local SLAM in parallel within all submaps,
followed by submap alignment where we skip the coarse level feature alignment as we observe it leads to degraded results on these datasets.
\cref{fig:fastcamo} shows qualitative comparisons of the estimated trajectories and meshes on the Floors-II, Stairs-I, and Stairs-II sequences.
\Cref{fig:demo} shows more visualizations for \AlgName on the Floors-I sequence.
\AlgName is able to accurately represent these larger-scale scenes while preserving fine details such as staircases and chairs.
In comparison, \iSDF loses many details and introduces more artifacts in areas where there are fewer observations.

\begin{figure}[t]
    \centering
    \includegraphics[trim=0 100 0 30, clip,width=\linewidth]{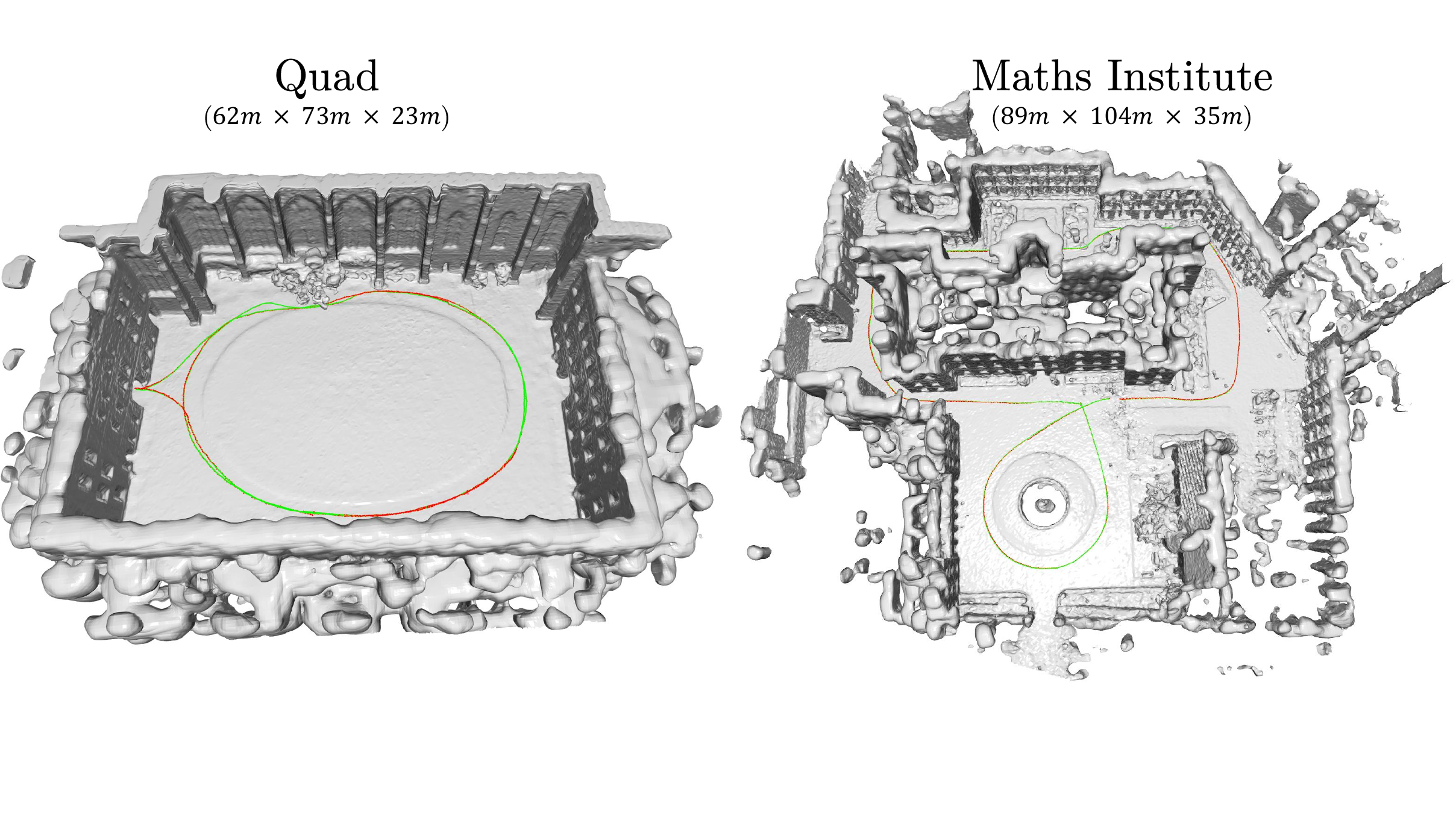}
    \caption{\edit{Qualitative evaluation on the Newer College dataset \cite{zhang2021ncdmulti}. We show the final reconstructed mesh along with estimated and reference robot trajectories in red and green, respectively.}}
    \label{fig:ncd}
\end{figure}

\edit{
\subsection{Evaluation on Newer College Dataset}
\label{sec:experiments:newer_college}

In this section, we evaluate \AlgName on the Quad-Easy and Maths-Easy sequences from the outdoor, \lidar-based Newer College Dataset \cite{zhang2021ncdmulti}.
For comparison, we run PIN-SLAM~\cite{pan2024pin}, KISS-ICP~\cite{vizzo2023kiss}, and point-to-point ICP in Open3D~\cite{zhou2018open3d} using their open-source implementations and default hyperparameters.
For KISS-ICP and point-to-point ICP, we obtain corresponding maps for evaluation by running our mapping pipeline with the pose estimates fixed.
For \AlgName, we first pre-train a decoder for each scene and then use the incremental implementation presented in \cref{sec:incremental} to process each sequence.
Point-to-point ICP is used to initialize the tracking optimization at every scan, and 
a new submap is created every 400 scans.
We evaluate two variants of our method in this experiment.
The first (\method{Odometry}) only performs incremental tracking and mapping.
The second (\method{Full}) also performs submap alignment and fusion as described in \cref{sec:global_optimization}, and its results are visualized in \cref{fig:ncd}.
\cref{tab:ncd} reports translation RMSE (m), rotation RMSE (deg), Chamfer L1 distance (cm), and F-score (\%).
In the Quad scene, our method achieves comparable performance with other state-of-the-art methods.
In the more challenging Maths Institute scene, \AlgName performs slightly worse than PIN-SLAM but still on par with KISS-ICP.
On both sequences, our results demonstrate that \AlgName substantially improves over the point-to-point ICP initialization. 
Further, the comparison between \method{Odometry} and \method{Full} validates the effectiveness of the proposed submap optimization approach to achieve global consistency.
}

\subsection{Ablation Studies}
\label{sec:experiments:ablations}

We conclude the experiments with ablation studies to provide more insight into several design choices in \AlgName.

\myParagraph{Ablation on encoder and decoder pre-training}
The first ablation study investigates the performance improvements brought by pre-training the encoder and decoder networks in \AlgName.
For this, we use the same setup as the local mapping experiments in \cref{tab:scannet_mae}.
For each scene, we report the F-score (\%) calculated with a 5 cm threshold and the training loss. The results are presented at epochs 10 and 100 to demonstrate both early-stage and later-stage training performance.

In \cref{tab:scannet_init}, we compare the following variants of \AlgName: 
\begin{itemize}
	\item \method{No-ED}: \cref{alg:hier_local_mapping} without pre-trained decoder or encoder initialization. The grid features are initialized from a normal distribution with standard deviation $10^{-4}$. Both the grid features and the decoder are optimized jointly.
        \item \method{No-E}: \cref{alg:hier_local_mapping} using only pre-trained decoder and no encoder initialization. The grid features are initialized from a normal distribution with standard deviation $10^{-4}$.
        \item \method{Full}: default \cref{alg:hier_local_mapping} using both pre-trained decoder and encoder initialization.
\end{itemize}
As shown in \cref{tab:scannet_init}, our hierarchical initialization scheme significantly speeds up training loss optimization, particularly when both the pre-trained encoder and decoder are used (see rows 1 and 3). Even with only the pre-trained decoder, optimization remains faster than training from scratch (rows 1 and 2). As expected, after sufficient training (\ie, at 100 epochs), all three methods converge in terms of training loss. Notably, using the pre-trained decoder yields better performance than training from scratch.
\edit{We attribute the slight degradation of \method{Full} compared to \method{No-E} on scene 0011 (100 epochs) to domain mismatch since our encoders are pre-trained on synthetic Replica scenes only. 
We do indeed observe Replica-like artifacts in the \method{Full} results. 
With greater realism and diversity in pre-training, we expect \method{Full} to further improve.
Overall, \cref{tab:scannet_init} still shows clear benefits of pre-trained encoders especially at early stages (10 epochs).}

\input{tables/scannet_init}

\myParagraph{Ablation on hierarchical alignment}
To investigate the effects of each alignment stage in \cref{alg:hier_alignment}, we conduct ablation experiments on ScanNet scenes 0011 and 0207, each containing four submaps. We use the same setup as in the submap alignment experiments. In \cref{tab:scannet_hier}, we compare the following variants of \AlgName:

\begin{itemize}
	\item \method{Coarse only}: \cref{alg:hier_alignment} with only coarse grid ($l=1$) feature-based alignment.
	\item \method{Coarse+Fine}:  \cref{alg:hier_alignment} with both coarse grid ($l=1$) and fine grid ($l=2$) feature-based alignment.
	\item \method{Full}: default \cref{alg:hier_alignment} with both coarse grid ($l=1$) and fine grid ($l=2$) feature-based alignment, and a final alignment stage using predicted SDF values.
\end{itemize}

In \cref{tab:scannet_hier}, one can see that the full hierarchical approach (coarse, fine, and SDF) achieves the smallest rotation and translation errors, whereas using only the coarse alignment yields the largest errors. As expected, the full approach requires more computation time as the last stage based on SDF alignment requires using the decoder to predict SDF values. We do not report the F-score for \method{No-ED} at epoch 10 because it fails to produce a mesh due to insufficient training.

\input{tables/scannet_hier}

\myParagraph{Ablation on metric function}
We study how different metric functions used in the feature-based alignment stage (\ie, $\dist$ in \eqref{eq:sdf_pairwise_cost}) affect alignment accuracy.  Specifically, we compare negative cosine similarity, L1 distance, and L2 distance under three progressively increasing initial rotation and translation errors in scene 0011. Our results in \cref{fig:ablation_metric} show that L2 distance generally yields the most accurate submap alignment, while negative cosine similarity tends to produce the largest misalignment.

%% file: tables/scannet_align.tex
\begin{table}[t]
\centering
\renewcommand{\arraystretch}{1.3}
\caption{Submap alignment evaluation on ScanNet \cite{dai2017scannet} from initial errors of $5$~deg and $0.20$~m.
For each scene, we report the GPU time (sec), and the final submap rotation error (deg) and translation error (m) compared to ground truth. Results averaged over 10 trials. Best and second-best results are highlighted in \textbf{bold} and \underline{underline}.}
\label{tab:scannet_align}
\resizebox{\linewidth}{!}{%
\begin{tabular}{|l|rrr|rrr|}
\hline
\multicolumn{1}{|c|}{Scene} & \multicolumn{3}{c|}{0011 (159 poses, 4 submaps)} & \multicolumn{3}{c|}{0024 (227 poses, 5 submaps)} \\
\multicolumn{1}{|c|}{Methods} & \multicolumn{1}{c}{Time$\downarrow$} & \multicolumn{1}{c}{Rot err$\downarrow$} & \multicolumn{1}{c|}{Tran err$\downarrow$} & \multicolumn{1}{c}{Time$\downarrow$} & \multicolumn{1}{c}{Rot err$\downarrow$} & \multicolumn{1}{c|}{Tran err$\downarrow$} \\ \hline
MIPS \cite{tang2023mips} & {\ul 59.23} & 3.93 & 0.15 & {\ul 123.63} & 4.52 & 0.19 \\
VFPP \cite{zhai2024vox} & 74.16 & {\ul 2.40} & {\ul 0.11} & 137.58 & \textbf{2.24} & \textbf{0.08} \\
MISO & \textbf{10.53} & \textbf{1.86} & \textbf{0.06} & \textbf{18.85} & {\ul 3.14} & {\ul 0.14} \\
ICP \cite{choi2015robust} & - & 2.62 & 0.19 & - & 9.18 & 0.53 \\ \hline
\end{tabular}%
}
\end{table}

%% file: tables/ncd.tex
\begin{table*}[t]
\centering
\renewcommand{\arraystretch}{1.3}
\caption{\edit{Evaluation on Newer College dataset \cite{zhang2021ncdmulti}.
For each scene, we report translation RMSE (m), rotation RMSE (deg), Chamfer-L1 error (cm), and F-score (\%) with a threshold of $20$ cm. Best and second-best results are highlighted in \textbf{bold} and \underline{underline}, respectively.}}
\label{tab:ncd}
\resizebox{0.8\textwidth}{!}{%
\begin{tabular}{|l|rrrr|rrrr|}
\hline
\multicolumn{1}{|c|}{Scene}  & \multicolumn{4}{c|}{Quad (1991 scans)}                                                                                                                                   & \multicolumn{4}{c|}{Maths Institute (2160 scans)}                                                                                                                        \\
\multicolumn{1}{|c|}{Method} & \multicolumn{1}{c}{Tran err. $\downarrow$} & \multicolumn{1}{c}{Rot err. $\downarrow$} & \multicolumn{1}{l}{C-l1 $\downarrow$} & \multicolumn{1}{l|}{F-score $\uparrow$} & \multicolumn{1}{c}{Tran err. $\downarrow$} & \multicolumn{1}{c}{Rot err. $\downarrow$} & \multicolumn{1}{l}{C-l1 $\downarrow$} & \multicolumn{1}{l|}{F-score $\uparrow$} \\ \hline
ICP \cite{zhou2018open3d} + \AlgName~Mapping                & 4.06                                       & 11.02                                     & 42.0                                  & 32.58                                   & 3.57                                       & 15.66                                     & 48.87                                 & 24.68                                   \\
KISS-ICP \cite{vizzo2023kiss} + \AlgName~Mapping           & \textbf{0.08}                              & {\ul 0.98}                                & \textbf{13.96}                        & {\ul 82.68}                             & 0.24                                       & 2.4                                       & \textbf{20.65}                        & {\ul 63.41}                             \\
PIN-SLAM \cite{pan2024pin}                    & {\ul 0.10}                                 & \textbf{0.97}                             & 14.74                                 & 78.85                                   & \textbf{0.15}                              & \textbf{1.37}                             & {\ul 20.66}                           & \textbf{70.18}                          \\
MISO (Odometry)              & 0.47                                       & 1.58                                      & 18.53                                 & 68.41                                   & 0.25                                       & {\ul 1.57}                                & 23.28                                 & 57.78                                   \\
MISO (Full)                  & {\ul 0.10}                                 & {\ul 0.98}                                & {\ul 14.01}                           & \textbf{83.0}                           & {\ul 0.22}                                 & {\ul 1.57}                                & 20.89                                 & 62.07                                   \\ \hline
\end{tabular}%
}
\vspace{-0.3cm}
\end{table*}

%% file: sections/fastcamo_figure.tex
\begin{figure}[t]
    \centering
    \includegraphics[width=\linewidth]{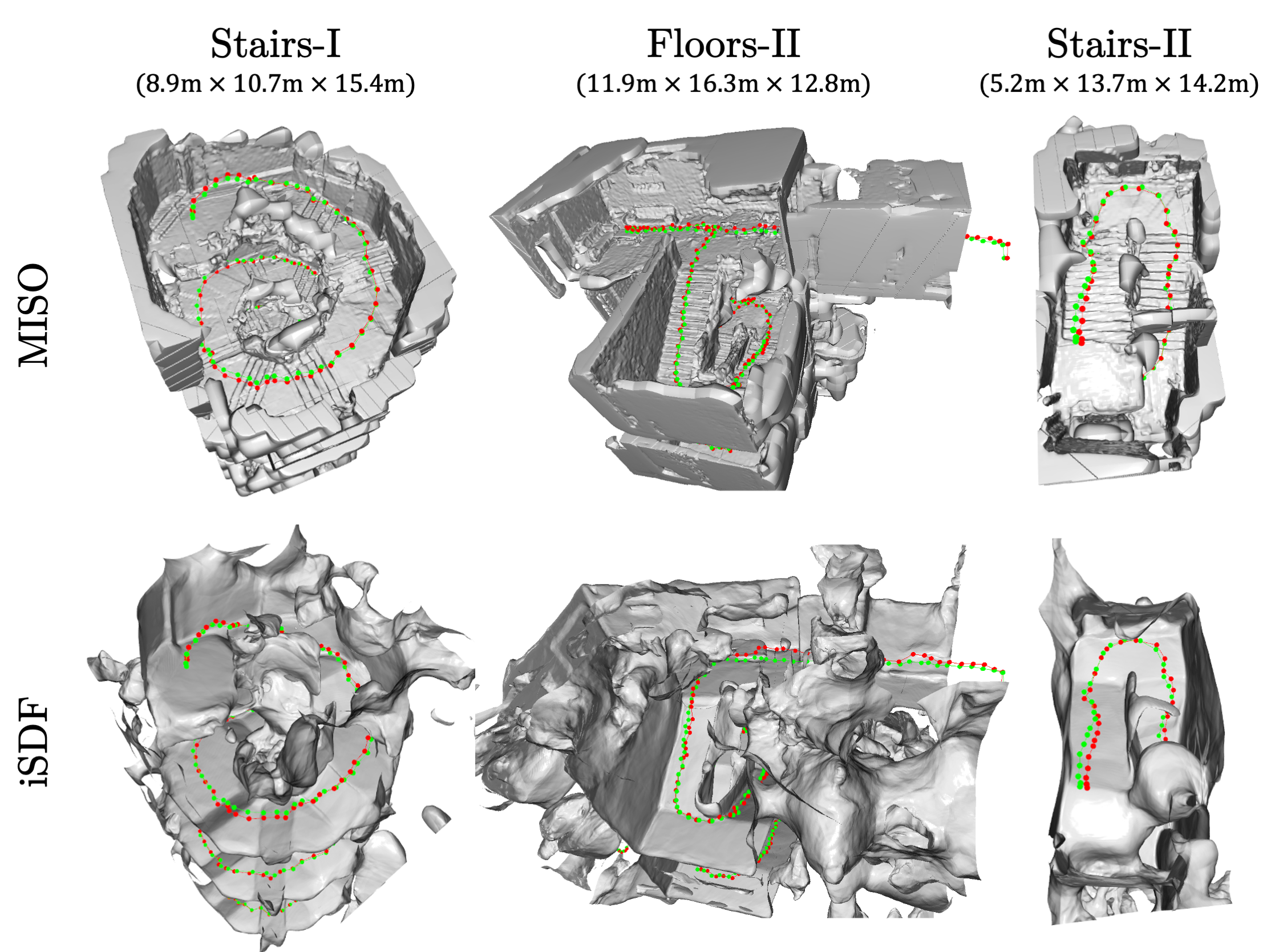}
    \caption{Qualitative evaluation on the \Fastcamo dataset \cite{tang2023mips}. We show the final reconstructed mesh along with estimated and reference robot trajectories (from \cite{tang2023mips}) in red and green, respectively.}
    \label{fig:fastcamo}
\end{figure}

%% file: tables/scannet_init.tex

\begin{table}[t]
\centering
\renewcommand{\arraystretch}{1.6}
\caption{Ablation study on the effect of pre-trained encoder and decoder networks. For each scene, we report F-score (\%) and the training loss at epochs 10 and 100. Results are averaged over five runs, and the best outcomes are highlighted in \textbf{bold}.}
\label{tab:scannet_init}
\begingroup
\resizebox{\linewidth}{!}{%
\begin{tabular}{|l|rr|rr|rr|rr|}
\hline
\multicolumn{1}{|c|}{Scene} & \multicolumn{2}{c|}{0011 \vspace{-0.25em}} & \multicolumn{2}{c|}{0011 \vspace{-0.25em}} & \multicolumn{2}{c|}{0207 \vspace{-0.25em}} & \multicolumn{2}{c|}{0207 \vspace{-0.25em}} \\ 
\multicolumn{1}{|c|}{} & \multicolumn{2}{c|}{\fontsize{6pt}{7pt}\selectfont (10 epoch) \vspace{-0.15em}} & \multicolumn{2}{c|}{\fontsize{6pt}{7pt}\selectfont (100 epoch) \vspace{-0.15em}} & \multicolumn{2}{c|}{\fontsize{6pt}{7pt}\selectfont (10 epoch) \vspace{-0.15em}} & \multicolumn{2}{c|}{\fontsize{6pt}{7pt}\selectfont (100 epoch) \vspace{-0.15em}}\\
\multicolumn{1}{|c|}{Method} 
& \multicolumn{1}{c}{F-score$\uparrow$} & \multicolumn{1}{c|}{Loss$\downarrow$}& \multicolumn{1}{c}{F-score$\uparrow$} & \multicolumn{1}{c|}{Loss$\downarrow$}
& \multicolumn{1}{c}{F-score$\uparrow$}& \multicolumn{1}{c|}{Loss$\downarrow$}& \multicolumn{1}{c}{F-score$\uparrow$} & \multicolumn{1}{c|}{Loss$\downarrow$}\\
\hline
{\method{No-ED}} 
& - & 0.206 & 59.3 & 0.061
& - & 0.196 & 58.5 & 0.056 \\
{\method{No-E}} 
& 52.8 & 0.139 & \textbf{69.5} & \textbf{0.056}
& 53.8 & 0.130 & \textbf{67.5} & \textbf{0.052} \\
{\method{Full}} 
 & \textbf{59.4} & \textbf{0.098} & 66.3 & 0.057 
 & \textbf{57.8} & \textbf{0.089} & 66.3 & 0.052 \\
\hline
\end{tabular}%
}
\endgroup
\end{table}

%% file: tables/scannet_hier.tex
\begin{table}[t]
\centering
\renewcommand{\arraystretch}{1.6}
\caption{Ablation study on hierarchical alignment on ScanNet \cite{dai2017scannet} from large initial errors of $10$~deg and $0.25$~m.
We report the alignment solver time (sec) and the final submap rotation error (deg) and translation error (m) compared to ground truth. Results are averaged over 10 trials. Best results are highlighted in \textbf{bold}.}
\label{tab:scannet_hier}
\begingroup
\resizebox{\linewidth}{!}{%
\begin{tabular}{|l|rrr|rrr|}
\hline
\multicolumn{1}{|c|}{Scene} & \multicolumn{3}{c|}{0011 (4 submaps)} & \multicolumn{3}{c|}{0207 (2 submaps)} \\
\multicolumn{1}{|c|}{Method} 
& \multicolumn{1}{c}{Time$\downarrow$} & \multicolumn{1}{c}{Rot err$\downarrow$} & \multicolumn{1}{c|}{Tran err$\downarrow$}
& \multicolumn{1}{c}{Time$\downarrow$} & \multicolumn{1}{c}{Rot err$\downarrow$} & \multicolumn{1}{c|}{Tran err$\downarrow$} \\
\hline
\method{Coarse only}
 & \textbf{1.20} & 4.73 & 0.16
 & \textbf{0.25} & 4.86 & 0.07 \\
\method{Coarse+Fine}
 & 2.54 & 3.68 & 0.12
 & 0.50 & 1.89 & 0.07 \\
\method{Full}
 & 10.13 & \textbf{2.71} & \textbf{0.09}
 & 1.68 & \textbf{1.11} & \textbf{0.05} \\
\hline
\end{tabular}%
}
\endgroup
\end{table}

%% file: sections/discussions.tex
\section{Limitations}
\label{sec:limitations}

While \AlgName demonstrates very promising results in our evaluation, it still has several limitations. 
One direction for improvement is motivated by the use of dense feature grids. 
While dense grids offer a range of advantages including fast interpolation, better free-space modeling, and the ease of designing encoder networks, their memory cost may present a challenge when scaling the method to very large environments.
To address this, combining \AlgName's hierarchical optimization approach with sparse data structures or tensor decomposition techniques is an exciting direction for future work. 
In addition, while we have only considered SDF reconstruction in this work, we expect the underlying approach in \AlgName to be applicable for modeling additional scene properties such as radiance, semantics, and uncertainty.
Lastly, in this work, we primarily demonstrated \AlgName in an offline batch optimization setting, 
where a static scene is reconstructed given robot odometry and observations.
Extending to online and dynamic settings is an important direction for future research.

\begin{figure}[t]
    \centering
    \begin{subfigure}[t]{0.49\linewidth}
        \centering
        \includegraphics[trim=0 10 0 0, clip,width=\linewidth]{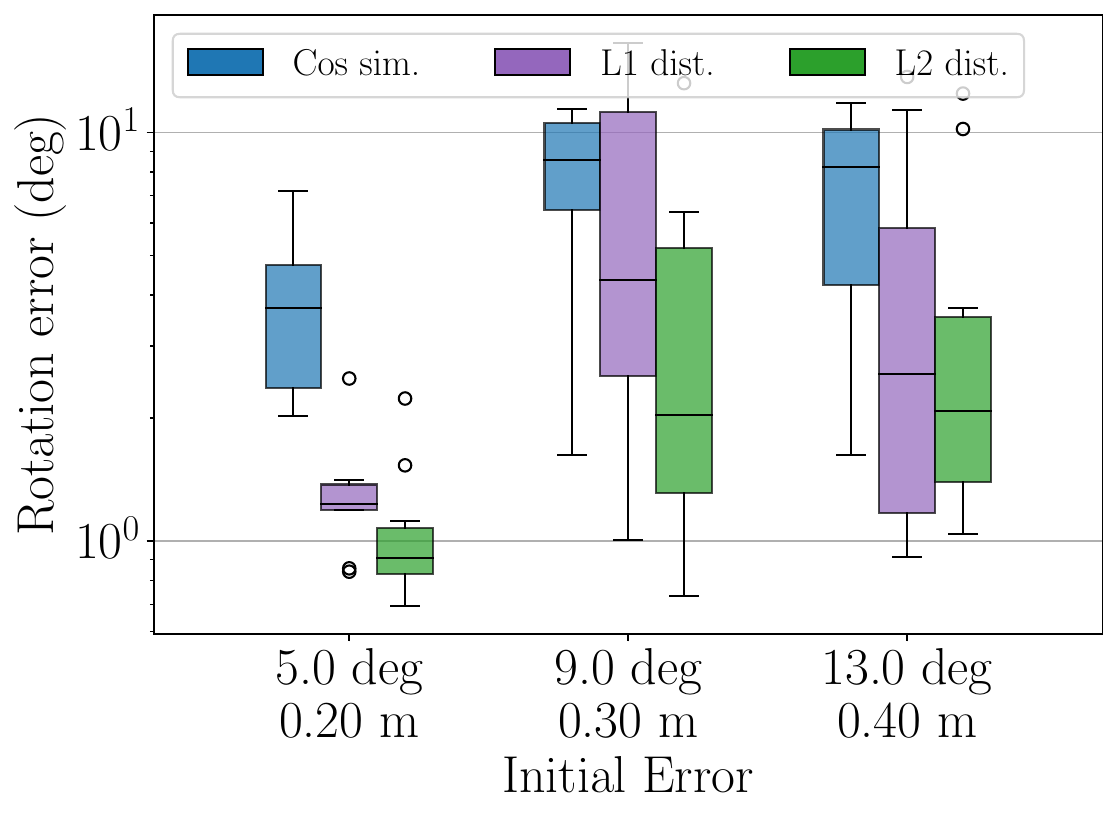}
        \caption{Rotation error (deg)}
        \label{fig:ablation_metric:rot}
    \end{subfigure}
    \hfill
    \begin{subfigure}[t]{0.49\linewidth}
        \centering
        \includegraphics[trim=0 10 0 0, clip,width=\linewidth]{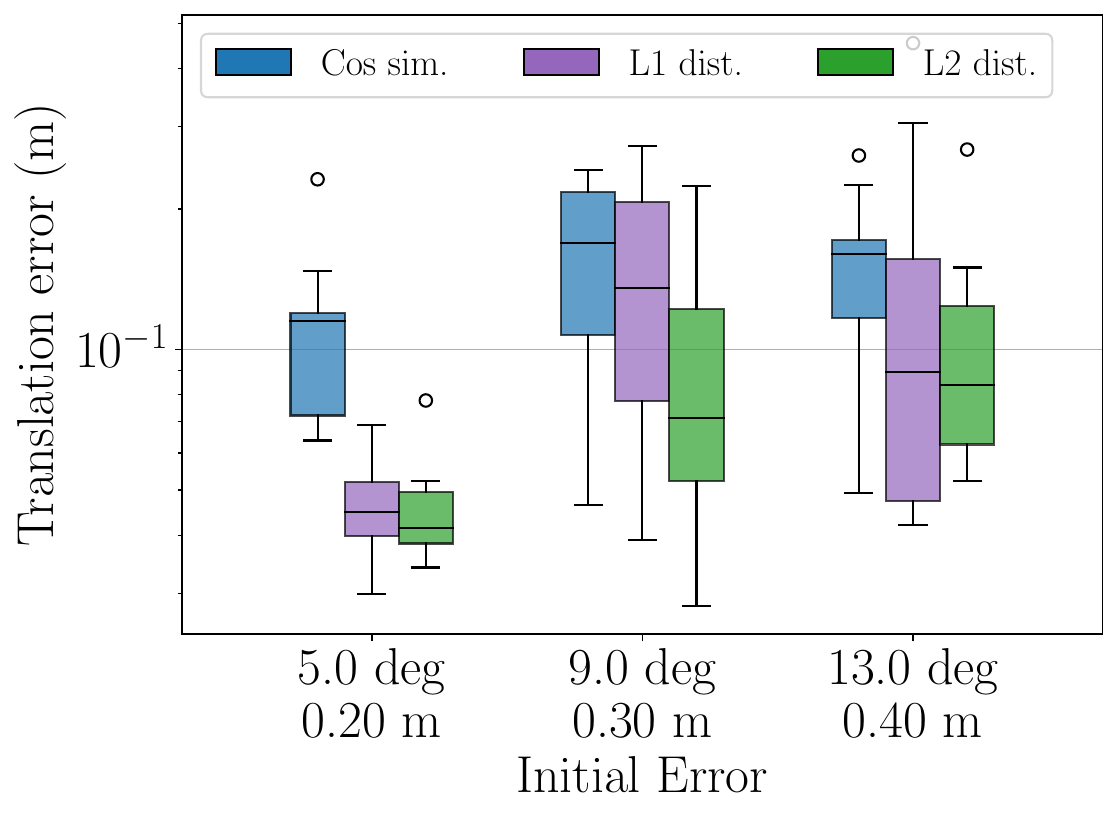}
        \caption{Translation error (m)}
        \label{fig:ablation_metric:tran}
    \end{subfigure}
    \caption{
Ablation study on hierarchical alignment on ScanNet \cite{dai2017scannet}. We compare negative cosine similarity, L1 distance and L2 distance under three different initial rotation and translation errors.}
    \label{fig:ablation_metric}
\end{figure}

%% file: sections/conclusion.tex
\section{Conclusion}
\label{sec:conclusion}

We presented \AlgName, a hierarchical back-end for neural implicit SLAM based on multiresolution submap optimization.
\AlgName performs local SLAM by representing each submap as a multiresolution feature grid.
To achieve fast optimization and enable generalization, \AlgName makes use of pre-trained hierarchical encoder and decoder networks, where the encoders serve to initialize the submap features from observations, and the decoder serves to decode optimized features to predict the scene geometry.
Furthermore, \AlgName presents a hierarchical method that sequentially uses implicit features and SDF predictions to align and fuse submaps, achieving globally consistent mapping.
\edit{Evaluation on the real-world ScanNet, \Fastcamo, and Newer College datasets demonstrated \AlgName's strong performance and superior efficiency.}

%% file: sections/grid_details.tex
\section{Implementation and Training Details}
\label{sec:grid_details}

In our experiments, each submap is a multiresolution feature grid that has two levels with spatial resolutions 0.5~m and 0.1~m, respectively.
The feature dimension at each level is set to $d=4$.
The decoder network $D_\theta$ is a single-layer MLP with hidden dimension 64. 
For the encoder network $E_{\phi_l}$ at each level $l$, we first use a small 3D CNN to process the input voxelized features.
Each CNN has 2 hidden layers, a kernel size of 3, and the number of feature channels doubles after every layer.
The CNN outputs at all 3D vertices are processed by a shared two-layer MLP with hidden dimension 16 to predict the target feature grid $F_l$ (see \cref{fig:encoder}).
We use ReLU as the nonlinear activation for all networks.

To train the decoder network $D_\theta$ offline, we compile an offline training dataset using six scenes from Replica \cite{replica19arxiv}. 
Within each scene $s \in S$, we randomly simulate 128 camera views with ground truth pose information. 
During training, all scenes share a single set of decoder parameters $\theta$, 
and each scene $s$ has its own dedicated grid features $F^{s}$.
We perform joint optimization using a loss similar to \cref{prob:local_mapping},
except we use ground truth camera poses and disable pose optimization, 

\begin{equation}
\label{eq:decoder_pretrain}
\underset{ \{F^s\}_s , \theta}{\min}
	 \sum_{s \in S}\sum_{k=1}^n \sum_{j=1}^{m_k} 
        c_j\bigl(
            h(\underline{T}^s_k x^k_j; F^s, \theta)
        \bigr),
\end{equation}
where $\underline{T}^s_k$ denote the ground truth camera poses. 
We use Adam \cite{kingma2014adam} with a learning rate of $10^{-3}$ and train for a total of 1200 epochs. 
During training, we employ a coarse-to-fine strategy where all fine level features are activated after 200 epochs. 
After training completes, the grid features $F^s$ are discarded and only decoder parameters $\theta$ is saved.

The offline training of encoder networks follows a similar setup and uses the same Replica scenes.
The design of the training loss is presented in \cref{sec:local_optimization} in the main paper.
We sequentially train the encoder networks from coarse to fine levels.
At level $l$, we use the trained encoders from previous levels to compute $F^s_{1:l-1}$, which are needed to evaluate the training loss in \eqref{eq:encoder_training_loss}.
To account for noisy pose estimates at test time, we simulate pose errors  with $1$~deg and $1$~cm  rotation and translation standard deviations. 
Each encoder is trained using Adam \cite{kingma2014adam} with a learning rate of $10^{-3}$ for 1000 epochs.

%% file: sections/proof.tex
\section{Proofs}
\label{sec:proof}

\begin{proof}[Proof of \cref{lem:linear_case}]
Without loss of generality, we assume the features $F_l$ is represented as a vector $F_l \in \Real^{|F_l|}$.
In the following, all points are expressed in the coordinate frame of the submap.
First, consider a single observed point $x_j \in \Real^3$ and its label $y_j \in \Real$.
At each level $l$, the output feature 
$f_l(x_j)$ in \cref{def:grid} is a linear function of the features $F_l$ at this level.
This means that there exists a matrix $K_l(x_j) \in \Real^{d \times |F_l|}$ such that
$f_l(x_j) = K_l(x_j) F_l$.
When the decoder $D_\theta$ is a linear function, the final output from the multiresolution feature grid is a weighted sum of features from all levels, \ie,
\begin{equation}
\begin{aligned}
h(x_j;F)
\!=\!
D_\theta \bigl(
	\bigoplus_{l \in [L]}
	f_l(x_j)
\bigr) 
\!= \!
\sum_{l=1}^L \theta_l^\top f_l(x_j) 
\!= \!
\sum_{l=1}^L \theta_l^\top K_l(x_j) F_l.
\end{aligned}
\end{equation}
In the following, we use $l$ to specifically refer to the target level that we seek to initialize in \eqref{eq:level_local_mapping}.
The measurement residual is,
\begin{equation}
h(x_j;F) - y_j = 
\underbrace{\sum_{l'=1}^{l-1} \theta_{l'}^\top K_{l'}(x_j) F_{l'} - y_j}_{r_{1:l-1}(x_j)}
+ \theta_{l}^\top K_{l}(x_j) F_{l},
\end{equation}
where we have used the fact that features at levels greater than $l$ are zero.
We can concatenate the above equation over all observations $x_1, \hdots, x_N$ as follows,
\begin{equation}
h(x;F) - y = 
\underbrace{
\begin{bmatrix}
r_{1:l-1}(x_1) \\
\vdots \\
r_{1:l-1}(x_N)
\end{bmatrix} 
}_{r_{1:l-1}(x)}
+ 
\underbrace{
\begin{bmatrix}
\theta_{l}^\top K_{l}(x_1) \\
\vdots \\
\theta_{l}^\top K_{l}(x_N)
\end{bmatrix}
}_{J}
F_l.
\end{equation} 
Note that the matrix $J$ is exactly the Jacobian of our model $h(x; F)$ with respect to the level-$l$ features $F_l$. 
Substituting the above equation into \eqref{eq:level_local_mapping} shows that the problem is equivalent to the following linear least squares,
\begin{equation}
\underset{{F_l}}{\min} \norm{h(x; F) - y}^2_2 = \norm{r_{1:l-1}(x) + J F_l}^2_2.
\end{equation}
Thus, the solution is obtained by solving its normal equations, yielding the final expression in 
\eqref{eq:level_local_mapping_closed_form}.
\end{proof}

%% file: sections/additional_results.tex
\section{Additional Results}
\label{sec:additional_experiments}

In \cref{fig:scannet_additional}, we show qualitative SDF visualizations on additional scenes from ScanNet \cite{dai2017scannet} under noisy poses.
For each method, a horizontal slice of the estimated SDF is shown where red and blue indicate positive and negative values.
These figures further supports the conclusions drawn in the paper, and the reader is referred to \cref{sec:experiments} for the discussion.

\begin{figure*}[t]
    \begin{center}
        \begin{subfigure}[t]{0.24\linewidth}
            \centering
            \includegraphics[width=\linewidth]{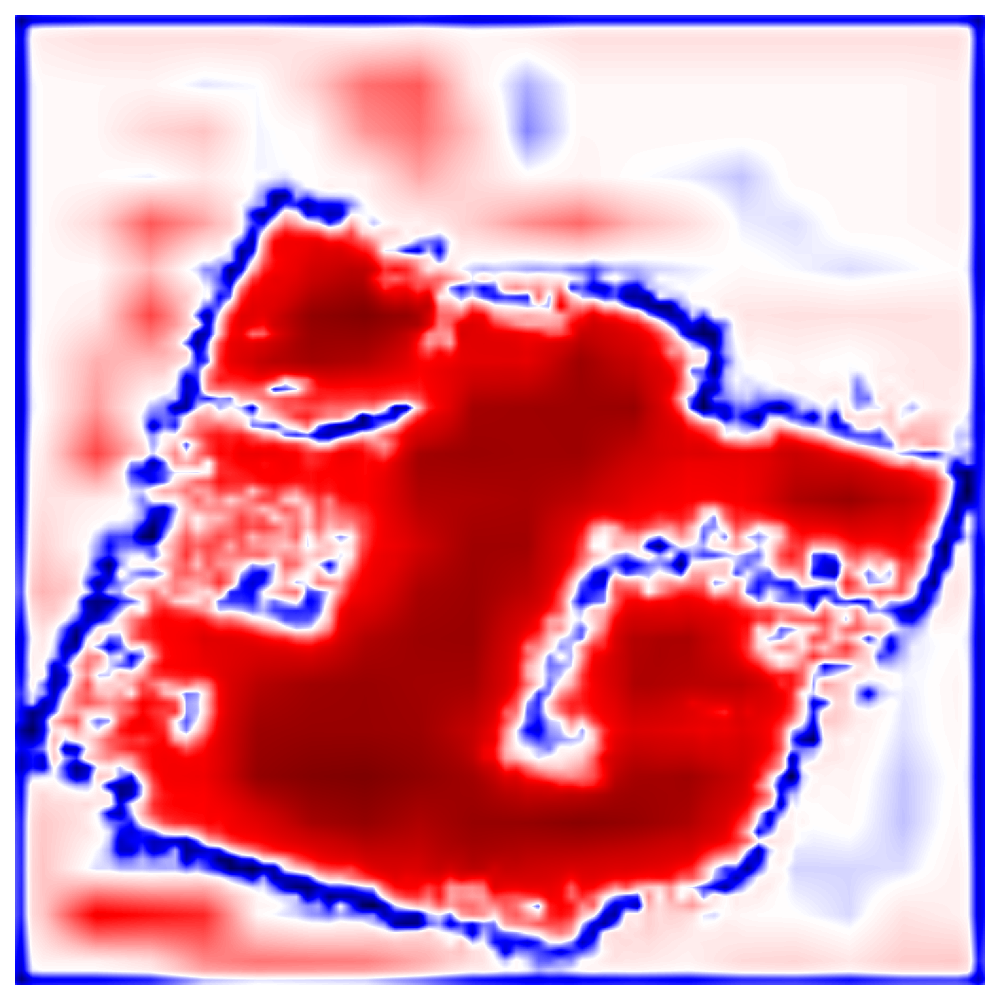}
            \caption{\AlgName (initialized)}
            \label{fig:scannet0000:init}
        \end{subfigure}
        \begin{subfigure}[t]{0.24\linewidth}
            \centering
            \includegraphics[width=\linewidth]{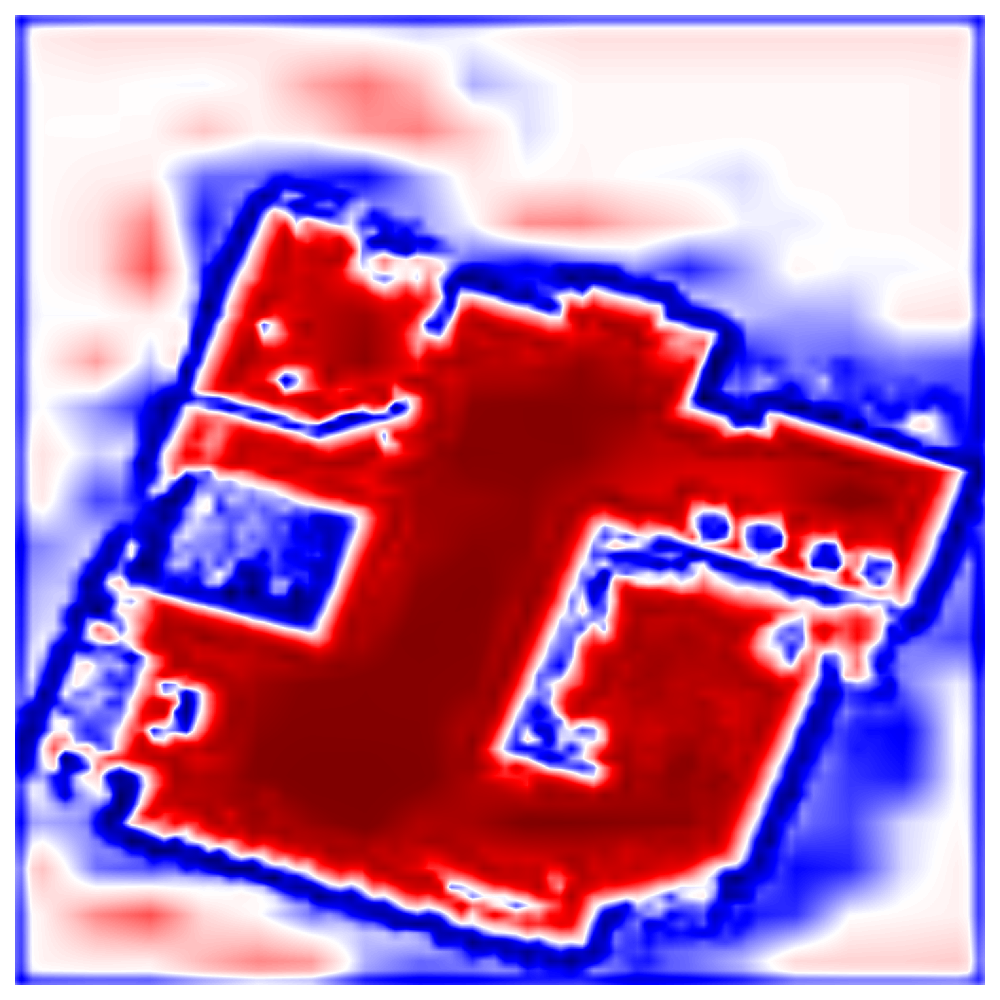}
            \caption{\AlgName (optimized)}
            \label{fig:scannet0000:opt}
        \end{subfigure} 
        \begin{subfigure}[t]{0.24\linewidth}
            \centering
            \includegraphics[width=\linewidth]{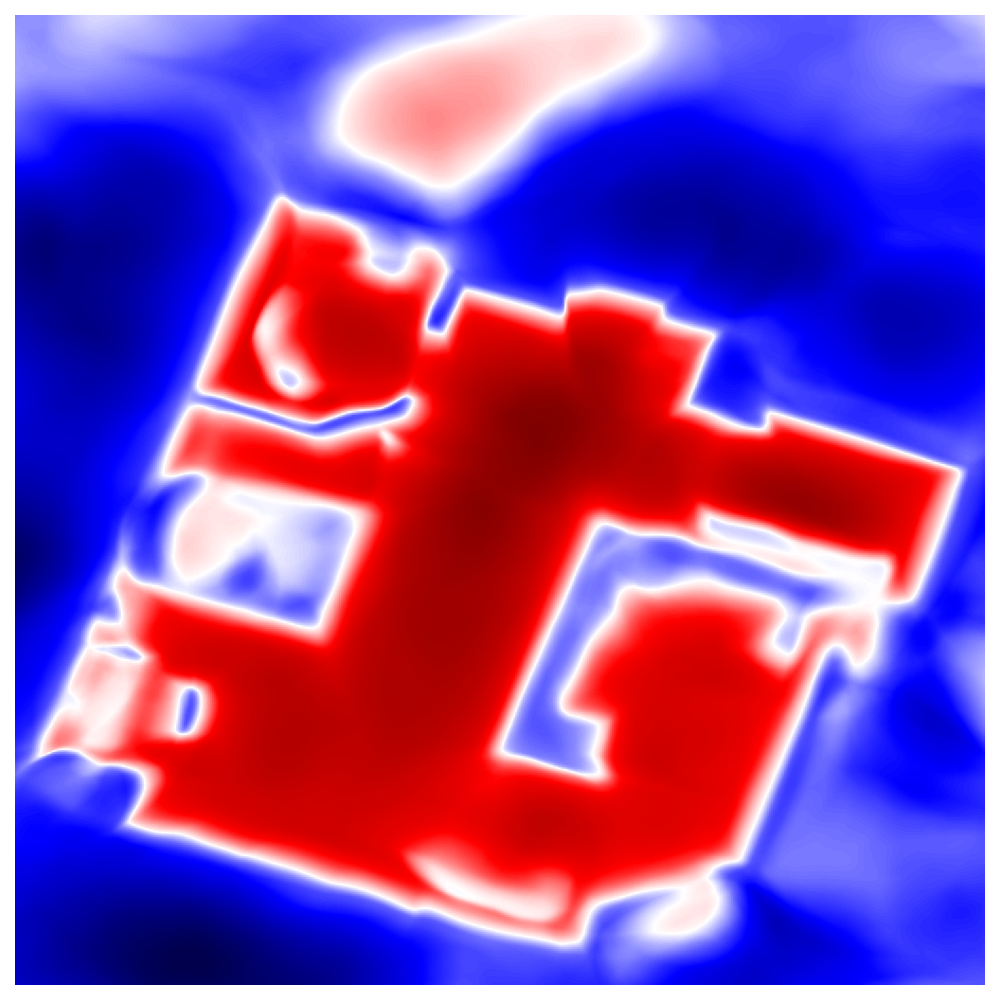}
            \caption{\iSDF (GT pose) \cite{ortiz2022isdf}}
            \label{fig:scannet0000:isdf}
        \end{subfigure}
        \begin{subfigure}[t]{0.24\linewidth}
            \centering
            \includegraphics[width=\linewidth]{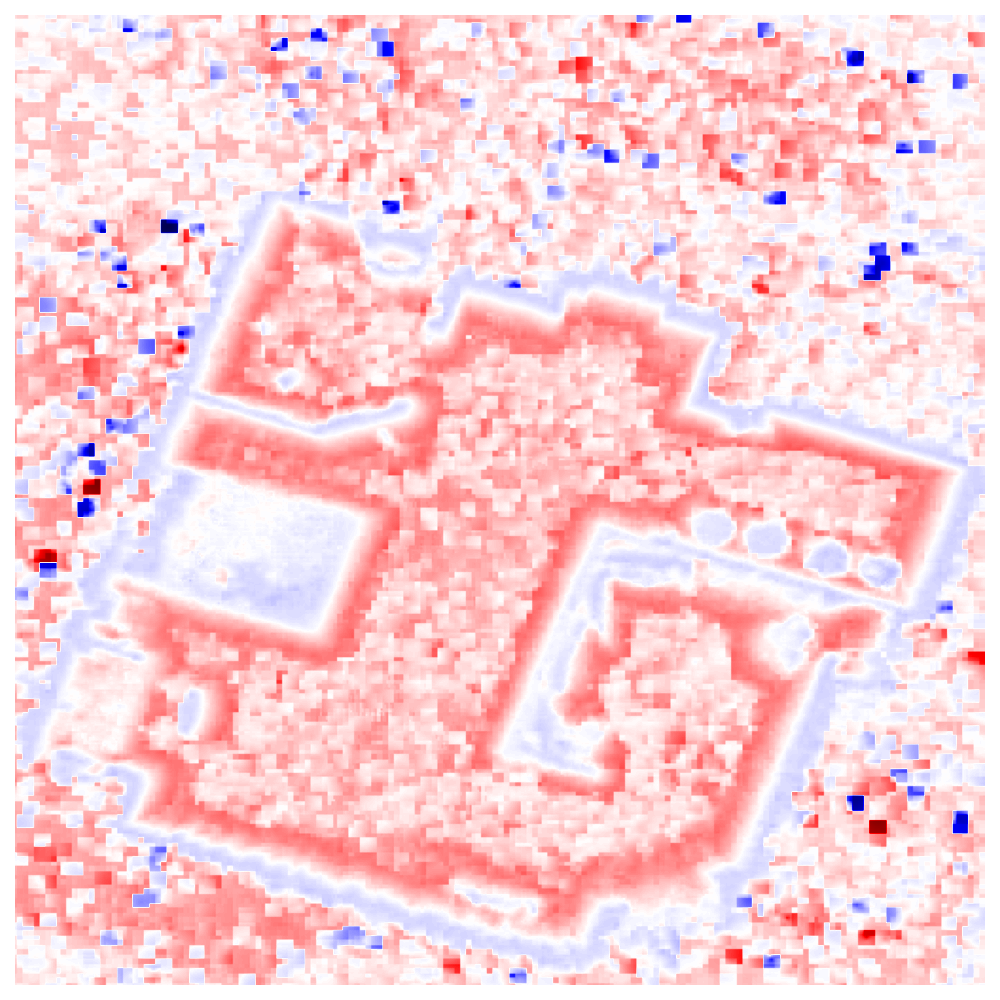}
            \caption{\Point (GT pose) \cite{pan2024pin}}
            \label{fig:scannet0000:pin}
        \end{subfigure}
        \\
        \begin{subfigure}[t]{0.24\linewidth}
            \centering
            \includegraphics[width=\linewidth]{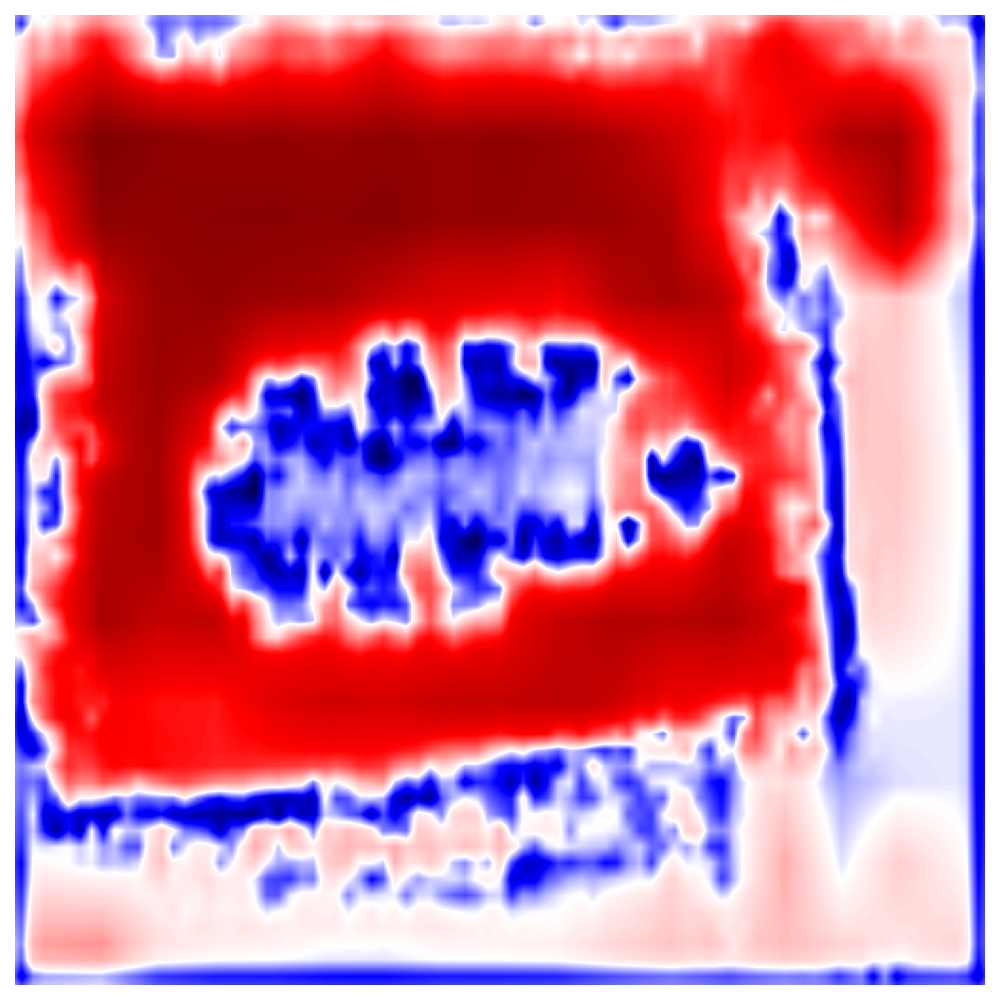}
            \caption{\AlgName (initialized)}
        \end{subfigure}
        \begin{subfigure}[t]{0.24\linewidth}
            \centering
            \includegraphics[width=\linewidth]{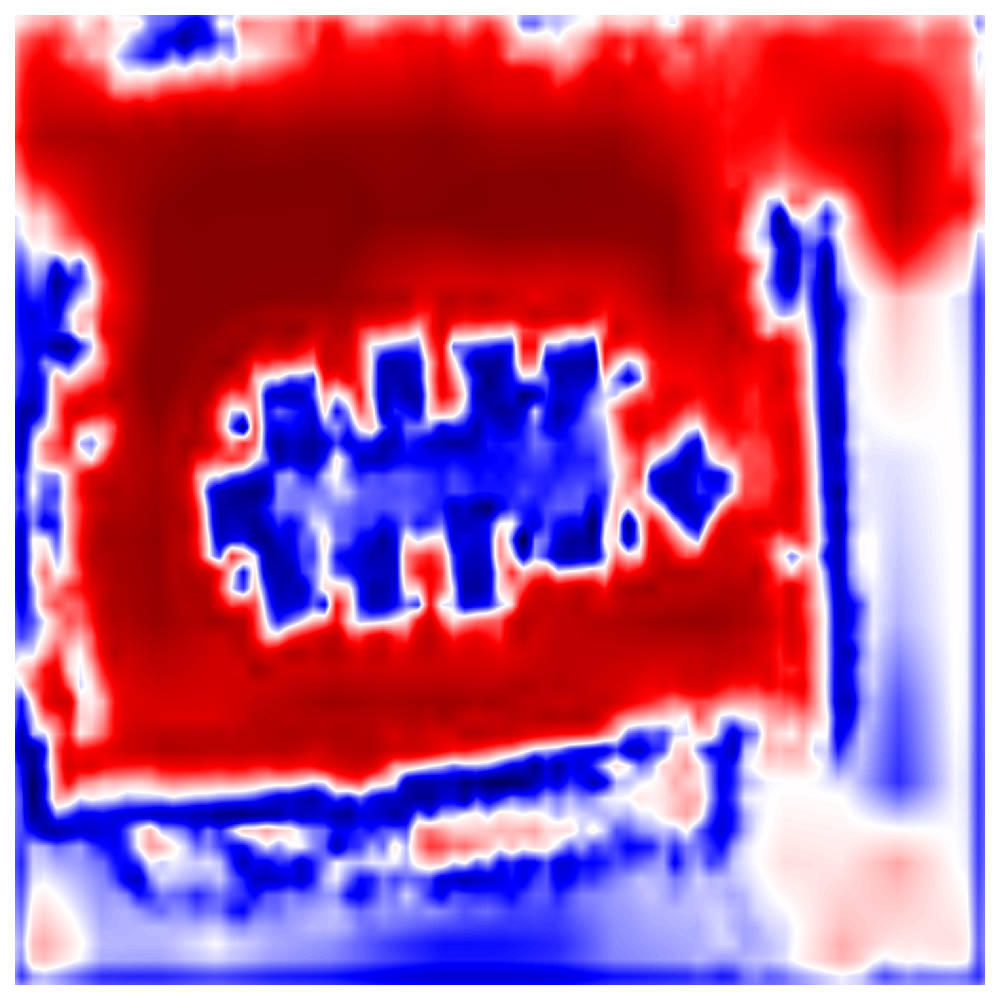}
            \caption{\AlgName (optimized) }
        \end{subfigure} 
        \begin{subfigure}[t]{0.24\linewidth}
            \centering
            \includegraphics[width=\linewidth]{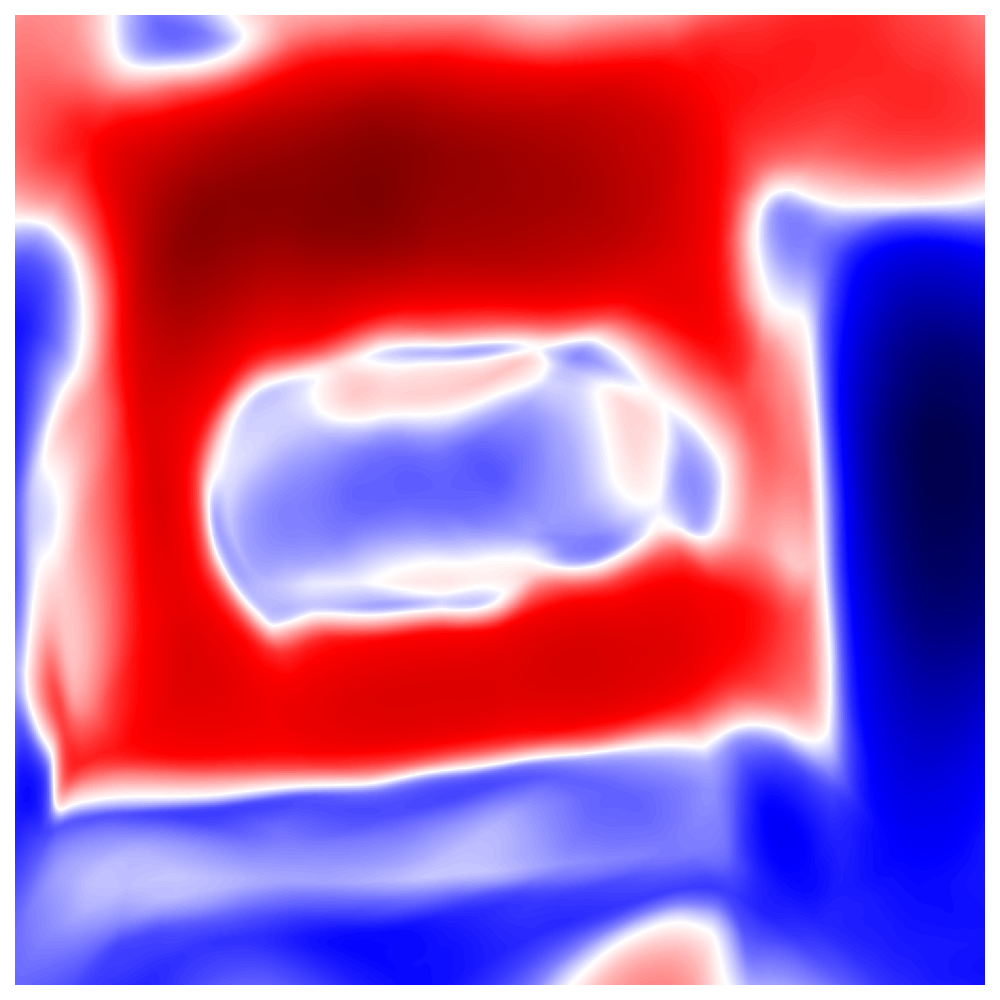}
            \caption{\iSDF (GT pose) \cite{ortiz2022isdf} }
        \end{subfigure}
        \begin{subfigure}[t]{0.24\linewidth}
            \centering
            \includegraphics[width=\linewidth]{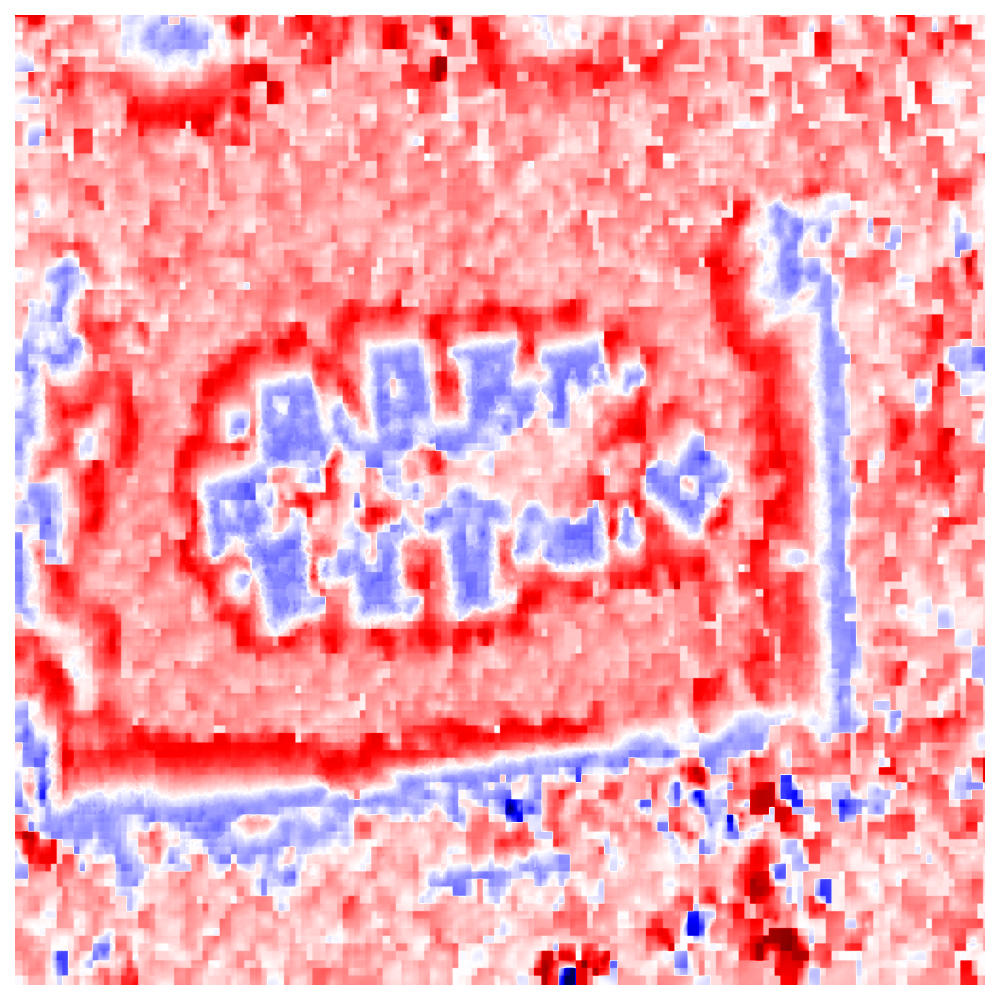}
            \caption{\Point (GT pose) \cite{pan2024pin} }
        \end{subfigure}
        \\
        \begin{subfigure}[t]{0.24\linewidth}
            \centering
            \includegraphics[width=\linewidth]{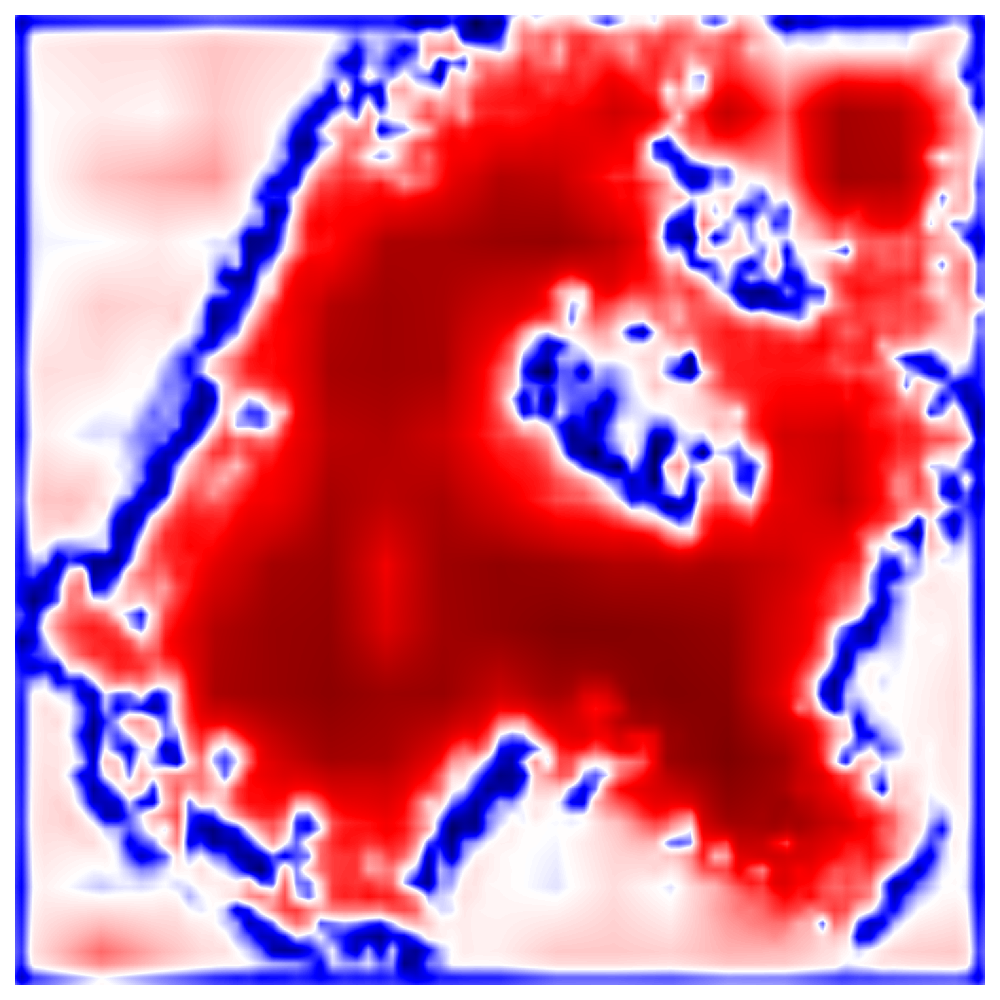}
            \caption{\AlgName (initialized)}
        \end{subfigure}
        \begin{subfigure}[t]{0.24\linewidth}
            \centering
            \includegraphics[width=\linewidth]{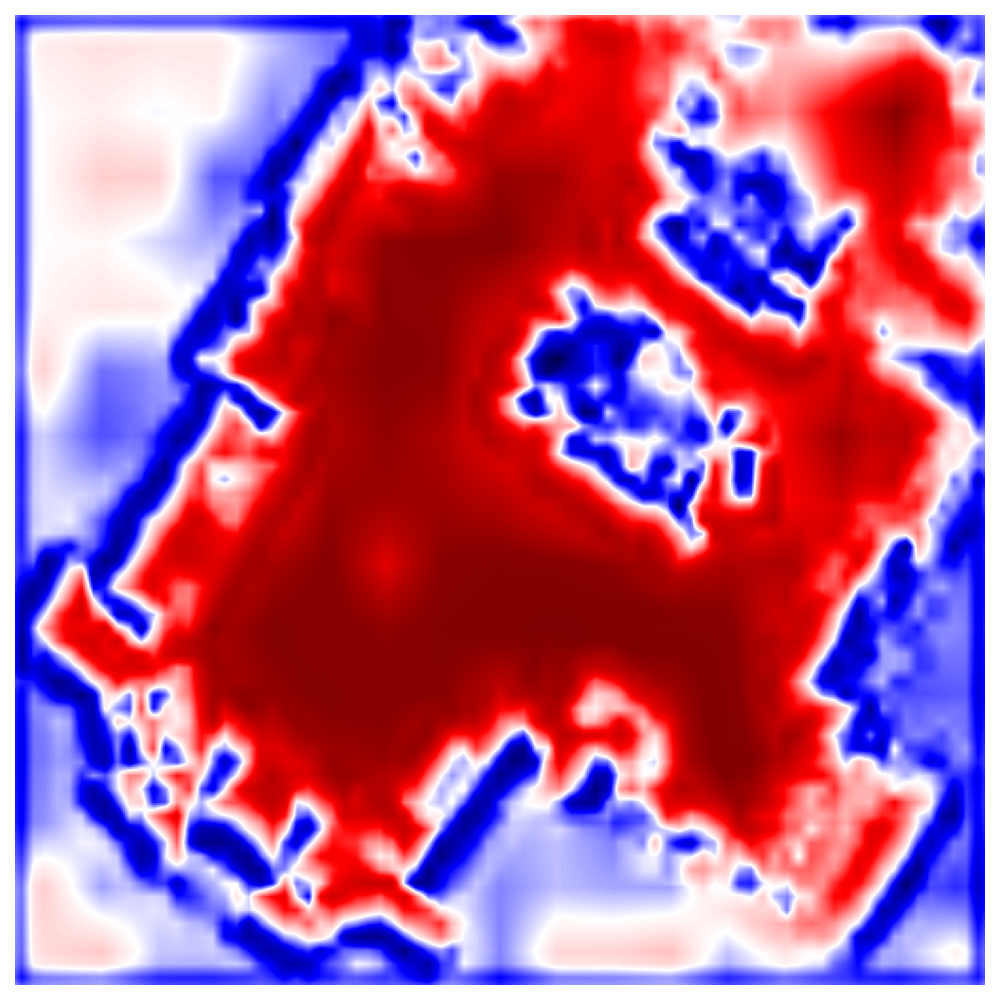}
            \caption{\AlgName (optimized) }
        \end{subfigure} 
        \begin{subfigure}[t]{0.24\linewidth}
            \centering
            \includegraphics[width=\linewidth]{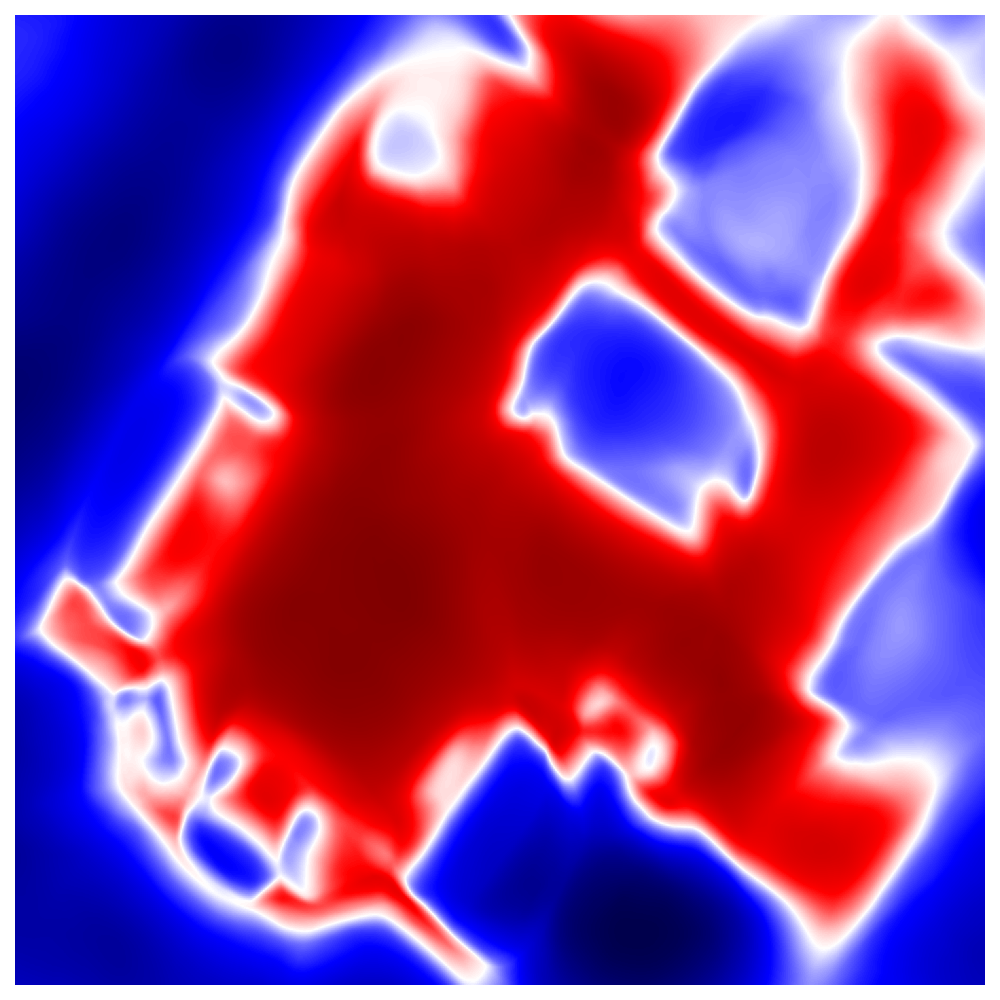}
            \caption{\iSDF (GT pose) \cite{ortiz2022isdf} }
        \end{subfigure}
        \begin{subfigure}[t]{0.24\linewidth}
            \centering
            \includegraphics[width=\linewidth]{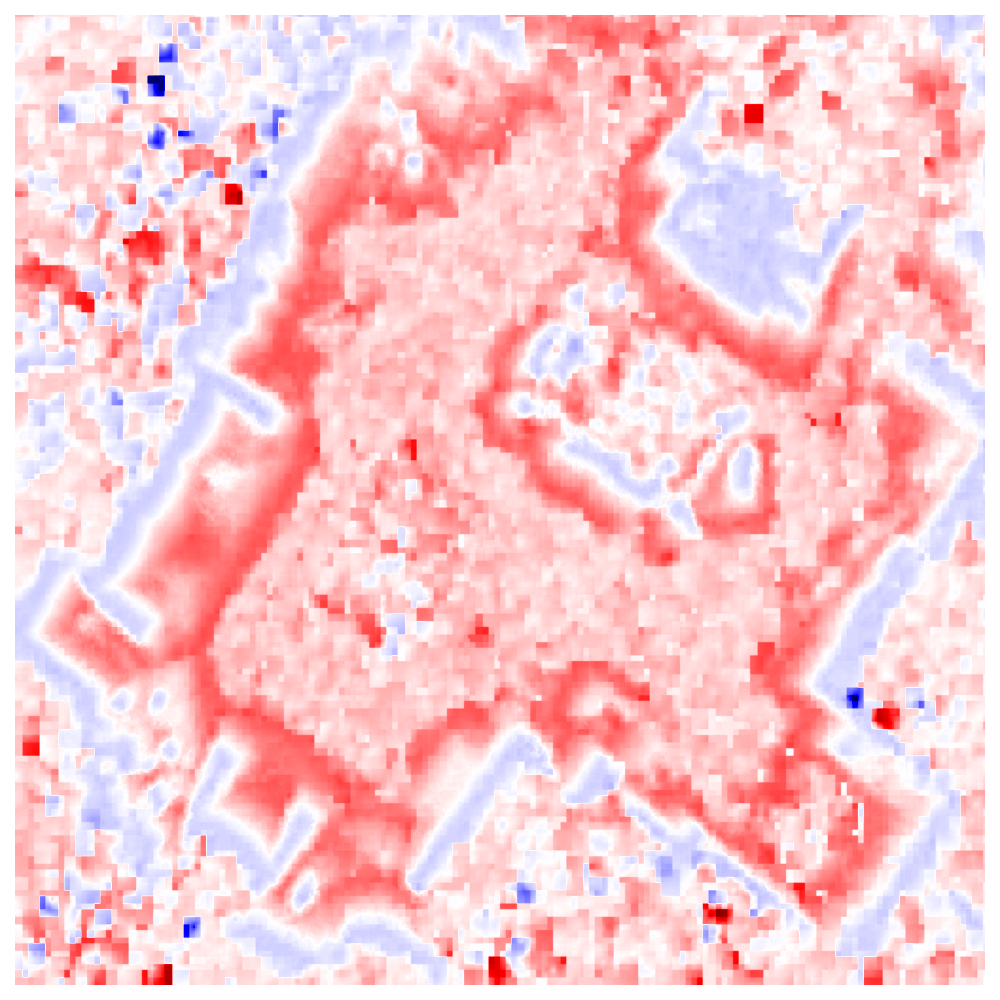}
            \caption{\Point (GT pose) \cite{pan2024pin} }
        \end{subfigure}
    \end{center}
    \caption{Qualitative comparison on additional ScanNet scenes 0000, 0011, and 0024, each shown in a row. 
    For each method, a horizontal slice of the estimated SDF is shown where red and blue indicate positive and negative values.}
    \label{fig:scannet_additional}
\end{figure*}

\input{tables/memory_measure}
We evaluate \AlgName’s memory requirements by measuring both the total number of parameters (covering the grid model and the decoder) and the GPU memory usage. The GPU memory usage is measured across the entire environment using PyTorch’s {\tt cuda.max\_memory\_allocated()} function. 
We test three scenes in \cref{tab:memory_measure}: ScanNet \cite{dai2017scannet} scene 011 and \Fastcamo \cite{tang2023mips} Stairs-I and Floors-I. 
The sizes of these environments, measured as the dimensions of their axis-aligned bounding boxes, are reported below,
\begin{itemize} 
\item ScanNet 0011: $6.0$~m $\times$ $8.2$~m $\times$ $2.7$~m
\item \Fastcamo Stairs-I: $8.9$~m $\times$ $10.7$~m $\times$ $15.4$~m
\item \Fastcamo Floors-I: $26.0$~m $\times$ $16.7$~m $\times$ $7.5$~m
\end{itemize}
We consider two configurations: one using only the coarse grid (\method{Coarse only}) and another using both the coarse and fine grids (\method{Full}). 
In all scenes, the decoder itself consistently contains about 0.12M parameters.

%% file: tables/memory_measure.tex

\begin{table}[H]
    \centering
    \renewcommand{\arraystretch}{1.6}
    \caption{Number of parameters and GPU memory usage of \AlgName.}
    \label{tab:memory_measure}
    \begingroup
    \resizebox{\linewidth}{!}{%
    \begin{tabular}{|l|rr|rr|rr|}
        \hline
        \multicolumn{1}{|c|}{Method} 
        & \multicolumn{2}{c|}{ScanNet (0011)} 
        & \multicolumn{2}{c|}{FastCamo (Stairs-I)} 
        & \multicolumn{2}{c|}{FastCamo (Floors-I)} \\
        
        \multicolumn{1}{|c|}{} 
        & \multicolumn{1}{c}{Params} 
        & \multicolumn{1}{c|}{GPU mem} 
        & \multicolumn{1}{c}{Params} 
        & \multicolumn{1}{c|}{GPU mem} 
        & \multicolumn{1}{c}{Params} 
        & \multicolumn{1}{c|}{GPU mem} \\
        \hline
        \method{Coarse only}
        & 0.14M & 2.11GB
        & 0.25M & 1.60GB
        & 0.26M & 2.64GB \\
        
        \method{Full}
        & 0.70M & 2.18GB
        & 9.08M & 1.67GB
        & 8.48M & 2.71GB \\
        \hline
    \end{tabular}%
    }
    \endgroup
\end{table}

%% file: main.bbl
\begin{thebibliography}{10}

\bibitem{tang2023mips}
Y.~Tang, J.~Zhang, Z.~Yu, H.~Wang, and K.~Xu, ``{MIPS-Fusion:
  Multi-implicit-submaps for scalable and robust online neural RGB-D
  reconstruction},'' {\em ACM Transactions on Graphics (TOG)}, 2023.

\bibitem{xie2022neural}
Y.~Xie, T.~Takikawa, S.~Saito, O.~Litany, S.~Yan, N.~Khan, F.~Tombari,
  J.~Tompkin, V.~Sitzmann, and S.~Sridhar, ``{Neural fields in visual computing
  and beyond},'' {\em Computer Graphics Forum (CGF)}, 2022.

\bibitem{tosi2024nerfs}
F.~Tosi, Y.~Zhang, Z.~Gong, E.~Sandstr{\"o}m, S.~Mattoccia, M.~R. Oswald, and
  M.~Poggi, ``{How NeRFs and 3D Gaussian Splatting are reshaping SLAM: A
  survey},'' {\em arXiv preprint arXiv:2402.13255}, 2024.

\bibitem{estrada2005hierarchical}
C.~Estrada, J.~Neira, and J.~D. Tard{\'o}s, ``{Hierarchical SLAM: Real-time
  accurate mapping of large environments},'' {\em IEEE Transactions on Robotics
  (T-RO)}, 2005.

\bibitem{frese2005multilevel}
U.~Frese, P.~Larsson, and T.~Duckett, ``{A multilevel relaxation algorithm for
  simultaneous localization and mapping},'' {\em IEEE Transactions on Robotics
  (T-RO)}, 2005.

\bibitem{grisetti2010hierarchical}
G.~Grisetti, R.~K{\"u}mmerle, C.~Stachniss, U.~Frese, and C.~Hertzberg,
  ``{Hierarchical optimization on manifolds for online 2D and 3D mapping},'' in
  {\em International Conference on Robotics and Automation (ICRA)}, 2010.

\bibitem{takikawa2021nglod}
T.~Takikawa, J.~Litalien, K.~Yin, K.~Kreis, C.~Loop, D.~Nowrouzezahrai,
  A.~Jacobson, M.~McGuire, and S.~Fidler, ``{Neural geometric level of detail:
  Real-time rendering with implicit 3D shapes},'' in {\em IEEE/CVF Conference
  on Computer Vision and Pattern Recognition (CVPR)}, 2021.

\bibitem{sun2022direct}
C.~Sun, M.~Sun, and H.-T. Chen, ``Direct voxel grid optimization: Super-fast
  convergence for radiance fields reconstruction,'' in {\em IEEE/CVF Conference
  on Computer Vision and Pattern Recognition (CVPR)}, 2022.

\bibitem{muller2022instant}
T.~M{\"u}ller, A.~Evans, C.~Schied, and A.~Keller, ``{Instant neural graphics
  primitives with a multiresolution hash encoding},'' {\em ACM Transactions on
  Graphics (TOG)}, 2022.

\bibitem{zhai2024vox}
H.~Zhai, H.~Li, X.~Yang, G.~Huang, Y.~Ming, H.~Bao, and G.~Zhang,
  ``{Vox-Fusion++: Voxel-based neural implicit dense tracking and mapping with
  multi-maps},'' {\em arXiv preprint arXiv:2403.12536}, 2024.

\bibitem{ortiz2022isdf}
J.~Ortiz, A.~Clegg, J.~Dong, E.~Sucar, D.~Novotny, M.~Zollhoefer, and
  M.~Mukadam, ``{iSDF: Real-time neural signed distance fields for robot
  perception},'' in {\em Robotics: Science and Systems (RSS)}, 2022.

\bibitem{kerbl2023gaussian}
B.~Kerbl, G.~Kopanas, T.~Leimk{\"u}hler, and G.~Drettakis, ``{3D Gaussian
  splatting for real-time radiance field rendering},'' {\em ACM Transactions on
  Graphics (TOG)}, 2023.

\bibitem{park2019deepsdf}
J.~J. Park, P.~Florence, J.~Straub, R.~Newcombe, and S.~Lovegrove, ``{DeepSDF:
  Learning continuous signed distance functions for shape representation},'' in
  {\em IEEE/CVF Conference on Computer Vision and Pattern Recognition (CVPR)},
  2019.

\bibitem{mescheder2019occupancy}
L.~Mescheder, M.~Oechsle, M.~Niemeyer, S.~Nowozin, and A.~Geiger, ``{Occupancy
  networks: Learning 3D reconstruction in function space},'' in {\em IEEE/CVF
  Conference on Computer Vision and Pattern Recognition (CVPR)}, 2019.

\bibitem{mildenhall2021nerf}
B.~Mildenhall, P.~P. Srinivasan, M.~Tancik, J.~T. Barron, R.~Ramamoorthi, and
  R.~Ng, ``{NeRF: Representing scenes as neural radiance fields for view
  synthesis},'' {\em Communications of the ACM (CACM)}, 2021.

\bibitem{azinovic2022neural}
D.~Azinovi{\'c}, R.~Martin-Brualla, D.~B. Goldman, M.~Nie{\ss}ner, and
  J.~Thies, ``{Neural RGB-D surface reconstruction},'' in {\em IEEE/CVF
  Conference on Computer Vision and Pattern Recognition (CVPR)}, 2022.

\bibitem{fridovich2022plenoxels}
S.~Fridovich-Keil, A.~Yu, M.~Tancik, Q.~Chen, B.~Recht, and A.~Kanazawa,
  ``{Plenoxels: Radiance fields without neural networks},'' in {\em IEEE/CVF
  Conference on Computer Vision and Pattern Recognition (CVPR)}, 2022.

\bibitem{fridovich2023k}
S.~Fridovich-Keil, G.~Meanti, F.~R. Warburg, B.~Recht, and A.~Kanazawa,
  ``{K-Planes: Explicit radiance fields in space, time, and appearance},'' in
  {\em IEEE/CVF Conference on Computer Vision and Pattern Recognition (CVPR)},
  2023.

\bibitem{chen2022tensorf}
A.~Chen, Z.~Xu, A.~Geiger, J.~Yu, and H.~Su, ``{TensoRF: Tensorial radiance
  fields},'' in {\em European Conference on Computer Vision (ECCV)}, 2022.

\bibitem{xu2022point}
Q.~Xu, Z.~Xu, J.~Philip, S.~Bi, Z.~Shu, K.~Sunkavalli, and U.~Neumann,
  ``{Point-NeRF: Point-based neural radiance fields},'' in {\em IEEE/CVF
  Conference on Computer Vision and Pattern Recognition (CVPR)}, 2022.

\bibitem{li2023neuralangelo}
Z.~Li, T.~M{\"u}ller, A.~Evans, R.~H. Taylor, M.~Unberath, M.-Y. Liu, and C.-H.
  Lin, ``{Neuralangelo: High-fidelity neural surface reconstruction},'' in {\em
  IEEE/CVF Conference on Computer Vision and Pattern Recognition (CVPR)}, 2023.

\bibitem{yu2021plenoctrees}
A.~Yu, R.~Li, M.~Tancik, H.~Li, R.~Ng, and A.~Kanazawa, ``{Plenoctrees for
  real-time rendering of neural radiance fields},'' in {\em IEEE/CVF
  International Conference on Computer Vision (ICCV)}, 2021.

\bibitem{jiang2023h2}
C.~Jiang, H.~Zhang, P.~Liu, Z.~Yu, H.~Cheng, B.~Zhou, and S.~Shen,
  ``H$_2$-mapping: Real-time dense mapping using hierarchical hybrid
  representation,'' {\em Robotics and Automation Letters}, 2023.

\bibitem{wang2022go}
J.~Wang, T.~Bleja, and L.~Agapito, ``{GO-Surf: Neural feature grid optimization
  for fast, high-fidelity RGB-D surface reconstruction},'' in {\em IEEE
  International Conference on 3D Vision (3DV)}, 2022.

\bibitem{sucar2021imap}
E.~Sucar, S.~Liu, J.~Ortiz, and A.~J. Davison, ``{iMAP: Implicit mapping and
  positioning in real-time},'' in {\em IEEE/CVF International Conference on
  Computer Vision (ICCV)}, 2021.

\bibitem{yang2022vox}
X.~Yang, H.~Li, H.~Zhai, Y.~Ming, Y.~Liu, and G.~Zhang, ``{Vox-Fusion: Dense
  tracking and mapping with voxel-based neural implicit representation},'' in
  {\em IEEE International Symposium on Mixed and Augmented Reality (ISMAR)},
  2022.

\bibitem{zhu2024nicer}
Z.~Zhu, S.~Peng, V.~Larsson, Z.~Cui, M.~R. Oswald, A.~Geiger, and M.~Pollefeys,
  ``{NICER-SLAM: Neural implicit scene encoding for RGB SLAM},'' in {\em
  International Conference on 3D Vision (3DV)}, 2024.

\bibitem{deng2023nerf}
J.~Deng, Q.~Wu, X.~Chen, S.~Xia, Z.~Sun, G.~Liu, W.~Yu, and L.~Pei,
  ``{NeRF-LOAM: Neural implicit representation for large-scale incremental
  lidar odometry and mapping},'' in {\em IEEE/CVF International Conference on
  Computer Vision (ICCV)}, 2023.

\bibitem{pan2024pin}
Y.~Pan, X.~Zhong, L.~Wiesmann, T.~Posewsky, J.~Behley, and C.~Stachniss,
  ``{PIN-SLAM: LiDAR SLAM using a point-based implicit neural representation
  for achieving global map consistency},'' {\em IEEE Transactions on Robotics
  (T-RO)}, 2024.

\bibitem{pan2025pings}
Y.~Pan, X.~Zhong, L.~Jin, L.~Wiesmann, M.~Popovi{\'c}, J.~Behley, and
  C.~Stachniss, ``{PINGS: Gaussian Splatting Meets Distance Fields within a
  Point-Based Implicit Neural Map},'' {\em arXiv preprint arXiv:2502.05752},
  2025.

\bibitem{johari2023eslam}
M.~M. Johari, C.~Carta, and F.~Fleuret, ``{ESLAM: Efficient dense SLAM system
  based on hybrid representation of signed distance fields},'' in {\em IEEE/CVF
  Conference on Computer Vision and Pattern Recognition (CVPR)}, 2023.

\bibitem{wang2023co}
H.~Wang, J.~Wang, and L.~Agapito, ``{Co-SLAM: Joint coordinate and sparse
  parametric encodings for neural real-time SLAM},'' in {\em IEEE/CVF
  Conference on Computer Vision and Pattern Recognition (CVPR)}, 2023.

\bibitem{zhang2023go}
Y.~Zhang, F.~Tosi, S.~Mattoccia, and M.~Poggi, ``{GO-SLAM: Global optimization
  for consistent 3D instant reconstruction},'' in {\em IEEE/CVF International
  Conference on Computer Vision (ICCV)}, 2023.

\bibitem{kahler2016real}
O.~K{\"a}hler, V.~A. Prisacariu, and D.~W. Murray, ``{Real-time large-scale
  dense 3D reconstruction with loop closure},'' in {\em European Conference on
  Computer Vision (ECCV)}, 2016.

\bibitem{liso2024loopy}
L.~Liso, E.~Sandstr{\"o}m, V.~Yugay, L.~Van~Gool, and M.~R. Oswald,
  ``{Loopy-SLAM: Dense neural SLAM with loop closures},'' in {\em IEEE/CVF
  Conference on Computer Vision and Pattern Recognition (CVPR)}, 2024.

\bibitem{deng2024plgslam}
T.~Deng, G.~Shen, T.~Qin, J.~Wang, W.~Zhao, J.~Wang, D.~Wang, and W.~Chen,
  ``{PLGSLAM: Progressive neural scene represenation with local to global
  bundle adjustment},'' in {\em IEEE/CVF Conference on Computer Vision and
  Pattern Recognition (CVPR)}, 2024.

\bibitem{matsuki2024newton}
H.~Matsuki, K.~Tateno, M.~Niemeyer, and F.~Tombari, ``{NEWTON: Neural
  view-centric mapping for on-the-fly large-scale SLAM},'' {\em IEEE Robotics
  and Automation Letters (RA-L)}, 2024.

\bibitem{liu2023efficient}
S.~Liu and J.~Zhu, ``Efficient map fusion for multiple implicit slam agents,''
  {\em Transactions on Intelligent Vehicles (T-IV)}, 2023.

\bibitem{hu2024cp}
J.~Hu, M.~Mao, H.~Bao, G.~Zhang, and Z.~Cui, ``{CP-SLAM: Collaborative neural
  point-based SLAM system},'' in {\em Advances in Neural Information Processing
  Systems (NeurIPS)}, 2024.

\bibitem{zhu2022nice}
Z.~Zhu, S.~Peng, V.~Larsson, W.~Xu, H.~Bao, Z.~Cui, M.~R. Oswald, and
  M.~Pollefeys, ``{NICE-SLAM: Neural implicit scalable encoding for SLAM},'' in
  {\em IEEE/CVF Conference on Computer Vision and Pattern Recognition (CVPR)},
  2022.

\bibitem{nocedal1999numerical}
J.~Nocedal and S.~J. Wright, {\em Numerical optimization}.
\newblock Springer, 1999.

\bibitem{kingma2014adam}
D.~P. Kingma, ``{Adam: A method for stochastic optimization},'' in {\em
  International Conference on Learning Representations (ICLR)}, 2014.

\bibitem{paszke2019pytorch}
A.~Paszke, S.~Gross, F.~Massa, A.~Lerer, J.~Bradbury, G.~Chanan, T.~Killeen,
  Z.~Lin, N.~Gimelshein, L.~Antiga, {\em et~al.}, ``{PyTorch: An imperative
  style, high-performance deep learning library},'' in {\em Advances in Neural
  Information Processing Systems (NeurIPS)}, 2019.

\bibitem{gropp2020implicit}
A.~Gropp, L.~Yariv, N.~Haim, M.~Atzmon, and Y.~Lipman, ``{Implicit geometric
  regularization for learning shapes},'' in {\em International Conference on
  Machine Learning (ICML)}, 2020.

\bibitem{peng2020convolutional}
S.~Peng, M.~Niemeyer, L.~Mescheder, M.~Pollefeys, and A.~Geiger,
  ``{Convolutional occupancy networks},'' in {\em European Conference on
  Computer Vision (ECCV)}, 2020.

\bibitem{vahdat2020nvae}
A.~Vahdat and J.~Kautz, ``{NVAE: A deep hierarchical variational
  autoencoder},'' in {\em Advances in Neural Information Processing Systems
  (NeurIPS)}, 2020.

\bibitem{dai2017scannet}
A.~Dai, A.~X. Chang, M.~Savva, M.~Halber, T.~Funkhouser, and M.~Nie{\ss}ner,
  ``{ScanNet: Richly-annotated 3D reconstructions of indoor scenes},'' in {\em
  IEEE/CVF Conference on Computer Vision and Pattern Recognition (CVPR)}, 2017.

\bibitem{park2017colored}
J.~Park, Q.-Y. Zhou, and V.~Koltun, ``Colored point cloud registration
  revisited,'' in {\em International Conference on Computer Cision}, 2017.

\bibitem{replica19arxiv}
J.~Straub, T.~Whelan, L.~Ma, Y.~Chen, E.~Wijmans, S.~Green, J.~J. Engel,
  R.~Mur-Artal, C.~Ren, S.~Verma, A.~Clarkson, M.~Yan, B.~Budge, Y.~Yan,
  X.~Pan, J.~Yon, Y.~Zou, K.~Leon, N.~Carter, J.~Briales, T.~Gillingham,
  E.~Mueggler, L.~Pesqueira, M.~Savva, D.~Batra, H.~M. Strasdat, R.~De~Nardi,
  M.~Goesele, S.~Lovegrove, and R.~Newcombe, ``{The Replica dataset: A digital
  replica of indoor spaces},'' {\em arXiv preprint arXiv:1906.05797}, 2019.

\bibitem{zhang2021ncdmulti}
L.~Zhang, M.~Camurri, D.~Wisth, and M.~Fallon, ``Multi-camera lidar inertial
  extension to the newer college dataset,'' {\em arXiv preprint
  arXiv:2112.08854}, 2021.

\bibitem{choi2015robust}
S.~Choi, Q.-Y. Zhou, and V.~Koltun, ``Robust reconstruction of indoor scenes,''
  in {\em Conference on computer vision and pattern recognition}, 2015.

\bibitem{zhou2018open3d}
Q.-Y. Zhou, J.~Park, and V.~Koltun, ``{Open3D: A modern library for 3D data
  processing},'' {\em arXiv preprint arXiv:1801.09847}, 2018.

\bibitem{vizzo2023kiss}
I.~Vizzo, T.~Guadagnino, B.~Mersch, L.~Wiesmann, J.~Behley, and C.~Stachniss,
  ``{KISS-ICP}: In defense of point-to-point {ICP}--simple, accurate, and
  robust registration if done the right way,'' {\em IEEE Robotics and
  Automation Letters (RA-L)}, 2023.

\end{thebibliography}
